\documentclass{article}

\PassOptionsToPackage{numbers, compress}{natbib}

\usepackage{nccmath}

    \usepackage[preprint]{neurips_2022}

\usepackage[utf8]{inputenc} %
\usepackage[T1]{fontenc}    %
\usepackage[hidelinks, backref=page]{hyperref}       %
\usepackage{url}            %
\usepackage{booktabs}       %
\usepackage{amsfonts}       %
\usepackage{nicefrac}       %
\usepackage{microtype}      %
\usepackage{xcolor}         %

\usepackage{svg}
\usepackage{wrapfig}
\usepackage{tablefootnote}
\usepackage{algorithmic}
\usepackage{graphicx}
\usepackage{awesomebox}
\usepackage{alertmessage}
\usepackage{enumitem}
\usepackage{pifont}%
\usepackage{amsthm}
\usepackage{caption}
\usepackage{subcaption}

\usepackage{comment}
\usepackage{float}
\usepackage{authblk}

\usepackage{soul}

\newcommand{\ie}{i.e.\@\xspace}
\newcommand{\Ie}{I.e.\@\xspace}
\newcommand{\eg}{e.g.\@\xspace}

\newcommand{\cf}{cf.\@\xspace} %

\newcommand{\wrt}{w.r.t.\@\xspace} %

\newcommand{\inv}[1]{\ensuremath{#1^{-1}}}
\newcommand{\transpose}[1]{\ensuremath{#1^{T}}}
\newcommand{\diag}[1]{\ensuremath{\mathrm{diag}\parenthesis{#1}}}

\newcommand{\mat}[1]{\ensuremath{\boldsymbol{\mathrm{#1}}}}

\newcommand{\sumk}[1][M]{\ensuremath{\sum_{k=1}^{#1}}}

\newcommand{\derivative}[2]{\ensuremath{\dfrac{\partial #1}{\partial #2}}}

\newcommand{\cols}[1]{\ensuremath{\brackets{#1}_{:k}}}
\newcommand{\rows}[1]{\ensuremath{\brackets{#1}_{k:}}}
\newcommand{\norm}[1]{\ensuremath{\left\Vert#1\right\Vert}}
\newcommand{\normsquared}[1]{\ensuremath{\norm{#1}^2}}
\newcommand{\abs}[1]{\ensuremath{\left|#1\right|}}
\newcommand{\logabsdet}[1]{\ensuremath{\log\left|#1\right|}}

\newcommand{\expnum}[2]{\ensuremath{{#1}\mathrm{e}{#2}}}
\newcommand{\expectation}[1]{\ensuremath{\mathbb{E}_{#1}}}

\newcommand{\normal}[1][\ensuremath{I_n}]{\ensuremath{\mathcal{N}\parenthesis{0;#1}}}

\newcommand{\parenthesis}[1]{\ensuremath{\left(#1\right)}}
\newcommand{\brackets}[1]{\ensuremath{\left[#1\right]}}
\newcommand{\braces}[1]{\ensuremath{\left\{#1\right\}}}
\newcommand{\trace}[1]{\ensuremath{\mathrm{tr}\parenthesis{#1}}}

\newcommand{\kl}[2]{\ensuremath{\text{KL}\brackets{#1|| #2}}}

\newtheorem{assum}{Assumption}
\newtheorem{prop}{Proposition}
\newtheorem{conject}{Conjecture}
\newtheorem{lem}{Lemma}
\newtheorem{remark}{Remark}
\newtheorem{definition}{Definition}

\newtheorem{cor}{Corollary}

\newcommand{\patrik}[1]{\textcolor{cyan}{[\textbf{Patrik:} #1]}}

\newcommand{\luigi}[1]{\textcolor{red}{[\textbf{Luigi:} #1]}}

\newcommand{\jack}[1]{\textcolor{blue}{[\textbf{Jack:} #1]}}

\newcommand{\michel}[1]{\textcolor{teal}{[\textbf{Michel:} #1]}}

\usepackage[
	textsize=tiny,
	textwidth=2.8cm
	]{todonotes}
\newcommand{\georg}[2][]{\todo[backgroundcolor=green!30!white, linecolor=green!50!white, #1]{Georg: #2}
}
\newcommand{\luigig}[2][]{\todo[backgroundcolor=red!30!white, linecolor=red!50!white, #1]{Luigi: #2}
}

\renewcommand{\derivative}[2]{\ensuremath{\frac{\partial #1}{\partial #2}}}

\newcommand{\RR}{\mathbb{R}}
\newcommand{\decpar}{\boldsymbol{\theta}}
\newcommand{\decod}{{{\rm f}^{\decpar}}}

\newcommand{\encpar}{\boldsymbol{\phi}}
\newcommand{\encparopt}{\widehat{\boldsymbol{\phi}}}
\newcommand{\encmean}{\boldsymbol{\mu}^{\encpar}}
\newcommand{\encmeancomp}{{\mu}^{\encpar}}
\newcommand{\encmeanopt}{\boldsymbol{\mu}^{\encparopt}}
\newcommand{\encmeancompopt}{{\mu}^{\encparopt}}
\newcommand{\encstd}{\boldsymbol{\sigma}^{\encpar}}
\newcommand{\encstdcomp}{{\sigma}^{\encpar}}
\newcommand{\encstdopt}{\boldsymbol{\sigma}^{\encparopt}}
\newcommand{\encstdcompopt}{{\sigma}^{\encparopt}}

\newcommand{\encparstar}{{\boldsymbol{\phi}^*}}

\newcommand{\encmeanstar}{\boldsymbol{\mu}^{\encparstar}}
\newcommand{\encmeancompstar}{{\mu}^{\encparstar}}
\newcommand{\encstdstar}{\boldsymbol{\sigma}^{\encparstar}}
\newcommand{\encstdcompstar}{{\sigma}^{\encparstar}}

\newcommand{\Lambdamat}{\ensuremath{\boldsymbol{\mathrm{\Lambda}}}}
\newcommand{\Umat}{\ensuremath{\boldsymbol{\mathrm{U}}}}
\newcommand{\Mmat}{\ensuremath{\boldsymbol{\mathrm{M}}}}
\newcommand{\ptheta}[1]{\ensuremath{p_{\gls{decpar}}(#1)}}

\newcommand{\xmark}{\ding{55}}%

\usepackage{thmtools,thm-restate}
\usepackage[acronym, automake, toc, nomain, nopostdot, style=tree, nonumberlist]{glossaries}
\usepackage{glossary-mcols}
\usepackage{xparse}
\usepackage{xspace}
\setglossarystyle{mcolindex}

\newacronym{mpa}{MPA}{Measure Preserving Automorphism}
\newacronym{iid}{i.i.d.}{independent and identically distributed}
\newacronym{vmf}{vMF}{von Mises-Fisher}
\newacronym{pd}{PD}{positive definite}
\newacronym{psd}{PSD}{positive semi-definite}
\newacronym{nd}{ND}{negative definite}
\newacronym{nsd}{NSD}{negative semi-definite}

\newacronym{ae}{AE}{AutoEncoder}
\newacronym{lae}{LAE}{Linear Autoencoder}
\newacronym{vae}{VAE}{Variational Autoencoder}
\newacronym{cvvae}{CV-VAE}{Constant-Variance Variational Autoencoder}
\newacronym{ivae}{iVAE}{Identifiable Variational Autoencoder}
\newacronym{rae}{RAE}{Regularized Autoencoder}
\newacronym{grae}{GRAE}{Gaussian Regularized Autoencoder}
\newacronym{lvm}{LVM}{Latent Variable Model}

\newcommand{\betavae}{$\beta$-\gls{vae}\xspace}

\newacronym{kld}{KL}{Kullback-Leibler Divergence}
\newacronym{elbo}{{\text{\upshape ELBO}}}{evidence lower bound}
\newacronym{pca}{PCA}{Principal Component Analysis}
\newacronym{ppca}{PPCA}{Probabilistic Principal Component Analysis}
\newacronym{ebm}{EBM}{Energy-Based Model}

\newacronym{icm}{ICM}{Independent Causal Mechanisms}
\newacronym{sem}{SEM}{Structural Equation Model}
\newacronym{lingam}{LiNGAM}{Linear Non-Gaussian Acyclic Model}
\newacronym{dag}{DAG}{Directed Acyclic Graph}
\newacronym{anm}{ANM}{Additive Noise Model}
\newacronym{cd}{CD}{Causal Discovery}
\newacronym{crl}{CRL}{Causal Representation Learning}

\newacronym{ica}{ICA}{Independent Component Analysis}
\newacronym{nlica}{NLICA}{nonlinear Independent Component Analysis}
\newacronym{bss}{BSS}{Blind Source Separation}

\newacronym{ima}{{\text{\upshape IMA}}}{Independent Mechanism Analysis}
\newacronym{igci}{IGCI}{Information Geometric Causal Inference}

\newacronym{nce}{NCE}{Noise Contrastive Estimation}
\newacronym{pcl}{PCL}{Permutation-Contrastive Learning}
\newacronym{tcl}{TCL}{Time-Contrastive Learning}
\newacronym{iia}{IIA}{Independent Innovation Analysis}

\newacronym{ai}{AI}{Artificial Intelligence}
\newacronym{ml}{ML}{Machine Learning}
\newacronym{dl}{DL}{Deep Learning}
\newacronym{ssl}{SSL}{Self-Supervised Learning}
\newacronym{cl}{CL}{Contrastive Learning}
\newacronym{gmc}{GMC}{Geometric Multimodal Contrastive Learning}
\newacronym{nlp}{NLP}{Natural Language Processing}
\newacronym{gdl}{GDL}{Geometric Deep Learning}

\newacronym{dnn}{DNN}{Deep Neural Network}
\newacronym{nn}{NN}{Neural Network}
\newacronym{ann}{ANN}{Artificial Neural Network}

\newacronym{mlp}{MLP}{Multi-Layer Perceptron}
\newacronym{fc}{FC}{Fully Connected}
\newacronym{ar}{AR}{AutoRegressive}

\newacronym{cn}{conv}{Convolutional layer}
\newacronym{cnn}{CNN}{Convolutional Neural Network}
\newacronym{gnn}{GNN}{Graph Neural Network}

\newacronym{rnn}{RNN}{Recurrent Neural Network}
\newacronym{lstm}{LSTM}{Long Short-Term Memory}
\newacronym{gru}{GRU}{Gated Recurrent Unit}
\newacronym{relu}{ReLU}{Rectified Linear Unit}
\newacronym{bn}{BN}{Batch Normalization}

\newacronym{gan}{GAN}{Generative Adversarial Network}

\newacronym{sgd}{SGD}{Stochastic Gradient Descent}
\newacronym{adam}{ADAM}{Adaptive Moment Estimation}
\newacronym{svd}{SVD}{Singular Value Decomposition}
\newacronym{wls}{WLS}{Weighted Least Squares}

\newacronym{dgp}{DGP}{Data Generating Process}
\newacronym{map}{MAP}{Maximum A Posteriori}
\newacronym{mle}{MLE}{Maximum Likelihood Estimation}

\newacronym{etf}{ETF}{Equiangular Tight Frame}

\newacronym{mse}{MSE}{Mean Squared Error}
\newacronym{mae}{MAE}{Mean Absolute Error}
\newacronym{ce}{{\text{\upshape CE}}}{Cross Entropy}

\newacronym{sid}{SID}{Structural Intervention Distance}
\newacronym{shd}{SHD}{Structural Hamming Distance}

\newacronym{mcc}{MCC}{Mean Correlation Coefficient}
\newacronym{mig}{MIG}{Mutual Information Gap}
\newacronym{dci}{DCI}{Disentanglement Completeness Informativeness score}

\newacronym{api}{API}{Application Programming Interface}
\newacronym{cpu}{CPU}{Central Processing Unit}
\newacronym{gpu}{GPU}{Graphics Processing Unit}

\newacronym{gt}{{\text{\upshape GT}}}{ground truth}

\newglossary{abbrev}{abs}{abo}{Nomenclature}

\newglossaryentry{obs}{
    name        = \ensuremath{\boldsymbol{x}} ,
    description = {observation vector} ,
    type        = abbrev,
}

\newglossaryentry{obscomp}{
    name        = \ensuremath{x} ,
    description = {observation single component} ,
    type        = abbrev,
}

\newglossaryentry{Obs}{
    name        = \ensuremath{\mathcal{X}} ,
    description = {observation space} ,
    type        = abbrev,
}

\newglossaryentry{obsdim}{
    name        = \ensuremath{d} ,
    description = {dimensionality of the observation space \gls{Obs}} ,
    type        = abbrev,
}

\newglossaryentry{latent}{
    name        = \ensuremath{\boldsymbol{z}} ,
    description = {latent vector} ,
    type        = abbrev,
}

\newglossaryentry{latentcomp}{
    name        = \ensuremath{z} ,
    description = {latent single component} ,
    type        = abbrev,
}

\newglossaryentry{Latent}{
    name        = \ensuremath{\mathcal{Z}} ,
    description = {latents} ,
    type        = abbrev,
}

\newglossaryentry{latentdim}{
    name        = \ensuremath{d} ,
    description = {dimensionality of the latent space \gls{Latent}} ,
    type        = abbrev,
}

\newglossaryentry{identity}{
    name        = \ensuremath{\boldsymbol{\mathrm{I}}_{\gls{obsdim}}} ,
    description = {\gls{obsdim}-dimensional identity matrix} ,
    type        = abbrev,
}

\newglossaryentry{jacobian}{
    name        = \ensuremath{\boldsymbol{\mathrm{J}}} ,
    description = {Jacobian matrix} ,
    type        = abbrev,
}

\newglossaryentry{hessian}{
    name        = \ensuremath{\boldsymbol{\mathrm{H}}} ,
    description = {Hessian matrix} ,
    type        = abbrev,
}

\newglossaryentry{cov}{
    name        = \ensuremath{\boldsymbol{\mathrm{\Sigma}}},
    description = {covariance matrix} ,
    type        = abbrev,
}

\newglossaryentry{loss}{
    name        = \ensuremath{\mathcal{L}} ,
    description = {loss function} ,
    type        = abbrev,
}

\newglossaryentry{entropy}{
    name        = \ensuremath{\mathrm{H}} ,
    description = {entropy} ,
    type        = abbrev,
}

\newglossaryentry{vaes}{type=abbrev,name=\acrlong{vae},description={\nopostdesc}}

\newglossaryentry{q}{
    name        = \ensuremath{q_{\gls{encpar}}(\gls{latent}|\gls{obs})} ,
    description = {variational posterior of the \acrshort{vae}, mapping $\gls{obs}\mapsto\gls{latent}$ parametrized by \gls{encpar}} ,
    type        = abbrev,
    parent      = vaes,
}
\newglossaryentry{qopt}{
    name        = \ensuremath{q_{\widehat{\gls{encpar}}}(\gls{latent}|\gls{obs})} ,
    description = {optimal variational posterior of the \acrshort{vae}, mapping $\gls{obs}\mapsto\gls{latent}$ parametrized by \gls{encpar}} ,
    type        = abbrev,
    parent      = vaes,
}

\newglossaryentry{encpar}{
    name        = \ensuremath{\boldsymbol{\phi}} ,
    description = {parameters of the variational posterior \gls{q}} ,
    type        = abbrev,
    parent      = vaes,
}

\newglossaryentry{encparopt}{
    name        = \ensuremath{\widehat{\boldsymbol{\phi}}} ,
    description = {optimal parameters of the variational posterior \gls{q}} ,
    type        = abbrev,
    parent      = vaes,
}

\newglossaryentry{var_family}{
    name        = \ensuremath{\mathcal{Q}} ,
    description = {distribution family of the variational posterior \gls{q} } ,
    type        = abbrev,
    parent      = vaes,
}

\newglossaryentry{pz}{
    name        = \ensuremath{p_0(\gls{latent})} ,
    description = {latent prior distribution} ,
    type        = abbrev,
    parent      = vaes,
}

\newglossaryentry{px}{
    name        = \ensuremath{p_{\gls{decpar}}(\gls{obs})} ,
    description = {marginal likelihood } ,
    type        = abbrev,
    parent      = vaes,
}

\newglossaryentry{pdata}{
    name        = \ensuremath{p(\gls{obs})} ,
    description = {data distribution } ,
    type        = abbrev,
    parent      = vaes,
}

\newglossaryentry{mean_enc}{
    name        = \ensuremath{\mu_{\gls{latent}|\gls{obs}}} ,
    description = {mean encoder of the \acrshort{vae}, \ie, $\expectation{\gls{latent}\sim\gls{q}}\parenthesis{\gls{latent}}$, mapping $\gls{obs}\mapsto\gls{latent}$} ,
    type        = abbrev,
    parent      = vaes,
}

\newglossaryentry{var_cov}{
    name        = \ensuremath{\gls{cov}^{\gls{encpar}}_{\gls{latent}|\gls{obs}}} ,
    description = {covariance matrix of \gls{q}} ,
    type        = abbrev,
    parent      = vaes,
}

\newglossaryentry{sigmak}{
    name        = \ensuremath{{\sigma}_{k}^{\gls{encpar}}(\gls{obs})^{2}} ,
    description = {variance of \gls{q} in dimension $k$} ,
    type        = abbrev,
    parent      = vaes,
}

\newglossaryentry{sigmaopt}{
    name        = \ensuremath{\boldsymbol{\sigma}^{\gls{encparopt}}(\gls{obs})^{2}} ,
    description = {optimal variance of \gls{q}} ,
    type        = abbrev,
    parent      = vaes,
}

\newglossaryentry{sigmaoptk}{
    name        = \ensuremath{{\sigma}_{k}^{\gls{encparopt}}(\gls{obs})^{2}} ,
    description = {optimal variance of \gls{q} in dimension $k$} ,
    type        = abbrev,
    parent      = vaes,
}
\newcommand{\sigmaoptkcube}{\ensuremath{{\sigma}_{k}^{\gls{encparopt}}(\gls{obs})^{3}}}
\newglossaryentry{mu}{
    name        = \ensuremath{\boldsymbol{\mu}^{\gls{encpar}}(\gls{obs})} ,
    description = {mean of \gls{q}} ,
    type        = abbrev,
    parent      = vaes,
}

\newglossaryentry{muk}{
    name        = \ensuremath{{\mu}_{k}^{\gls{encpar}}(\gls{obs})} ,
    description = {mean of \gls{q} in dimension $k$} ,
    type        = abbrev,
    parent      = vaes,
}

\newglossaryentry{muopt}{
    name        = \ensuremath{\boldsymbol{\mu}^{\gls{encparopt}}(\gls{obs})} ,
    description = {optimal mean of \gls{q}} ,
    type        = abbrev,
    parent      = vaes,
}

\newglossaryentry{muoptk}{
    name        = \ensuremath{{\mu}_{k}^{\gls{encparopt}}(\gls{obs})} ,
    description = {optimal mean of \gls{q} in dimension $k$} ,
    type        = abbrev,
    parent      = vaes,
}

\newglossaryentry{gamma}{
    name        = \ensuremath{\gamma} ,
    description = {square root of the precision of the \gls{vae} decoder} ,
    type        = abbrev,
    parent      = vaes,
}

\newglossaryentry{betaloss}{
    name        = \ensuremath{\mathcal{L}_{\beta}} ,
    description = {\betavae loss function} ,
    type        = abbrev,
    parent      = vaes,
}

\newcommand{\gsq}{\ensuremath{\gls{gamma}^2}\xspace}

\newglossaryentry{pxz}{
    name        = \ensuremath{p_{\gls{decpar}}(\gls{obs}|\gls{latent})} ,
    description = {conditional distribution of the decoded samples of the \acrshort{vae}, mapping $\gls{latent}\mapsto\gls{obs}$, parametrized by \gls{decpar}} ,
    type        = abbrev,
    parent      = vaes,
}
\newglossaryentry{pzx}{
    name        = \ensuremath{p_{\gls{decpar}}(\gls{latent}|\gls{obs})} ,
    description = {true posterior distribution of the decoded samples of the \acrshort{vae}, mapping $\gls{obs}\mapsto\gls{latent}$, parametrized by \gls{decpar}} ,
    type        = abbrev,
    parent      = vaes,
}
\newglossaryentry{decpar}{
    name        = \ensuremath{\boldsymbol{\theta}} ,
    description = {parameters of the decoder \gls{pxz}} ,
    type        = abbrev,
    parent      = vaes,
}

\newglossaryentry{invdeccomp}{
    name        = \ensuremath{{g}^{\decpar}} ,
    description = {inverse decoder component} ,
    type        = abbrev,
    parent      = vaes,
}
\newcommand{\invdeccompk}[1][k]{\ensuremath{{g}^{\decpar}_{#1}}(\gls{obs})}

\newglossaryentry{invdec}{
    name        = \ensuremath{\mathrm{\boldsymbol{g}}^{\gls{decpar}}} ,
    description = {inverse decoder} ,
    type        = abbrev,
    parent      = vaes,
}

\newglossaryentry{dec}{
    name        = \ensuremath{\mathrm{\boldsymbol{f}}^{\gls{decpar}}} ,
    description = {decoder} ,
    type        = abbrev,
    parent      = vaes,
}

\newglossaryentry{distortion}{
    name        = \ensuremath{D} ,
    description = {Distortion of \cite{alemi_fixing_2018}, the same as the reconstruction term of the \acrshort{elbo} for $\beta=1$} ,
    type        = abbrev,
    parent      = vaes,
}

\newglossaryentry{rate}{
    name        = \ensuremath{R} ,
    description = {Rate of \cite{alemi_fixing_2018}, the same as the \acrshort{kld} term of the \acrshort{elbo} for $\beta=1$} ,
    type        = abbrev,
    parent      = vaes,
}

\newglossaryentry{lindec}{
    name        = \ensuremath{\boldsymbol{\mathrm{W}}} ,
    description = {weight matrix of a linear decoder} ,
    type        = abbrev,
    parent      = vaes,
}

\newglossaryentry{linenc}{
    name        = \ensuremath{\boldsymbol{\mathrm{V}}} ,
    description = {weight matrix of a linear encoder} ,
    type        = abbrev,
    parent      = vaes,
}

\newglossaryentry{imas}{type=abbrev,name=\acrlong{ima},description={\nopostdesc}}

\newglossaryentry{mixing}{
    name        = \ensuremath{\inv{g}} ,
    description = {inverse of the learned unmixing of the \acrshort{ima}, mapping $\gls{latent}\mapsto\gls{obs}$ } ,
    type        = abbrev,
    parent      = imas,
}

\newglossaryentry{lin_mixing}{
    name        = \ensuremath{A} ,
    description = {ground-truth \emph{linear} mixing process of the \acrshort{ima}, mapping $\gls{latent}\mapsto\gls{obs}$ } ,
    type        = abbrev,
    parent      = imas,
}

\newglossaryentry{cima_local}{
    name        = \ensuremath{c_{\acrshort{ima}}} ,
    description = {local \acrshort{ima} contrast } ,
    type        = abbrev,
    parent      = imas,
}

\newglossaryentry{cima_global}{
    name        = \ensuremath{C_{\acrshort{ima}}} ,
    description = {global \acrshort{ima} contrast } ,
    type        = abbrev,
    parent      = imas,
}

\newglossaryentry{source}{
    name        = \ensuremath{s} ,
    description = {sources (\acrshort{ica} equivalent of latents)} ,
    type        = abbrev,
    parent      = imas,
}

\newglossaryentry{rec_s}{
    name        = \ensuremath{\boldsymbol{y}} ,
    description = {reconstructed sources} ,
    type        = abbrev,
    parent      = imas,
}

\newglossaryentry{p_source}{
    name        = \ensuremath{p_{\gls{latent}}} ,
    description = {source distribution} ,
    type        = abbrev,
    parent      = imas,
}

\newglossaryentry{o}{
    name        = \ensuremath{\boldsymbol{\mathrm{O}}},
    description = {orthogonal matrix} ,
    type        = abbrev,
    parent      = imas,
}

\newglossaryentry{d}{
    name        = \ensuremath{\boldsymbol{\mathrm{D}}} ,
    description = {general diagonal matrix} ,
    type        = abbrev,
    parent      = imas,
}

\newglossaryentry{scalar}{
    name        = \ensuremath{\alpha} ,
    description = {scalar field} ,
    type        = abbrev,
    parent      = imas,
}

\newglossaryentry{imaloss}{
    name        = \ensuremath{\mathcal{L}_{\gls{ima}}} ,
    description = {\gls{ima} loss function} ,
    type        = abbrev,
    parent      = imas,
}

\newcommand{\unmix}[1][\gls{obs}]{\ensuremath{\gls{invdec}\parenthesis{#1}}\xspace}

\newcommand{\gtmix}[1][\gls{latent}]{\ensuremath{\gls{dec}\parenthesis{#1}}\xspace}

\newcommand{\unmixjacobian}[1][\gls{obs}]{\ensuremath{\gls{jacobian}_{\gls{invdec}}(#1)}\xspace}
\newcommand{\mixjacobian}[1][\unmix]{\ensuremath{\gls{jacobian}_{\gls{dec}}\parenthesis{#1}}\xspace}
\newcommand{\mixjacobianbis}[1]{\ensuremath{\gls{jacobian}_{\gls{dec}}\parenthesis{#1}}}

\newcommand{\gtmixjacobian}[1][\gls{latent}]{\ensuremath{\gls{jacobian}_{\gls{dec}}\parenthesis{#1}}\xspace}

\newcommand{\imaloss}{\ensuremath{\gls{imaloss}(\gls{dec}\!, \gls{latent})}\xspace}

\newcommand{\betaloss}{\ensuremath{\gls{betaloss}(\gls{obs};\gls{decpar}, \gls{encpar})}\xspace}
\newcommand{\betalossgamma}[1][\gamma]{\ensuremath{\gls{betaloss}(\gls{obs};\gls{decpar}, \gls{encpar}, #1)}\xspace}
\newcommand{\betalossgammastar}[1][\gamma]{\ensuremath{\gls{betaloss}^*(\gls{obs};\gls{decpar}, \gls{encpar}, #1)}\xspace}

\newcommand{\elbolossgamma}[1][\gamma]{\ensuremath{\gls{elbo}(\gls{obs};\gls{decpar}, \gls{encpar}, #1)}\xspace}
\newcommand{\elbolossgammastar}[1][\gamma]{\ensuremath{\gls{elbo}^*(\gls{obs};\gls{decpar}, \gls{encpar}, #1)}\xspace}

\NewDocumentCommand{\cima}{ O{\gls{dec}} O{\gls{latent}}  }{\ensuremath{\gls{cima_local} ( #1\!,  #2) }\xspace}
\NewDocumentCommand{\Cima}{ O{\gls{dec}} O{\ensuremath{p_0}  }}{\ensuremath{\gls{cima_global} ( #1,  #2) }\xspace}

\makeglossaries

\usepackage{titlesec}
\titlespacing*{\section}{0pt}{0.2ex plus .1ex minus .1ex}{0.1ex plus .1ex minus .1ex}
\titlespacing*{\subsection}{0pt}{0.1ex plus .1ex minus .1ex}{0.05ex plus .05ex minus .05ex}

\usepackage{cleveref}
\crefname{section}{\S}{\S\S}
\crefname{figure}{Fig.}{Figs.}
\crefname{prop}{Prop.}{Props.}
\crefname{appendix}{Appx.}{Appxs.}
\crefname{theorem}{Thm.}{Thms.}
\crefname{definition}{Defn.}{Defns.}
\crefname{cor}{Corollary}{Corollaries}
\crefname{lem}{Lemma}{Lemmas}
\crefname{table}{Tab.}{Tabs.}
\crefname{assum}{Assum.}{Assums.}
\usepackage[toc,page,header]{appendix}
\usepackage{minitoc}

\title{Embrace the Gap: VAEs Perform\\ Independent Mechanism Analysis}

\author[1]{\href{mailto:patrik.reizinger@uni-tuebingen.de}{Patrik~Reizinger\thanks{Equal contribution. Code available at: \href{https://github.com/rpatrik96/ima-vae}{\texttt{github.com/rpatrik96/ima-vae}}}$\ $  }{}}
\author[2]{Luigi~Gresele$^*$}
\author[1]{Jack~Brady$^*$}
\author[2,3]{Julius~von~K{\"u}gelgen}
\author[2,4]{Dominik~Zietlow}
\author[2]{Bernhard~Schölkopf}
\author[2]{Georg~Martius}
\author[1]{Wieland~Brendel}
\author[2]{Michel~Besserve\thanks{Senior author}$\ $ }
\affil[1]{%
    University of Tübingen, Germany
}
\affil[2]{%
    Max Planck Institute for Intelligent Systems, Tübingen, Germany
}
\affil[3]{%
    University of Cambridge, Cambridge, United Kingdom 
  }
\affil[4]{%
    Amazon Web Services, Tübingen, Germany 
  }
\affil[ ]{%
    \texttt{\{patrik.reizinger,jack.brady,wieland.brendel\}@uni-tuebingen.de}\\
  \texttt{\{luigi.gresele,jvk,bs,gmartius,besserve\}@tue.mpg.de}\\ \texttt{zietld@amazon.de}
  }

\begin{document}
\doparttoc %
\faketableofcontents %

\maketitle

\iftrue

\vspace{-0.5em}
\begin{abstract}
\vspace{-0.5em}

Variational autoencoders (VAEs) are a popular framework for modeling complex data distributions; they can be efficiently trained via variational inference by maximizing the evidence lower bound (ELBO), at the expense of a gap to the exact (log-)marginal likelihood. While VAEs are commonly used for disentangled representation learning, it is unclear why ELBO maximization would yield such representations, since unregularized maximum likelihood estimation generally cannot invert the data-generating process without additional assumptions. Yet, VAEs often succeed at this task. We seek to elucidate this apparent paradox by studying nonlinear VAEs in the limit of near-deterministic decoders. We first prove that, in this regime, the optimal encoder approximately inverts the decoder---a commonly used but unproven conjecture---which we refer to as {\em self-consistency}. Leveraging self-consistency, we show that the ELBO converges to a regularized log-likelihood. This 
allows VAEs to perform what has recently been termed independent mechanism analysis (IMA): it adds an inductive bias towards decoders with column-orthogonal Jacobians, which 
helps recovering the true latent factors. The gap between ELBO and log-likelihood is therefore welcome, since it bears unanticipated benefits for nonlinear representation learning. In experiments on synthetic and image data, we show that VAEs uncover the true latent factors when the data generating process satisfies the IMA assumption.
\end{abstract}

\section{Introduction}
\label{sec:Introduction}
\glspl{lvm} 
allow to effectively approximate 
a complex %
data distribution and to sample from it%
~\citep{bishop2006pattern, murphy2012machine}. %
Deep \glspl{lvm} employ a neural network (the \textit{decoder} or \textit{generator}) to parameterize the conditional distribution of the observations given latent variables, which are typically assumed to be independent.
However, \gls{mle} of the model parameters
is computationally intractable.
In %
\textit{\glspl{vae}}
~\citep{kingma_auto-encoding_2014, rezende2014stochastic}, %
the exact log-likelihood is substituted with %
a tractable lower bound, the 
\gls{elbo}. 
This objective introduces
an approximate posterior of the latents given the observations (the \textit{encoder}%
) from a suitable variational distribution whose mean and covariance are parametrized by neural networks.
The encoder is introduced to efficiently train a deep \gls{lvm}: however, it is not explicitly designed to extract useful representations~\cite{doersch_tutorial_2021, rubenstein_2019}.

\looseness-1 
Nonetheless, \glspl{vae} and their variants are widely used in representation learning~\cite{higgins2016beta, alemi_fixing_2018}, where they often
recover semantically meaningful representations 
~\cite{kumar_variational_2018, chen2018isolating, kim2018disentangling, burgess_understanding_2018}. %
Our understanding of this empirical success is still incomplete, since (deep) \gls{lvm}s with independent latents are nonidentifiable from i.i.d.\ data~\cite{hyvarinen_nonlinear_1999, locatello_challenging_2019}; different models fitting the data equally well may yield arbitrarily different representations, thus making the recovery of a ground truth generative model 
impossible. 
While auxiliary variables, weak supervision~\cite{hyvarinen_nonlinear_2017, hyvarinen_nonlinear_2019, gresele_incomplete_2019, locatello_weakly-supervised_2020, zimmermann_contrastive_2021, halva_disentangling_2021}, or  specific model constraints~\cite{hyvarinen_nonlinear_1999, zhang2008minimal, zhang_identifiability_2012, horan_when_2021, gresele_independent_2021} can help identifiability, the mechanism through which the \gls{elbo} may enforce a useful inductive bias remains unclear, despite recent efforts~\citep{burgess_understanding_2018, rolinek_variational_2019, kumar_implicit_2020,Dai2020:usualsuspects,zietlow_demystifying_2021}.%

\looseness-1 
In this work, we investigate %
the benefits of optimizing the \gls{elbo} for representation learning
by analyzing \glspl{vae} in a {\em near-deterministic} limit for the conditional distribution parametrized by the nonlinear decoder. %
Our first result concerns the encoder's optimality in this regime. %
Previous works %
relied %
on the intuitive assumption that the encoder inverts the decoder in the optimum~\citep{nielsen_survae_2020,kumar_implicit_2020,zietlow_demystifying_2021};
we formalize this \textit{self-consistency} assumption and prove its validity for the optimal variational posterior 
in the %
near-deterministic nonlinear %
regime.

\looseness-1
Using self-consistency, we
show that the %
\gls{elbo} tends to a regularized log-likelihood---rather than to the exact one as conjectured in previous work~\cite{nielsen_survae_2020}. The regularization term allows \glspl{vae}
to perform what has been termed \gls{ima}~\cite{gresele_independent_2021}: it encourages column orthogonality of the decoder's Jacobian. This generalizes previous findings based on linearizations or approximations of the \gls{elbo}~\cite{rolinek_variational_2019, lucas_dont_2019, kumar_implicit_2020}, and
allows us to characterize the gap \wrt
the log-likelihood in the deterministic limit.
Our results elucidate the gap between \gls{elbo} and exact log-likelihood as a possible mechanism through which the \gls{elbo} implements a useful inductive bias.
Unlike the unregularized log-likelihood, 
the \gls{ima}-regularized objective can help invert the data generating process under suitable assumptions~\cite{gresele_independent_2021}.
We verify this %
by training \glspl{vae} %
in experiments on synthetic and image data, showing that they %
can recover the ground truth factors when the \gls{ima} assumptions are met.

The \textbf{contributions} of this paper can be summarized as follows:
\vspace{-.5 em}
\begin{itemize}
    [nolistsep,leftmargin=*]
    \item we characterize and prove %
    \textit{self-consistency} of \glspl{vae} %
    in the near-deterministic regime (i.e., when the decoder variance tends to zero), justifying its usage in previous works~(\cref{sec:self_const});
    \item we show that under self-consistency, the \gls{elbo} %
    converges to a regularized log-likelihood~(\cref{sec:implicit_constraints}),
    and discuss its possible role as a useful inductive bias in representation learning;
    \item we test the applicability of our theoretical results %
    in experiments on synthetic and image data, and %
    show %
    that \glspl{vae} recover the true latent factors when the \gls{ima} assumptions are met~(\cref{sec:experiments}).%
\end{itemize}

\definecolor{figblue}{HTML}{4A90E2}
\definecolor{figred}{HTML}{D0021B}

\begin{figure}[tb]
    \centering
    	\tikzset {_eibd9pq72/.code = {\pgfsetadditionalshadetransform{ \pgftransformshift{\pgfpoint{0 bp } { 0 bp }  }  \pgftransformscale{1 }  }}}
\pgfdeclareradialshading{_8pif6o2sh}{\pgfpoint{0bp}{0bp}}{rgb(0bp)=(1,0,0);
rgb(0bp)=(1,0,0);
rgb(25bp)=(1,1,1);
rgb(400bp)=(1,1,1)}

\tikzset {_t4fmnmc7s/.code = {\pgfsetadditionalshadetransform{ \pgftransformshift{\pgfpoint{0 bp } { 0 bp }  }  \pgftransformscale{1 }  }}}
\pgfdeclareradialshading{_u79fn3e4q}{\pgfpoint{0bp}{0bp}}{rgb(0bp)=(0.29,0.56,0.89);
rgb(0bp)=(0.29,0.56,0.89);
rgb(25bp)=(1,1,1);
rgb(400bp)=(1,1,1)}
\tikzset{every picture/.style={line width=0.75pt}} %

\begin{tikzpicture}[x=0.75pt,y=0.75pt,yscale=-.825,xscale=.825]
\path  [shading=_8pif6o2sh,_eibd9pq72] (381.5,102.84) .. controls (381.5,90) and (391.91,79.59) .. (404.75,79.59) .. controls (417.59,79.59) and (428,90) .. (428,102.84) .. controls (428,115.68) and (417.59,126.09) .. (404.75,126.09) .. controls (391.91,126.09) and (381.5,115.68) .. (381.5,102.84) -- cycle ; %
 \draw  [color={rgb, 255:red, 0; green, 0; blue, 0 }  ,draw opacity=1 ] (381.5,102.84) .. controls (381.5,90) and (391.91,79.59) .. (404.75,79.59) .. controls (417.59,79.59) and (428,90) .. (428,102.84) .. controls (428,115.68) and (417.59,126.09) .. (404.75,126.09) .. controls (391.91,126.09) and (381.5,115.68) .. (381.5,102.84) -- cycle ; %

\draw  [color={rgb, 255:red, 0; green, 0; blue, 0 }  ,draw opacity=1 ] (392.57,102.84) .. controls (392.57,96.11) and (398.02,90.66) .. (404.75,90.66) .. controls (411.48,90.66) and (416.93,96.11) .. (416.93,102.84) .. controls (416.93,109.56) and (411.48,115.02) .. (404.75,115.02) .. controls (398.02,115.02) and (392.57,109.56) .. (392.57,102.84) -- cycle ;
\draw  [fill={rgb, 255:red, 0; green, 0; blue, 0 }  ,fill opacity=1 ] (401.38,102.84) .. controls (401.38,100.97) and (402.89,99.46) .. (404.75,99.46) .. controls (406.61,99.46) and (408.13,100.97) .. (408.13,102.84) .. controls (408.13,104.7) and (406.61,106.21) .. (404.75,106.21) .. controls (402.89,106.21) and (401.38,104.7) .. (401.38,102.84) -- cycle ;

\path  [shading=_u79fn3e4q,_t4fmnmc7s] (157.59,111.89) .. controls (157.59,96.66) and (165.39,84.32) .. (175.02,84.32) .. controls (184.65,84.32) and (192.45,96.66) .. (192.45,111.89) .. controls (192.45,127.11) and (184.65,139.46) .. (175.02,139.46) .. controls (165.39,139.46) and (157.59,127.11) .. (157.59,111.89) -- cycle ; %
 \draw  [color={rgb, 255:red, 0; green, 0; blue, 0 }  ,draw opacity=1 ] (157.59,111.89) .. controls (157.59,96.66) and (165.39,84.32) .. (175.02,84.32) .. controls (184.65,84.32) and (192.45,96.66) .. (192.45,111.89) .. controls (192.45,127.11) and (184.65,139.46) .. (175.02,139.46) .. controls (165.39,139.46) and (157.59,127.11) .. (157.59,111.89) -- cycle ; %

\draw  [color={rgb, 255:red, 0; green, 0; blue, 0 }  ,draw opacity=1 ] (165.89,111.89) .. controls (165.89,103.91) and (169.98,97.45) .. (175.02,97.45) .. controls (180.06,97.45) and (184.15,103.91) .. (184.15,111.89) .. controls (184.15,119.86) and (180.06,126.33) .. (175.02,126.33) .. controls (169.98,126.33) and (165.89,119.86) .. (165.89,111.89) -- cycle ;
\draw  [fill={rgb, 255:red, 0; green, 0; blue, 0 }  ,fill opacity=1 ] (171.65,111.89) .. controls (171.65,110.03) and (173.16,108.51) .. (175.02,108.51) .. controls (176.89,108.51) and (178.4,110.03) .. (178.4,111.89) .. controls (178.4,113.75) and (176.89,115.26) .. (175.02,115.26) .. controls (173.16,115.26) and (171.65,113.75) .. (171.65,111.89) -- cycle ;

\draw   (103,77) .. controls (123,67) and (128,53) .. (163,51.5) .. controls (198,50) and (237,133) .. (216,151) .. controls (195,169) and (138.5,167.5) .. (118.5,137.5) .. controls (98.5,107.5) and (83,87) .. (103,77) -- cycle ;
\draw   (395.5,50) .. controls (415.5,40) and (466,42) .. (476,70) .. controls (486,98) and (502,117) .. (487,140) .. controls (472,163) and (397.5,169.5) .. (377.5,139.5) .. controls (357.5,109.5) and (375.5,60) .. (395.5,50) -- cycle ;
\draw [color={rgb, 255:red, 208; green, 2; blue, 27 }  ,draw opacity=1 ]   (180.5,107.5) .. controls (219.7,78.1) and (353.98,73.19) .. (396.98,97.48) ;
\draw [shift={(399.5,99)}, rotate = 213.02] [fill={rgb, 255:red, 208; green, 2; blue, 27 }  ,fill opacity=1 ][line width=0.08]  [draw opacity=0] (8.93,-4.29) -- (0,0) -- (8.93,4.29) -- cycle    ;
\draw [color={rgb, 255:red, 74; green, 144; blue, 226 }  ,draw opacity=1 ]   (183.51,119.87) .. controls (242.22,149.73) and (369.16,142.82) .. (401.5,109.5) ;
\draw [shift={(180,118)}, rotate = 29.31] [fill={rgb, 255:red, 74; green, 144; blue, 226 }  ,fill opacity=1 ][line width=0.08]  [draw opacity=0] (8.93,-4.29) -- (0,0) -- (8.93,4.29) -- cycle    ;

\draw (97,172.4) node [anchor=north west][inner sep=0.75pt] (diag)   {$\diag{\! \sigma _{1}^{\gls{encpar}}(\gls{obs})^{2},\! \dotsc , \sigma _{d}^{\gls{encpar}}(\gls{obs})^{2}\!}\! =\! \unmixjacobian\dfrac{1}{\gamma ^{2}} \unmixjacobian^{T}$};
\draw (585,181.9) node [anchor=north west][inner sep=0.75pt] (pxz)  {\gls{pxz}};
\draw (467,181.9) node [anchor=north west][inner sep=0.75pt] (gpxz)  {$\gls{invdec}_{*}[\gls{pxz}]$};
\draw (1.5,181.9) node [anchor=north west][inner sep=0.75pt] (q)   {\gls{q}};

\draw (405.5,55.9) node [anchor=north west][inner sep=0.75pt]    {$\mathcal{X}=\mathbb{R}^{d}$};
\draw (268,85.9) node [anchor=north west][inner sep=0.75pt]  [color={rgb, 255:red, 208; green, 2; blue, 27 }  ,opacity=1 ]  {$\gls{obs}=\gtmix$};
\draw (135.5,106.4) node [anchor=north west][inner sep=0.75pt]    {$\gls{latent}_{0}$};
\draw (267,58.4) node [anchor=north west][inner sep=0.75pt]  [color={rgb, 255:red, 208; green, 2; blue, 27 }  ,opacity=1 ]  {$\mathrm{Decoder}$};

\draw [color={rgb, 255:red, 208; green, 2; blue, 27 }  ,draw opacity=1 ][line width=1.5]  (525.19,115.03) -- (554.95,128.82)(547.91,73.78) -- (525.98,121.15) (550.7,121.34) -- (554.95,128.82) -- (546.5,130.41) (540.44,78.03) -- (547.91,73.78) -- (549.51,82.23)  ;

\draw (486.74,76.44) node [anchor=north west][inner sep=0.75pt]  [font=\footnotesize,color={rgb, 255:red, 208; green, 2; blue, 27 }  ,opacity=1 ,rotate=-359.16]  {$\frac{\partial \gls{dec}}{\partial \gls{latentcomp}_{i}}(\gls{latent}) \ $};
\draw (516.53,133.47) node [anchor=north west][inner sep=0.75pt]  [font=\footnotesize,color={rgb, 255:red, 208; green, 2; blue, 27 }  ,opacity=1 ,rotate=-359.16]  {$\frac{\partial \gls{dec}}{\partial \gls{latentcomp}_{j}}(\gls{latent}) \ $};

\draw (102.5,75.9) node [anchor=north west][inner sep=0.75pt]    {$\mathcal{Z} =\mathbb{R}^{d}$};
\draw (261,144.4) node [anchor=north west][inner sep=0.75pt]  [color={rgb, 255:red, 74; green, 144; blue, 226 }  ,opacity=1 ]  {$\gls{latent}=\unmix$};
\draw (430.5,89.4) node [anchor=north west][inner sep=0.75pt]    {$\gls{dec}( \gls{latent}_{0})$};
\draw (263,117.4) node [anchor=north west][inner sep=0.75pt]  [color={rgb, 255:red, 74; green, 144; blue, 226 }  ,opacity=1 ]  {$\mathrm{Encoder}$};

\draw [color={rgb, 255:red, 74; green, 144; blue, 226 }  ,draw opacity=1 ][line width=1.5]  (54.5,125.6) -- (93.5,125.6)(58.4,90.5) -- (58.4,129.5) (86.5,120.6) -- (93.5,125.6) -- (86.5,130.6) (53.4,97.5) -- (58.4,90.5) -- (63.4,97.5)  ;

\draw (.2,99.4) node [anchor=north west][inner sep=0.75pt]  [font=\footnotesize,color={rgb, 255:red, 74; green, 144; blue, 226 }  ,opacity=1 ]  {$\nabla \invdeccompk[i]$};
\draw (56.5,131.9) node [anchor=north west][inner sep=0.75pt]  [font=\footnotesize,color={rgb, 255:red, 74; green, 144; blue, 226 }  ,opacity=1 ]  {$\nabla \invdeccompk[j]$};

\draw [-latex] (q) -- node [above]{Cov}  (diag);
\draw [-latex] (gpxz) -- node [above]{Cov}(diag);
\draw [-latex] (pxz) -- node [above]{\gls{invdec}} (gpxz);

\end{tikzpicture}
    \caption{\small \looseness-1
    {\bf Modeling choices in \glspl{vae} promote  \textit{\acrfull{ima}}~\citep{gresele_independent_2021}.} %
    We assume a Gaussian \gls{vae}~\eqref{eq:gauss_enco_deco}, and 
    prove that in the near-deterministic regime
    the mean \textcolor{figblue}{encoder} approximatetely inverts the mean \textcolor{figred}{decoder}, $\gls{invdec} \!\!\approx\!\! \gls{dec}{}^{-1}$ ({\em self-consistency},~\cref{prop:selfconsist}). %
    \textbf{Bottom:} Closing the gap requires
    matching the covariances of the variational (LHS, \gls{q}) and the true posterior  (RHS, approximated by %
    $\gls{invdec}_*[\gls{pxz}]$, \cf~\cref{sec:implicit_constraints} for details).
    Under self-consistency, an \textcolor{figblue}{encoder} with diagonal covariance enforces a row-orthogonal \textcolor{figblue}{encoder} Jacobian \unmixjacobian---or equivalently, a column-orthogonal  \textcolor{figred}{decoder} Jacobian \gtmixjacobian.
    This regularization was termed \acrfull{ima}~\cite{gresele_independent_2021} and shown to be
    beneficial for learning the true latent factors. The connection  elucidates unintended benefits of using the \gls{elbo} for representation learning.
    } 
    \label{figure:fig1}
\end{figure}

\section{Background}
\label{sec:bg}

We will connect two unsupervised learning objectives: the \gls{elbo} in \glspl{vae}
and the \gls{ima}-regularized %
log-likelihood. Both stem from \glspl{lvm} with %
latent variables \gls{latent} distributed according to 
a \textit{prior} 
\gls{pz}, and a mapping from \gls{latent} to observations \gls{obs} given by a conditional generative model %
\gls{pxz}. 
\textbf{\acrlong{vae}s.}
Optimizing the data likelihood \gls{px} in deep \glspl{lvm}---\ie, finding decoder parameters \gls{decpar} maximizing $\int \gls{pxz}\gls{pz}d\gls{latent}$---is intractable in general, so approximate objectives are required. 
Variational approximations~\citep{struwe2000variational}
replace the true posterior \gls{pzx} by an approximate one, called the \textit{variational posterior} \gls{q}, which is a stochastic mapping $\gls{obs}\mapsto\gls{latent}$ with parameters \gls{encpar}. This allows to evaluate a tractable \acrfull{elbo}~\cite{kingma_auto-encoding_2014, rezende2014stochastic} of the model's log-likelihood that can be defined as
\begin{equation}\label{eq:elbo}
    \gls{elbo}(\gls{obs},\decpar,\gls{encpar})=\expectation{\gls{q}}\brackets{\log \gls{pxz}}-\kl{\gls{q}}{\gls{pz}}.
\end{equation}
The two terms in \eqref{eq:elbo} are sometimes interpreted as a reconstruction term measuring the sample quality of the decoder and a regularizer---the \gls{kld} between the prior and the encoder~\cite{kingma_introduction_2019}. 
The variational approximation trades off computational efficiency with a difference \wrt the exact log-likelihood, %
which is expressed alternatively as (see \cite{doersch_tutorial_2021, kingma_introduction_2019} and \cref{app:complement})
\begin{equation} \label{eq:elbo_kl_truepost}
    \gls{elbo} (\gls{obs},\decpar,\gls{encpar}) = \log \gls{px} -\kl{\gls{q}}{\gls{pzx}},
\end{equation}
where the \gls{kld} between variational and true posteriors %
characterizes the \textit{gap}: 
if the variational family of \gls{q} does not include \gls{pzx}, the \gls{elbo} will be strictly smaller than %
$\log\gls{px}$. 

\looseness-1 \glspl{vae}~\citep{kingma_auto-encoding_2014} %
rely on the variational approximation in~\eqref{eq:elbo} 
to train deep \glspl{lvm} where neural networks parametrize the {\em encoder} \gls{q} and the {\em decoder} \gls{pxz}. 
A common modeling choice constrains the variational family of \gls{q} to a factorized Gaussian with posterior means \gls{muk} and variances \gls{sigmak} 
for the $k^{th}$ factor $\gls{latentcomp}_{k}|\gls{obs}$, %
and with a diagonal covariance \gls{var_cov} %
; and  the decoder to %
a factorized Gaussian, conditional on \gls{latent}, with mean \gtmix\xspace and an isotropic covariance in \gls{obsdim} dimensions, %
\begin{equation} \label{eq:gauss_enco_deco}
    \gls{latentcomp}_{k}|\gls{obs}\sim\mathcal{N}(\gls{muk},\gls{sigmak})\,;
    \qquad
    \gls{obs}|\gls{latent} \sim \mathcal{N}\parenthesis{\gtmix,\gamma^{-2}\gls{identity}}.%
\end{equation}

\textbf{The deterministic limit of \glspl{vae}.}
The stochasticity of \glspl{vae} makes it nontrivial to relate them to generative models with deterministic decoders such as \acrlong{ica} (see paragraph below), though postulating a deterministic regime (where the decoder precision \gsq\xspace becomes infinite) is possible. %
Interestingly, \citet{nielsen_survae_2020} explored this deterministic limit and argued that \textit{deterministic} \glspl{vae} optimize an exact log-likelihood, similar to normalizing flows~\citep{rezende_variational_2016,papamakarios_normalizing_2021}. Normalizing flows model arbitrarily complex distributions using a simple base distribution \gls{pz} and nonlinear, \textit{deterministic and invertible} transformations \gls{dec}. Through a change of variables,\footnote{note that in normalizing flows the change of variables is usually expressed in terms of $\gls{invdec}=\gls{dec}{}^{-1}$} %
the likelihood of the original variables becomes
\begin{align}\label{eq:change_of_var}
    \log \gls{px} &= \log\gls{pz} - \logabsdet{\gtmixjacobian}.
\end{align}
The comparison is nontrivial, since \glspl{vae} contain an encoder and a decoder, whereas normalizing flows consist of a single architecture.  \citet{nielsen_survae_2020} made this analogy by resorting to what we call a \textit{self-consistency assumption}, stating that the \gls{vae} encoder inverts the decoder. We define self-consistency in the \textit{near-deterministic} regime: as the decoder variance goes to zero, i.e. $\gamma\to +\infty$.

\begin{definition}[(Near-deterministic) self-consistency]\label{def:selfcons} \looseness-1
For a fixed \gls{decpar}, assume that mean decoder $\gls{dec}$ is invertible with inverse $\gls{invdec}$, and that a map associates each choice of decoder parameters and observation $( \gls{decpar},\gamma, \gls{obs})$ to an encoder parameter $( \gls{decpar},\gamma, \gls{obs}) \mapsto \widehat{\gls{encpar}}( \gls{decpar},\gamma, \gls{obs})$%
, we say the \gls{vae} is self-consistent whenever
    \begin{eqnarray}
        \gls{muopt} \to \gls{invdec}(\gls{obs})\quad \mbox{and}
        \quad \gls{sigmaopt} \to  \boldsymbol{0}\,\,\mbox{, as } \gamma \to +\infty\,.
    \end{eqnarray}
\end{definition}
\looseness-1 
The encoder parameter map $\widehat{\gls{encpar}}$ reflects the choice of a particular encoder model for each $( \gls{decpar},\gamma)$ pair:\footnote{both the \gls{elbo} and $\widehat{\gls{encpar}}$ depends on the decoder precision $\gamma$: we will omit this in the following for simplicity} in~\cref{sec:self_const}, we study this problem by introducing and justifying a particular choice for \gls{encparopt} (see also~\cref{sec:limitations}). %
This self-consistency assumption %
appears central to deterministic claims~\cite{nielsen_survae_2020, kumar_implicit_2020}, %
but has not yet been proven. In particular, \citet{nielsen_survae_2020}  %
assume that taking the deterministic limit %
is well-behaved. However, \glspl{vae}' \textit{near-deterministic} properties
have not been investigated analytically.

\textbf{Identifiability, \acrshort{ica}, and \gls{ima}.}
\iftrue
\gls{ica}~\citep{comon1994independent, hyvarinen_independent_2001} models observations as the {\em mixing} of a latent vector \gls{latent} with independent components through a deterministic function $\boldsymbol{f}$, \ie, $\gls{obs} = \boldsymbol{f}(\gls{latent}), p_0(\gls{latent})= \prod_i p_0(\gls{latentcomp}_i)$.\footnote{the conditional distribution $p(\gls{obs}|\gls{latent})$ is therefore degenerate} 
In \acrshort{ica} the focus is on defining conditions under which the original latent variables can be recovered from observations---i.e., the
model is ``identifiable by design''~\cite{hyvarinen_nonlinear_2019}.
The goal %
is to learn an unmixing $\gls{invdec}$  such that the recovered components $\gls{rec_s} =\unmix$ are estimates of the true ones up to some ambiguities (e.g., permutation and element-wise nonlinear transformations). Unfortunately, the nonlinear problem is nonidentifiable without further constraints~\cite{darmois1951analyse, hyvarinen_nonlinear_1999}: any two observationally equivalent models can yield components which are arbitrarily entangled%
, thus making recovery of the ground truth factors impossible. 
This is typically shown by suitably constructed counterexamples~\cite{hyvarinen_nonlinear_1999, locatello_challenging_2019}, and it was argued to imply %
impossibility statements for unsupervised disentanglement~\cite{locatello_challenging_2019, tschannen_mutual_2020}. Identifiability can be recovered when %
{\em auxiliary} variables~\cite{hyvarinen_nonlinear_2019, gresele_incomplete_2019, khemakhem_variational_2020, halva_disentangling_2021} are available, or exploiting a temporal structure in the data~\cite{hyvarinen_nonlinear_2017, halva_hidden_2020}.

Restrictions on the mixing function class (e.g., linear~\cite{comon1994independent}) are another possibility to recover identifiability~\cite{hyvarinen_nonlinear_1999, zhang2008minimal}.
Recently,~\citet{gresele_independent_2021} proposed restricting the function class 
by taking inspiration from the 
{\em principle of independent causal mechanisms}~\cite{peters_elements_2018}, in an approach termed \acrfull{ima}.
\acrshort{ima} postulates that the latent components influence the observations ``independently'', where influences correspond to the partial derivatives $\nicefrac{\partial\gls{dec}}{\partial\gls{latentcomp}_k}$, and their non-statistical independence amounts to an orthogonality condition.
While full identifiability has not been proved for this model class%
, it was shown to rule out classical families of spurious solutions used as counterexamples to identifiability of unconstrained non-linear ICA \citep{gresele_independent_2021,buchholz_function_2022}. Moroever, \citet{buchholz_function_2022} further demonstrated local identifiability of this function class. Also, IMA constraints were empirically shown \citep{gresele_independent_2021,sliwa_probing_2022} to help recover the ground truth 
through
regularization of the %
log-likelihood in~\eqref{eq:change_of_var} with an objective $\imaloss:= \log \gls{px} - \lambda \cdot \cima $%
, where $\lambda>0$ %
and the regularization term
\cima and its expectation %
\Cima
are given by
\begin{align}\label{eq:ima_objective}
    \cima =\! 
    \sumk[d]\log\norm{\tfrac{\partial\gls{dec}}{\partial\gls{latentcomp}_k}\parenthesis{\gls{latent}}}\!-\!\logabsdet{\gtmixjacobian}
    \!; \quad \Cima\! = \expectation{\gls{pz}}\! \brackets{\cima}\!,
\end{align} 
\looseness-1 
and termed \textit{local} (resp.\ {\it global}) \gls{ima} contrast. %
When \gls{dec} is in the \gls{ima} function class (i.e., \Cima vanishes), the objective is equal to the log-likelihood; otherwise, it lower bounds it.
\fi
\section{Theory}
\label{sec:theory}

Our theoretical analysis assumes that all the model's defining densities ($\gls{pz}$, $\gls{q}$ and $\gls{pxz}$) are factorized. We also assume a Gaussian decoder, %
matching common modeling practice in \glspl{vae}. %
 \begin{assum}[Factorized VAE class with isotropic Gaussian decoder and log-concave prior] \label{assum:VAE}%
    \looseness-1 We are given a fixed latent prior and three parameterized classes of $\,\RR^d \to \RR^d$ mappings: the mean decoder class $\gls{decpar}\mapsto \gls{dec}$, and the mean and standard deviation encoder classes, $\gls{encpar}\mapsto \encmean$ and $\gls{encpar}\mapsto \encstd$ s.t. %
    \begin{enumerate}[label=(\roman*),nolistsep]
        \item  %
        $\gls{pz}\sim \prod_k m(\gls{latentcomp}_k)$, with $m$ being smooth and fully supported on $\RR$, having bounded non-positive  second-order, and bounded third-order logarithmic derivatives;%
        \item the encoder and decoder are of the form  in~\eqref{eq:gauss_enco_deco}, with isotropic decoder covariance $\nicefrac{1}{\gsq}\gls{identity}$;
        \item the variational mean and variance encoder classes are universal approximators; \item for all $\gls{decpar}$, $f^{\gls{decpar}}:\RR^d\to \RR^d$ is a bijection with inverse $\gls{invdec}$, and both are $C^2$ with bounded first and second order derivatives.  %
    \end{enumerate}
\end{assum}
Crucially, {\em both the mean encoder and the mean decoder can be nonlinear}. %
Moreover, the family of log-concave priors contains the commonly-used Gaussian distribution as a special case.
We study the {\em near-deterministic decoder} regime of such models, where
$\gamma\! \to\! +\infty$. 
This regime is expected to model data generating processes with vanishing observation noise well---in line with the typical \acrshort{ica} setting%
---and is commonly considered in theoretical analyses of \glspl{vae}, \eg, in~\cite{nielsen_survae_2020}
(which additionally assumes quasi-deterministic encoders), and in~\cite{lucas_dont_2019, kumar_implicit_2020}. 
Unlike~\citet{nielsen_survae_2020}, we consider a large but finite $\gamma$, not {\em at} the limit $\gamma\!=\!\infty,$ where the decoder is fully deterministic. In fact, for any large but finite $\gamma$, the objective is well-behaved and amenable to theoretical analysis, while the KL-divergence is undefined in the 
deterministic setting. %
The requirement in assumption \textit{(iv)} deviates from common practice in \glspl{vae}---where observations are typically higher-dimensional---but it allows to connect \glspl{vae} and exact likelihood methods such as normalizing flows~\cite{nielsen_survae_2020} (see also~\cref{sec:limitations}).

Due to considering $\gamma \to +\infty$, results are stated in the following ``big-O'' notation for an integer $p$:%
\[
f(\gls{obs},\gamma) = g(\gls{obs},\gamma) + O_{\gamma\to +\infty}(\nicefrac{1}{\gamma^p}) \iff \gamma^p\|f(\gls{obs},\gamma) - g(\gls{obs},\gamma)\| \mbox{ is bounded as } \gamma\to +\infty\,.
\]

\subsection{Self-consistency}
\label{sec:self_const}
In this section, we will prove a {\em self-consistency} result in the near-deterministic regime. This rests on characterizing optimal variational posteriors (i.e., those minimizing the \gls{elbo} gap \wrt the likelihood) for a \textit{particular point}
\gls{obs} and \textit{fixed decoder parameters }\gls{decpar}. Based on \eqref{eq:elbo_kl_truepost}, any associated optimal choice of encoder parameters satisfies %
\begin{equation}\label{eq:minKLopt}
    \gls{encparopt} (\gls{obs},\gls{decpar}) \in \arg\max_{\gls{encpar}} \, \gls{elbo}(\gls{obs};\gls{decpar},\gls{encpar})=   \arg\min_{\gls{encpar}} \, \kl{\gls{q}}{\gls{pzx}}\,.
\end{equation}
We call \textit{self-consistent \gls{elbo}} the resulting achieved value, denoted as
\begin{equation}
    \gls{elbo}^*(\gls{obs};\gls{decpar}) = \gls{elbo} (\gls{obs};\gls{decpar},\gls{encparopt}(\gls{obs},\gls{decpar})) \,.
\end{equation}
The expression in~\eqref{eq:minKLopt} corresponds to a problem of \textit{information projection}~\cite{cover_elements_1991, murphy2012machine}
of $\gls{pzx}$ onto the set of factorized Gaussian distributions. This means that given a variational family, we search for the optimal \gls{q} to minimize the \gls{kld} to \gls{pzx}. %
While such information projection problems are well studied for closed convex sets where they yield a unique minimizer \citep{csiszar2003information},  the set projected onto in our case is not convex (convex combinations of arbitrary Gaussians are not Gaussian), making this problem of independent interest. After establishing upper and lower bounds on the KL divergence (exposed in~\cref{prop:KLlowerLipschitz}-\ref{prop:KLupperLipschitz} in~\cref{appendix:kl_div_bounds})%
, we obtain the  
 following self-consistency result. 
\begin{restatable}{prop}{propselfconsist}[Self-consistency of near-deterministic \glspl{vae}]\label{prop:selfconsist}
    Under Assumption~\ref{assum:VAE}, for all \gls{obs}, \gls{decpar}, as $\gamma\to +\infty$, there exists at least one global minimum solution of %
    (\ref{eq:minKLopt}). %
    These solutions satisfy 
    \begin{equation}\label{eq:selfcons}
        \gls{muopt}=\unmix%
        +O(\nicefrac{1}{\gamma})
    \quad \mbox{and} \quad
      \gls{sigmaoptk} =  O(\nicefrac{1}{\gamma^2})\,,\ \mbox{for all }k\,.  %
    \end{equation}
\end{restatable}
\cref{prop:selfconsist} states that minimizing the \gls{elbo} gap (equivalently, maximizing the \gls{elbo}) \wrt the encoder parameters $\gls{encpar}$ implies in the limit of large $\gamma$ that the encoder's mean \gls{mu} %
tends to $\gls{invdec}(\gls{obs})$, the image of $\gls{obs}$ by the \textit{inverse} decoder. We can interpret this as the decoder ``inverting'' the encoder.
Additionally, the variances of the encoder will converge to zero.

Let us now consider the relevance of this result for training \glspl{vae}, \ie, maximizing the expectation of the \gls{elbo} for an observed distribution \gls{pdata}. While  maximization \textit{only} \wrt \gls{encpar}  in~\eqref{eq:minKLopt} {does not match common practice}---which is learning  \gls{decpar} and \gls{encpar} \textit{jointly}---it models this process in the limit of large-capacity encoders. Indeed, in this case,  \eqref{eq:minKLopt} can be solved for each \gls{obs} as a separate learning problem, which entails that the following inequality is satisfied for any parameter choice
\begin{multline}\label{eq:self_cons_elbo}
    \expectation{\gls{obs}\sim \gls{pdata}}\brackets{\gls{elbo} (\gls{obs};\gls{decpar},\gls{encpar})}=\medint\int \gls{pdata}\gls{elbo} (\gls{obs};\gls{decpar},\gls{encpar}) d\gls{obs}\\ 
    \leq \medint\int \gls{pdata}\gls{elbo} (\gls{obs};\gls{decpar},\gls{encparopt}(\gls{obs},\gls{decpar})) d\gls{obs}
    =: 
     \expectation{\gls{obs}\sim \gls{pdata}}\brackets{\gls{elbo}^* (\gls{obs};\gls{decpar})}\,.
 \end{multline}
  The joint optimization of encoder and decoder parameters thus reduces to optimizing the subset of pairs $(\gls{decpar},\gls{encparopt}(\gls{obs},\gls{decpar}))$, and is equivalent to optimizing the expected self-consistent \gls{elbo}, that is
  \begin{equation}\label{eq:optimreduction}
      \underset{\gls{decpar},\gls{encpar}}{\mbox{maximize}}\, \expectation{\gls{obs}\sim \gls{pdata}}\brackets{\gls{elbo} (\gls{obs};\gls{decpar},\gls{encpar})}
    \iff
           \underset{\gls{decpar}}{\mbox{maximize}}\, \expectation{\gls{obs}\sim \gls{pdata}}\brackets{\gls{elbo}^* (\gls{obs};\gls{decpar})}
  \end{equation}
\looseness-1 This problem reduction is aligned with the original purpose of the \gls{elbo}:  building a tractable but optimal likelihood approximation. Namely, (i) $\gls{elbo}^*$ depends on the same parameters as the likelihood ($\gls{obs}$, $\gamma$ and $\decpar$)%
, (ii) its gap $\kl{\gls{q}}{\gls{pzx}}$ is minimal. 
The problem reduction of \eqref{eq:optimreduction} %
allows us to compare the optimality of different decoders and \cref{prop:selfconsist} helps addressing the case of near-deterministic decoders.

\subsection{Self-consistent \gls{elbo}, \gls{ima}-regularized log-likelihood and identifiability of VAEs}%
\label{sec:implicit_constraints}
We want to investigate how the choice of \gls{q} and \gls{pxz}
implicitly regularizes the Jacobians of their means \gls{mu} and \gtmix 
in the near-deterministic regime. Exploiting %
self-consistency, we are able to precisely characterize how this happens:
we formalize this in \cref{prop:vae_ima}. %

\begin{figure}[tb]
    \centering
    \includesvg[]{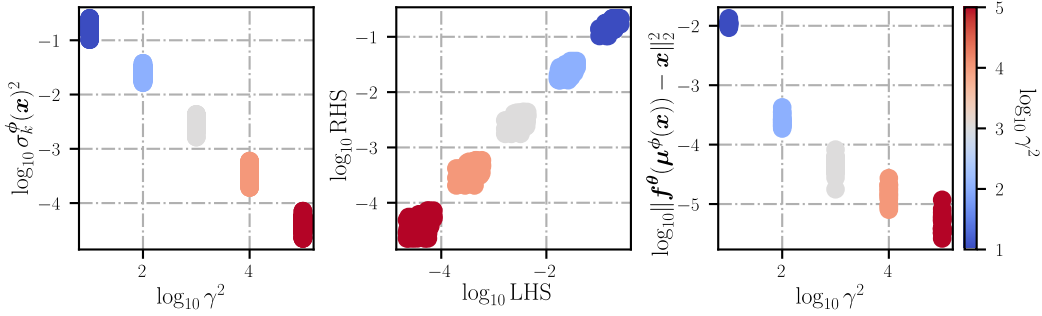}
    \caption{Self-consistency (\cref{prop:selfconsist}) in \gls{vae} training,  on a log-log plot, \cf \ref{subsec:exp_self_cons} for details. \textbf{Left}: convergence of \gls{sigmaoptk} to ${0}$;  \textbf{Center:} connecting  \gls{sigmaoptk}, \gsq, and the column norms of the decoder Jacobian via LHS and RHS of (\ref{eq:optl_sigma});   \textbf{Right:} convergence of \gls{muopt} to \unmix}
    \label{figure:self_cons}
\end{figure}

\begin{restatable}{theorem}{theoremvaeima}[\glspl{vae} with a near-deterministic decoder approximate the \gls{ima} objective]\label{prop:vae_ima}
    Under \autoref{assum:VAE}, the variational posterior satisfies
    \begin{equation}
      \gls{sigmaoptk}=\parenthesis{-\frac{d^2\log p_0}{d\gls{latentcomp}_k^2}(g^{\decpar}_k(\gls{obs}))+\gsq\norm{\cols{\mixjacobianbis{\gls{invdec}(\gls{obs})}}} ^{2}}^{-1} +O(\nicefrac{1}{\gamma^3})\,,
    \label{eq:optl_sigma}
    \end{equation}
    and the self-consistent \gls{elbo}~\eqref{eq:self_cons_elbo}
    approximates the \gls{ima}-regularized log-likelihood~\eqref{eq:ima_objective}:%
    \begin{align}\label{eq:imaconv}
        \gls{elbo}^* (\gls{obs};\gls{decpar})
        &=
        \log \gls{px} - \gls{cima_local}(\gls{dec}, \gls{invdec}(\gls{obs}))
        +O_{\gamma \to \infty}\left(\nicefrac{1}{\gamma^2}\right).
    \end{align}
\end{restatable}
Proof is in~\cref{sec:app_proofs}. Below, we provide a qualitative argument on the interplay between distributional assumptions in the \gls{vae} and implicit constraints on the decoder's Jacobian and its inverse%
. %

{\bf Modeling assumptions implicitly regularize the %
mean decoder class \gls{dec} under self-consistency.}
In the near deterministic regime, $\gls{px}$ gets close to the pushforward distribution of the prior by the mean decoder $\gls{dec}_*\brackets{\gls{pz}}$, which can be used to show that the true posterior ${\gls{pzx}=\gls{pxz}\gls{pz}/\gls{px}}$ is approximately the pushforward through the inverse mean decoder $\gls{invdec}_*\brackets{\gls{pxz}}$ (see \cref{app:complement} for more details). %
If we select a given latent $\gls{latent}_0$ and denote its image by $\gtmix[\gls{latent}_0]$ , then we can locally linearize \gls{invdec} by its Jacobian $\gls{jacobian}_{\gls{invdec}}=\unmixjacobian[\gtmix[\gls{latent}_0]]$, yielding a Gaussian for the pushforward distribution $\gls{invdec}_*\brackets{\gls{pxz}}$ with covariance $\nicefrac{1}{\gsq}\gls{jacobian}_{\gls{invdec}}\transpose{\gls{jacobian}_{\gls{invdec}}}$.
As the sufficient statistics of a Gaussian are given by its mean and covariance, the structure of the posterior covariance \gls{var_cov} (which is by design diagonal, \cf~\eqref{eq:gauss_enco_deco}) is crucial for minimizing the gap in~\eqref{eq:elbo_kl_truepost}. Practically, this implies that in the zero gap limit, the covariances of \gls{q} and \gls{pzx} should match, \ie,
$\nicefrac{1}{\gsq}\gls{jacobian}_{\gls{invdec}}\transpose{\gls{jacobian}_{\gls{invdec}}}$ will be diagonal with entries \gls{sigmak}
and therefore
$\gls{jacobian}_{\gls{invdec}}$
has orthogonal rows. We can express the decoder Jacobian via the inverse function theorem as $\mixjacobianbis{\gls{latent}_0}=\unmixjacobian[\gtmix[\gls{latent}_0]]^{-1}$. As the inverse of a row-orthogonal matrix has orthogonal columns, \gls{dec} satisfies the \gls{ima} principle.
Additionally, we can relate the variational posterior's variances to the column-norms of $\gls{jacobian}_{\gls{dec}}$ as $\gls{sigmak}=\nicefrac{1}{\gsq}%
    {\Vert\cols{\mixjacobianbis{\gls{latent}_0}}\Vert}^{-2}$, as predicted by~\eqref{eq:optl_sigma}.

Our argument indicates that minimizing the gap between the \gls{elbo} and the log-likelihood encourages column-orthogonality in $\gls{jacobian}_{\gls{dec}}$ by matching the covariances of \gls{q} and $\gls{invdec}_*\brackets{\gls{pxz}}$. When $\gls{q}\!=\!\gls{pzx}$, the gap is closed; %
this is only possible if the decoder is in the \gls{ima} class, for which \gls{cima_local} vanishes and %
 the \gls{elbo} {\em tends to an exact log-likelihood}. To the best of our knowledge, we are the first to prove this for nonlinear functions, extending related work for linear \glspl{vae}%
~\cite{lucas_dont_2019}.

\looseness-1
\textbf{Implications for identifiability of \glspl{vae}.} 
While previous works argued that the \gls{vae} objective favors decoders with a column-orthogonal Jacobian~\cite{rolinek_variational_2019, kumar_implicit_2020}, they did not exactly characterize how: our result shows that the self-consistent \gls{elbo} tends to a regularized log-likelihood, where the regularization term \gls{cima_local} explicitly enforces this (soft) constraint. %
Thus, it possibly explains why \glspl{vae} are successful in learning disentangled representations: namely, the \gls{ima} function class provably rules out certain spurious solutions for nonlinear \gls{ica}~\cite{gresele_independent_2021}, and the \gls{ima}-regularized log-likelihood was empirically shown to be beneficial in recovering the true latent factors. Thus, we speak about \textit{embracing the gap}, as its functional form equips \glspl{vae} with a useful inductive bias. While the \gls{ima} function class has not yet been shown to be  identifiable in the classical sense\,
such results exist for special cases
such as conformal maps ($d=2$~\cite{hyvarinen_nonlinear_1999}, generalized by the very recent work in~\cite{buchholz_function_2022}), isometries~\cite{horan_when_2021} and
for closely-related unsupervised nonlinear \gls{ica} models~\cite{zheng2022identifiability}. Moreover, \citet{buchholz_function_2022} demonstrate a \textit{local} form of identifiability for the \gls{ima} function class. 
In the following, we empirically corroborate that \glspl{vae}: 1) recover the ground truth sources when the mixing satisfies \gls{ima}, and thereby 2) achieve unsupervised disentanglement.

\section{Experiments}
\label{sec:experiments}

\looseness-1
Our experiments serve three purposes: 1) demonstrating that self-consistency holds in practice (\cref{subsec:exp_self_cons}); 2) showing the relationship of the self-consistent $\gls{elbo}^*$, the \gls{ima}-regularized and unregularized log-likelihood objectives (\cref{subsec:exp_elbo_ima_likelihood}); and 3) providing empirical evidence that the connection to the \gls{ima} function class in VAEs can lead to success in learning disentangled representations~(\cref{subsec:exp_col_gamma_dis}). More details are provided in~\cref{sec:app_exp}.

\subsection{Self-consistency in practical conditions}
\label{subsec:exp_self_cons}

\textbf{Experimental setup.} We use a 3-layer \gls{mlp} with smooth Leaky ReLU nonlinearities~\citep{gresele_relative_2020} and orthogonal weight matrices---which intentionally does not belong to the \gls{ima} class, as our results are more general. The 60,000 source samples are drawn from a standard normal distribution and fed into a \gls{vae} composed of a 3-layer \gls{mlp} encoder and decoder with a Gaussian prior. We use 20 seeds for each  $\gsq\xspace \in \braces{\expnum{1}{1};\expnum{1}{2};\expnum{1}{3};\expnum{1}{4};\expnum{1}{5}}$.\\
\textbf{Results.}
\cref{figure:self_cons} summarizes our results, featuring the \textit{logarithms} on each axes. The \textbf{left} plot shows that the posterior variances \gls{sigmak} converge to zero with a $\nicefrac{1}{\gsq}$ rate, as predicted by \eqref{eq:selfcons}. The \textbf{center} plot shows that the expression for \gls{sigmak} corresponds to \eqref{eq:optl_sigma} in the optimum of the \gls{elbo} by comparing both sides of the equation. The \textbf{right} plot shows approximate convergence of the mean encodings \gls{muopt} to $\gls{invdec}(\gls{obs})$ with a $\nicefrac{1}{\gamma}$ rate (see~\cref{sec:limitations}). As \gls{dec} is not guaranteed to be invertible, we use instead the \textit{optimal} encoder and decoder parameters to compare $\gls{dec}(\gls{muopt})$ to \gls{obs}.

\begin{figure}[tb]
	\centering
	\includesvg[]{figures/moebius_mcc_cima.svg}
	\caption{\textbf{Left:} \gls{cima_local} and \gls{mcc} for 3-dimensional Möbius mixings \textbf{Right:} \acrshort{mcc} depending on the \textit{volume-preserving linear map's}
	\gls{cima_local} ($\gsq=\expnum{1}{5})$}
	\label{figure:mcc_vs_cima_vs_gamma}
\end{figure}

\subsection{Relationship between $\gls{elbo}^*$, \gls{ima}-regularized, and unregularized log-likelihoods}
\label{subsec:exp_elbo_ima_likelihood}

\begin{wrapfigure}{r}{6.5cm}
\centering
    \includesvg[width=11.5cm,keepaspectratio]{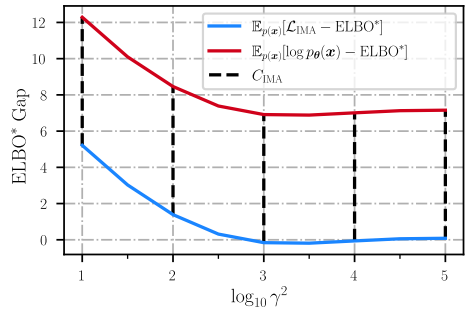}
    \caption{Comparison of the $\gls{elbo}^*$, the \gls{ima}-regularized and unregularized log-likelihoods over different \gsq. Error bars are omitted as they are orders of magnitudes smaller}
    \label{figure:ima_elbo_likelihood}
\end{wrapfigure} 

\looseness-1
\textbf{Experimental setup.} We use an \gls{mlp} \gls{dec} with square upper-triangular weight matrices and invertible element-wise nonlinearities to construct a mixing not in the \gls{ima} class~\citep{gresele_independent_2021} and fix the \gls{vae} decoder to the ground truth such that \eqref{eq:change_of_var} gives the true data log-likelihood. This way, we ensure that the unregularized and \gls{ima}-regularized log-likelihoods differ and make the claim of \citet{nielsen_survae_2020} comparable to ours. With a fixed decoder, the $\gls{elbo}^*$ depends only on \gls{encpar}, therefore we only train the encoder with \gsq\xspace values from \brackets{\expnum{1}{1};\expnum{1}{5}} (5 seeds each).\\
\textbf{Results.} \cref{figure:ima_elbo_likelihood} compares the difference of the estimate of $\gls{elbo}^*$ and the unregularized/\gls{ima}-regularized log-likelihoods after convergence over the whole dataset. As the decoder and the data are fixed, $\log \gls{px}$ and \gls{cima_global} will not change during training, only the $\gls{elbo}^*$ does. The figure shows that as $\gamma\!\!\to\!\!+\infty$, $\gls{elbo}^*$ approaches \imaloss, as predicted by \cref{prop:vae_ima}, and not  $\log\gls{px}$, as stated in \citep{nielsen_survae_2020}---the difference is \gls{cima_global}. 
\subsection{Connecting the \gls{ima} principle, \gsq, and disentanglement}
\label{subsec:exp_col_gamma_dis}

\textbf{Experimental setup (synthetic).} We use 3-dimensional conformal mixings (\ie, the Möbius transform~\citep{phillips1969liouville}) from the \gls{ima}  class with \textit{uniform} ground-truth and prior distributions. Our results quantify the relationship of the decoder Jacobian's \gls{ima}-contrast and identifiability with \gls{mcc}~\citep{hyvarinen_unsupervised_2016} and show how this translates to disentanglement---we note that \gls{mcc} was already used to quantify disentanglement~\citep{zimmermann_contrastive_2021,klindt_towards_2021}. To determine whether a mixing from the \gls{ima} class is beneficial for disentanglement, we apply a volume-preserving linear map after the Möbius transform (using 100 seeds) to make $\gls{cima_local}\neq 0$. Other parameters are the same as in \cref{subsec:exp_self_cons}, with the exception of picking the best $\gsq=\expnum{1}{5}$. \\
\textbf{Results (synthetic).} The \textbf{left} of \cref{figure:mcc_vs_cima_vs_gamma} empirically demonstrates the benefits of optimizing the \gls{ima}-regularized log-likelihood. By increasing \gsq,   \gls{mcc} increases, while  \gls{cima_local} decreases, suggesting that \glspl{vae} in the near-deterministic regime encourage disentanglement by enforcing the \gls{ima} principle. The \textbf{right} plot shows that when the mixing is outside the \gls{ima} class, \gls{mcc} decreases, corroborating the benefits of \gls{ima} class mixings for disentanglement. \\
\textbf{Experimental setup (image).} \looseness-1 We train a \gls{vae} (not \betavae) with a factorized Gaussian posterior and Beta prior on a Sprites image dataset generated using the spriteworld renderer~\citep{spriteworld19} with a Beta ground truth distribution. Similar to~\citep{jack_brady_isprites_2020}, we use four latent factors, namely, \textit{x- and y-position, color and size}, and omit factors that can be problematic, such as shape (as it is discrete) and rotation (due to symmetries)~\citep{rolinek_variational_2019,klindt_towards_2021}. Our choice is motivated by~\citep{horan_when_2021,donoho_image_2005} showing that this data-generating process may approximately satisfy the \gls{ima} principle.\\
\textbf{Results (image).} The \textbf{left} of \cref{figure:dsprites} indicates that \glspl{vae} can learn the true latent factors and \gls{mcc} is \textit{anticorrelated} with \gls{cima_local}, reinforcing the hypothesis that the data-generating process belongs to the \gls{ima} class. The \textbf{center} plot compares estimated and true latent factors from the best model (scaling and permutation indeterminacies are removed), whereas the \textbf{right} plot shows the corresponding latent interpolations---thus,  connecting identifiability (measured by \gls{mcc}) to disentanglement.

\begin{figure}[tb]
    \hspace{-0.5em}
	\begin{subfigure}[h]{0.4\textwidth}
		\centering	
		\includesvg[width=1.725\textwidth, height=2.35\textheight]{figures/dsprites_mcc_cima.svg}
	\end{subfigure}
    \hspace{0.5em}
	\begin{subfigure}[h]{0.59\textwidth}
		\vspace{-1.6em}
		\includegraphics[height=7em%
		]{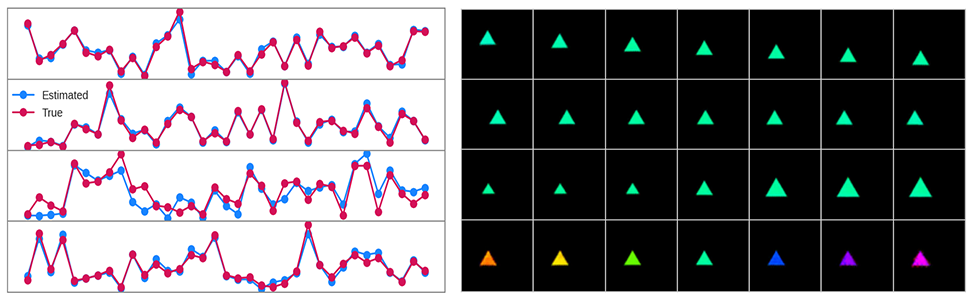}
	\end{subfigure}
	\caption{\textbf{Left:} \gls{cima_local} and \gls{mcc} for Sprites~\citep{spriteworld19} during  training $(\gsq\!=\!1)$;  \textbf{Center:} true and estimated latent factors for the best trained \gls{vae} on Sprites; \textbf{Right:} the corresponding latent interpolations and \gls{mcc} values (from top to bottom): $y$- ($0.989$), $x$-position ($0.996$), scale ($0.933$), and color ($0.989$)}
	\label{figure:dsprites}
\end{figure}

\section{Limitations}
\label{sec:limitations}
\textbf{The near-deterministic regime.} %
Our theory relies on $\gamma\!\!\to\!\!+\infty$; this is the regime where posterior collapse may be avoided~\citep{lucas_dont_2019}, and where calculating the reconstruction loss may be possible even without sampling~\citep{kumar_implicit_2020}.
However, in practice it may be unclear when \gsq is large enough. This seems to be problem-dependent~\citep{rolinek_variational_2019,lucas_dont_2019}, and possibly tied to the covariance of the observations~\cite{seitzer_pitfalls_2021, rybkin_simple_2021}.
Moreover, large values of \gsq may be %
harder %
to optimize due to an exploding reconstruction term in~\eqref{eq:elbo}. This may be one explanation for the slight deviation of~\cref{figure:self_cons}, right %
from our theory's predictions: while convergence of \gls{mu} to \gls{invdec} matches the prediction in~\cref{prop:selfconsist}, its rate is not precisely %
the one predicted for the self-consistent \gls{elbo}~\eqref{eq:self_cons_elbo}. Another cause could be the encoder's finite capacity. %
Nonetheless, we have experimentally shown that for realistic hyperparameters, \glspl{vae}' behavior matches the predictions of our theory for the near-deterministic regime.

\textbf{Dimensionality.}
The setting in~\cref{sec:theory} requires equal dimensionality for observations \gls{obs} and latents \gls{latent}, in line with work on normalizing flows~\cite{papamakarios_normalizing_2021} and nonlinear \gls{ica}~\cite{hyvarinen_nonlinear_2017, hyvarinen_nonlinear_2019, halva_hidden_2020} (but see, e.g., \cite{khemakhem_variational_2020}). For high-dimensional images, however, it is often assumed that \gls{obs} lives on a lower-dimensional manifold embedded in a higher-dimensional space, where the dimensionality of \gls{obs} is greater than \gls{latent}~\cite{dai2018diagnosing}. While our theoretical results do not cover this case, we observe empirically in \cref{figure:dsprites} that the predictions of our theory remain accurate when observations are high-dimensional images. Extending our theory to this setting could leverage ideas explored in, \eg,~\cite{dai2018diagnosing,cunningham2021change, caterini2021rectangular} and is left for future work.

\textbf{The \gls{elbo}, the self-consistent \gls{elbo}, and amortized inference.} \looseness-1
There are in principle multiple ways to obtain self-consistency (\cref{def:selfcons}). Notably, one could simply force the variational mean and variance encoder maps to behave this way; unlike~\cite{kumar_implicit_2020}, we model the actual behavior of \glspl{vae} trained under \gls{elbo} maximization, and obtain self-consistency as a result. For this, we assume that the optimal encoder, which minimizes the gap between \gls{elbo} and log-likelihood, can be learned.
This is not guaranteed in general, since it requires universal approximation capability of the encoder. 
On the other hand, \eqref{eq:self_cons_elbo} requires \textit{unamortized} inference to introduce $\gls{elbo}^*,$ which does not depend on \gls{encpar}. As in practice amortized inference may be used to efficiently estimate a single set of \gls{encpar} for all \gls{obs}~\citep{shu2018amortized}, it can lead to a suboptimal gap to the log-likelihood and discrepancies with our theoretical predictions.
\section{Discussion}
\label{sec:discussion}

\textbf{On disentanglement in unsupervised \glspl{vae}.} \looseness-1
It is widely believed that unsupervised \glspl{vae} cannot learn disentangled representations~\citep{locatello_challenging_2019,khemakhem_variational_2020}, motivating work on models with, \eg, conditional priors~\citep{khemakhem_variational_2020} or sparse decoding~\citep{moran2021identifiable}. We show that under certain assumptions,  
\gls{elbo} optimization can implement useful inductive biases for representation learning, yielding disentangled representations in unsupervised \glspl{vae}. However, while our results are formulated for~\glspl{vae}, some of the most successful models at disentanglement are modifications thereof---e.g., $\beta$-\glspl{vae}~\cite{higgins2016beta, burgess_understanding_2018}, with an additional parameter $\beta$ multiplying the \gls{kld} in~\eqref{eq:elbo}. 
While they %
deviate from the information projection setting considered in~\cref{sec:self_const}, their objectives are equivalent to the \gls{elbo} in a sense described in \cref{subsec:beta_vae_params}, which allows us to derive convergence to the \gls{ima}-regularized likelihood objective for $\nicefrac{\gamma}{\sqrt{\beta}}\to +\infty$. This encompasses the deterministic limit, and also the setting $\beta \to 0$ with constant $\gamma$ described in \cite{kumar_implicit_2020}. Whether this theoretical regime matches common practice remains an open question. 
Overall, we stress that we uncover {\em one} possible mechanism through which \glspl{vae} may achieve disentanglement. By connecting to \gls{ima}~\cite{gresele_independent_2021}, we discuss implications on recovering the ground truth under suitable assumptions, extending uniqueness results presented in~\cite{kumar_implicit_2020}. We speculate that our success in disentanglement is probably due to selecting data sets where the mixing is in the \gls{ima} class (\cf~\citep{horan_when_2021,donoho_image_2005}), which presumably was not the case in \citep{locatello_challenging_2019}. 

\textbf{Characterizing the \gls{elbo} gap for nonlinear models.} ~\cref{prop:vae_ima} characterizes the gap between \gls{elbo} and true log-likelihood for nonlinear \glspl{vae}, and extends the linear analysis of \citet{lucas_dont_2019} and the results of \citet{dai2018connections} in the affine case; we also empirically characterize the gap in the deterministic limit in~\cref{subsec:exp_elbo_ima_likelihood}. An unanticipated consequence of this result is that---consistent with \citep{lucas_dont_2019}---\glspl{vae} optimize the \gls{ima}-regularized log-likelihood in the near-deterministic limit, and not the unregularized one, as stated in~\citep{nielsen_survae_2020}. 

\textbf{Extensions to related work.}
Several papers discuss the (near-)deterministic regime~\citep{nielsen_survae_2020,rolinek_variational_2019,kumar_implicit_2020,dai2018diagnosing}. 
For example,~\citet{nielsen_survae_2020} postulate a deterministic \gls{vae} with the encoder inverting the decoder. Also~\citet{kumar_implicit_2020} work in that regime, %
but without justifying the relationship between the encoder and decoder. Although they show that the choice of \gls{pz} and \gls{q} influences 
uniqueness (by, \eg, ruling out rotations), this does not imply recovering the true latents. Our approach formalizes (\cref{def:selfcons}), proves (\cref{prop:selfconsist}), and demonstrates the practical feasibility of (\cref{sec:experiments})
the near-deterministic regime. To the best of our knowledge, all previous work relied on the linear case~\cite{lucas_dont_2019} or a (linear) approximation and the evaluation of the \gls{elbo} \textit{around a point} to show the inductive bias on the decoder Jacobian. However, our main result (\cref{prop:vae_ima}) yields a nonlinear equation where the decoder Jacobian can be evaluated at \textit{any point} and is equipped with a convergence bound. Moreover, the consistency of \gls{vae} estimation for identifiable models~\cite{khemakhem_variational_2020} requires guarantees on \gls{q}; our result helps proving these. \citet{dai2018diagnosing} use a non-factorized Gaussian variational posterior and prove in their Thm.~2 (including the $\dim \gls{obs}=\dim\gls{latent}$ case) that in the deterministic limit their $\kappa$-simple \gls{vae} can fit perfectly arbitrary observed data (barring few assumptions), while the \gls{elbo} gap tends to zero. In contrast, we use a factorized variational posterior; this prevents the \gls{elbo} gap to vanish in the deterministic limit, except in the special case of a decoder mean in the \gls{ima} class fitting the data perfectly. 
\citet{dai2018diagnosing} take the limit of ${\gamma\to+\infty}$ (here using $\gamma$ as the square root of the decoder precision and not the decoder variance as used in \cite{dai2018diagnosing}) to relate encoder and decoder properties in this limit in their Thm.~5, similarly to \cref{prop:selfconsist}. In contrast to our nonlinear analysis, this is derived when optimizing \wrt both encoder and decoder parameters, and with a non-factorized encoder assumption, leading to fundamentally different behavior of the solutions in the deterministic limit. The work done by \citet{sliwa_probing_2022}, simultaneously to ours, showcases an extensive empirical study highlighting that the \gls{ima} contrast allows distinguishing true and spurious solutions for a broad range of cases and outperforms standard regularizers such as weight decay. 
We discuss extended connections to the literature in \cref{sec:app_related} and \cref{sec:app_ext_ima_vae}.

\textbf{Covariance structure and \gls{ima}.}
We have shown that specific choices for encoder and decoder covariances regularize the decoder Jacobian, such that closing the \gls{elbo} gap constrains the decoder to belong to the \gls{ima} class. Following our intuition (\cref{figure:fig1}), assuming factorized \gls{q} and isotropic \gls{pxz}, \gls{ima} holds only for the \textit{decoder}; since in the other direction the pushforward of \gls{q} through \gls{dec} has covariance $\gtmixjacobian\gls{var_cov}\transpose{\gtmixjacobian},$ which cannot be used to make row orthogonality statements on $\gtmixjacobian$ in the general case. Additionally, we conjecture that assuming an isotropic encoder would constrain \gls{ima} to hold in both encoding and decoding directions (as both \gtmixjacobian and \unmixjacobian need to be column-orthogonal), such that the resulting decoder mean is constrained to have orthogonal columns of equal norms, which is a defining property of conformal maps \citep{buchholz_function_2022}. 
On the other hand, we conjecture that if the observation model is not isotropic, but the encoder model is, \gls{ima} would only tend to be enforced for the mean \textit{encoder} Jacobian, converging to the inverse decoder mean in the deterministic limit.   %

\textbf{Implications for recovering the true latent factors using unsupervised \glspl{vae}.} \looseness-1
Convergence of the \gls{elbo} to the \gls{ima}-regularized log-likelihood suggests that unsupervised \glspl{vae} may recover the true factors of variation according to current identifiability results of the \gls{ima} class~\cite{buchholz_function_2022}. This is based on the following reasoning: \textit{If the ground truth generative model belongs to the IMA class, unsupervised learning of the model with an infinite capacity \gls{vae} will, in the deterministic limit, ensure a solution that perfectly fits the data and whose decoder mean is also in the IMA class (by joint optimization of both the likelihood and the regularization term). 
Identifiability of the IMA class implies that the \gls{vae} will learn the true decoder (up to acceptable ambiguities); then, since self-consistency guarantees that the encoder inverts the decoder, the encoder infers the ground truth generative factors associated to observations.}
Although strict identifiability for all functions in the \gls{ima} class remains to be proven, three concurrent papers provide guarantees that go towards identifiability: \citet{leemann_disentangling_2022} proves identifiability for a subset of the \gls{ima} class in the context of concept discovery; \citet{zheng_identifiability_2022} shows identifiability of nonlinear \gls{ica} by assuming a specific sparsity structure of the decoder Jacobian (called \textit{structural sparsity}); whereas \citet{buchholz_function_2022} introduce the concept of \textit{local identifiability} and proves that \gls{ima} is locally identifiable. 

\looseness-1
Moreover, as mentioned in the above paragraph, we suspect that closing the \gls{elbo} gap with an isotropic encoder (while the encoder in \cref{prop:vae_ima} is only constrained to have diagonal covariance) constrains the decoder to be a conformal map. This is an interesting constraint, as nonlinear \gls{ica} with conformal mixings are identifiable: the two-dimensional case was first addressed with some additional constraints in~\citep{hyvarinen_nonlinear_1999}, while the general case (in arbitrary dimension) was shown to rule out certain spurious solutions for conformal mixings~\citep{gresele_independent_2021}, and finally proven to be identifiable by \citet{buchholz_function_2022} in concurrent work. Hence, we conjecture that given a ground truth generative model with a conformal map from latent to observation space, and an unsupervised \glspl{vae} with isotropic Gaussian encoders and decoders, the true latent factors can be recovered.

\textbf{Conclusion.}
\looseness-1
We provide a theoretical justification for \glspl{vae}' widely-used self-consistency assumption in the near-deterministic regime of small decoder variance. Using this result, we show that %
the self-consistent \gls{elbo} converges to the \gls{ima}-regularized log-likelihood, and not to the unregularized one. Thus, we can characterize the gap between \gls{elbo} and true log-likelihood and reason about its role as an inductive bias for representation learning in nonlinear \glspl{vae}. We characterize a set of assumptions under which unsupervised \glspl{vae} can be expected to disentangle and we demonstrate this behavior in experiments on synthetic and image data.%

\fi

\begin{ack}
The authors thank the anonymous reviewers for their suggestions. This work was supported by the German Federal Ministry of Education and Research (BMBF): Tübingen AI Center, FKZ: 01IS18039A \& 01IS18039B, and by the Machine Learning Cluster of Excellence, EXC number 2064/1 – Project number 390727645. 
Wieland Brendel acknowledges financial support via an Emmy Noether Grant funded by the German Research Foundation (DFG) under grant no. BR 6382/1-1. 
The authors thank the International Max Planck Research School for Intelligent Systems (IMPRS-IS) for supporting Dominik Zietlow and Patrik Reizinger. Patrik Reizinger acknowledges his membership in the European Laboratory for Learning and Intelligent Systems (ELLIS) PhD program.
\end{ack}

\bibliographystyle{plainnat}
\bibliography{processed_references}

\setglossarysection{subsection}
\printacronyms
\printglossary[type=abbrev, style=tree]

\newpage
\section*{Checklist}

\begin{enumerate}

\item For all authors...
\begin{enumerate}
  \item Do the main claims made in the abstract and introduction accurately reflect the paper's contributions and scope?
    \answerYes{See \cref{sec:theory}, \cref{sec:experiments}, \cref{sec:discussion}}
  \item Did you describe the limitations of your work?
    \answerYes{See \cref{sec:discussion}}
  \item Did you discuss any potential negative societal impacts of your work?
    \answerYes{See \cref{sec:app_impact}}
  \item Have you read the ethics review guidelines and ensured that your paper conforms to them?
    \answerYes{We use no data about human subjects.}
\end{enumerate}

\item If you are including theoretical results...
\begin{enumerate}
  \item Did you state the full set of assumptions of all theoretical results?
    \answerYes{See \cref{sec:theory}}
        \item Did you include complete proofs of all theoretical results?
    \answerYes{See \cref{sec:theory}, \cref{sec:app_proofs}}
\end{enumerate}

\item If you ran experiments...
\begin{enumerate}
  \item Did you include the code, data, and instructions needed to reproduce the main experimental results (either in the supplemental material or as a URL)?
    \answerYes{See the supplementary material}
    \item Did you specify all the training details (e.g., data splits, hyperparameters, how they were chosen)?
    \answerYes{See \cref{sec:experiments}, \cref{sec:app_exp}}
        \item Did you report error bars (e.g., with respect to the random seed after running experiments multiple times)?
    \answerYes{See \cref{sec:experiments}}
        \item Did you include the total amount of compute and the type of resources used (e.g., type of GPUs, internal cluster, or cloud provider)?
    \answerYes{See \cref{sec:app_exp}}
\end{enumerate}

\item If you are using existing assets (e.g., code, data, models) or curating/releasing new assets...
\begin{enumerate}
  \item If your work uses existing assets, did you cite the creators?
    \answerYes{See \cref{sec:experiments} and the supplementary material}
    \item Did you mention the license of the assets?
    \answerYes{See the supplementary material}
  \item Did you include any new assets either in the supplemental material or as a URL?
    \answerYes{We included the code and the experiment logs.}
  \item Did you discuss whether and how consent was obtained from people whose data you're using/curating?
    \answerNA{}
  \item Did you discuss whether the data you are using/curating contains personally identifiable information or offensive content?
    \answerNA{}
\end{enumerate}

\item If you used crowdsourcing or conducted research with human subjects...
\begin{enumerate}
  \item Did you include the full text of instructions given to participants and screenshots, if applicable?
     \answerNA{}
  \item Did you describe any potential participant risks, with links to Institutional Review Board (IRB) approvals, if applicable?
     \answerNA{}
  \item Did you include the estimated hourly wage paid to participants and the total amount spent on participant compensation?
     \answerNA{}
\end{enumerate}

\end{enumerate}
\newpage

\appendix

\addcontentsline{toc}{section}{Appendix} %
\part{Appendix} %
\parttoc %
\newpage
\section{Complementary notes}
\label{app:complement}

\subsection{\gls{elbo} decompositions}

\paragraph{Connection between \eqref{eq:elbo} and \eqref{eq:elbo_kl_truepost}.}
Here we show how the two decompositions of the \gls{elbo} objective in~\eqref{eq:elbo} and~\eqref{eq:elbo_kl_truepost} can be connected. We start from equation~\eqref{eq:elbo_kl_truepost}:
\begin{equation*} 
    \gls{elbo} (\gls{obs},\decpar,\gls{encpar}) = \log \gls{px} -\kl{\gls{q}}{\gls{pzx}}.
\end{equation*}
By definition of KL-divergence, and applying Bayes rule, %
we get
\begin{multline*}
    \gls{elbo} (\gls{obs},\decpar,\gls{encpar}) = \log \gls{px} -\int \gls{q}\parenthesis{\log \gls{q}-\log\gls{pzx}}d\gls{latent}\\
    =\log \gls{px} -\int \gls{q}\parenthesis{\log \gls{q}-\log\parenthesis{\gls{pxz}\dfrac{\gls{pz}}{\gls{px}}}}d\gls{latent}\,.
\end{multline*} 
We observe that the two terms involving $\gls{px}$ cancel, resulting in
\begin{equation*}
    \gls{elbo} (\gls{obs},\decpar,\gls{encpar}) = -\int \gls{q}\parenthesis{\log \gls{q}-\log\parenthesis{\gls{pxz}\gls{pz}}}d\gls{latent},
\end{equation*} 
which leads to \eqref{eq:elbo} by rearranging the terms:
\begin{equation*}%
    \gls{elbo}(\gls{obs},\decpar,\gls{encpar})=\expectation{\gls{q}}\brackets{\log \gls{pxz}}-\kl{\gls{q}}{\gls{pz}}.
\end{equation*}
\paragraph{Expressions for the two terms in equation~\eqref{eq:elbo} under \cref{assum:VAE}.}
The above two terms take the following form in our setting. %
For the second (``\gls{kld}'') term, we get
\begin{multline*}
-\kl{\gls{q}}{\gls{pz}}=\int \gls{q}\log \gls{pz}d\gls{latent}-\int \gls{q}\log \gls{q}d\gls{latent}\\=\mathbb{E}_{\gls{q}}[\log(\gls{pz})]+H(\gls{q})\,,
\end{multline*}
where $H$ denotes the entropy.
Writing the expression for the entropy of univariate Gaussian variables ($\nicefrac{1}{2}\log (2\pi \sigma^2)+\nicefrac{1}{2}$), we have under \cref{assum:VAE}
\[
H(\gls{q}) = \frac{d}{2}\left(\log (2\pi) +1\right) +  \frac{1}{2}\sumk[d] \log \gls{sigmak} = \kappa_d +  \frac{1}{2}\sumk[d] \log \gls{sigmak}, 
\]
where we introduce the dimension dependent constant $\kappa_d=\frac{d}{2}\left(\log (2\pi) +1\right).$
This leads to 
\begin{equation}\label{eq:postreg}
-\kl{\gls{q}}{\gls{pz}}=\mathbb{E}_{\gls{q}}[\log(\gls{pz})] +  \frac{1}{2}\sumk[d] \log \gls{sigmak} + \kappa_d\,.
\end{equation}

The first (``reconstruction'') term, under the isotropic Gaussian decoder of \cref{assum:VAE}, takes the form
\begin{equation}
    \expectation{\gls{q}}\brackets{\log \gls{pxz}}=-\frac{\gsq }{2}\expectation{\gls{q}}\brackets{\|\gls{obs}-\gtmix\|^2}+d\log\gamma-\frac{d}{2}\log(2\pi)\,.
    \label{eq:rec_term}
\end{equation}

\paragraph{Expression for the gap between \gls{elbo} and log-likelihood}
Let us now write the \gls{kld} divergence between variational and true posteriors, which is the gap appearing in \eqref{eq:elbo_kl_truepost}.
\[
\kl{\gls{q}}{\gls{pzx}} = -\int \gls{q}\log \gls{pxz}  d\gls{latent} - H(\gls{q})
\]
Using again the expression of the entropy of Gaussian variables, this leads to
\[
\kl{\gls{q}}{\gls{pzx}} = -\int \gls{q}\log \gls{pxz}  d\gls{latent} - \sumk[d] \log \encstdcomp_k (x) -\frac{d}{2}\left(\log (2\pi) +1\right),
\]
such that, using the Bayes formula for the true posterior and Assum.~\ref{assum:VAE}, we get
\begin{multline}
    \kl{\gls{q}}{\gls{pzx}} = -\sumk[d] \log \encstdcomp_k (\gls{obs})+c(\gls{obs},\gamma)\\+\frac{1}{2}\mathbb{E}_{z\sim q_{\encpar}(\cdot|\gls{obs})}\left[\|\gls{obs}-\gtmix\|^2\gamma^2  -2\sumk[d]\log m(z_k)\right] \,,
    \label{eq:KLgeneralexp}
\end{multline}
with additive constant $c(\gls{obs},\gamma)= -\frac{d}{2}\left(\log (\gamma^2)+1\right)+\log p_{\decpar}(\gls{obs})$.
Note the $\log(2\pi)$ term in the previous expression cancels with the one coming from the true log posterior. 

The analysis of the optima of~\eqref{eq:KLgeneralexp} is non-trivial due to the second term which involves taking expectations of  functions of \gls{latent} \wrt its posterior distribution $q_{\encpar}$ parameterized by $\encmean$ and $\encstd$. Much of the derivations to obtain our results will revolve around constructing bounds that no longer involve such expectations, but instead only depend on $\encmean$ and $\encstd$.

\subsection{Justification of the intuition}

We add here more qualitative details to the statement of \autoref{sec:implicit_constraints} that the true posterior density is approximately the pushforward of $\ptheta{\gls{obs}|\gls{latent}=\gls{latent}_0}$. Note that they are not meant to replace a rigorous treatment, which is deferred to \cref{sec:app_proofs}. 

As the decoder becomes deterministic, the marginal observed density becomes the pushforward of the latent prior by $\gls{dec}$ \footnote{because the conditional distribution of the decoder tends to a Dirac measure  at $\gls{dec}$} such that
\[
\gls{px} \approx p_0 \parenthesis{\gls{invdec}(\gls{obs})} |\unmixjacobian|\,.
\]
The true posterior is therefore approximately
\[
\gls{pzx} = \gls{pxz}\gls{pz}/\gls{px} \approx \gls{pxz}\gls{pz}/p_0 \parenthesis{\gls{invdec}(x)}  |\unmixjacobian|^{-1}\,.
\]
Conditioning on a given observation $\gls{obs}=\gtmix[\gls{latent}_0]$, we get
\begin{multline*}
\ptheta{\gls{latent}|\gls{obs}=\gtmix[\gls{latent}_0]} 
= \ptheta{\gtmix[\gls{latent}_0]|\gls{latent}}\gls{pz}/\ptheta{\gls{obs}=\gtmix[\gls{latent}_0]} 
\\
\approx 
\ptheta{\gtmix[\gls{latent}_0]|\gls{latent}}\gls{pz}/p_0 \parenthesis{\gls{invdec}(\gtmix[\gls{latent}_0])}  |\unmixjacobian[\gtmix[\gls{latent}_0]]|^{-1}\\
\approx \ptheta{\gtmix[\gls{latent}_0]|\gls{latent}}\gls{pz}/p_0 \parenthesis{\gls{latent}_0}  |\unmixjacobian[\gtmix[\gls{latent}_0]]|^{-1}\\
\end{multline*}
Neglecting the variations of the prior relative to those of the posterior (due to near-determinism), we make the approximation $\gls{pz}\approx p_0 \parenthesis{\gls{latent}_0}$ such that the above approximation becomes 
\[
\ptheta{\gls{latent}|\gls{obs}=\gtmix[\gls{latent}_0]} 
\approx \ptheta{\gtmix[\gls{latent}_0]|\gls{latent}}|\mixjacobian[\gls{latent}_0]|\,.
\]
Using the isotropic Gaussian decoder assumption, we get
\begin{align*}
\ptheta{\gls{latent}|\gls{obs}=\gtmix[\gls{latent}_0]}
\approx\dfrac{\gamma^d}{\sqrt{2\pi}^d}\exp\parenthesis{-\dfrac{\gsq}{2}\normsquared{\gtmix[\gls{latent}_0]-\gtmix[\gls{latent}]}}|\mixjacobian[\gls{latent}_0]|\,.
\end{align*}
In the near-deterministic regime, this posterior distribution should be concentrated in the region where $\gls{latent}$ is close to $\gls{latent}_0$, we can then further approximate this density using a Taylor formula
\begin{multline*}
\ptheta{z|\gls{obs}=\gtmix[\gls{latent}_0]}
\approx \frac{\gamma^d}{\sqrt{2\pi}^d}\exp\parenthesis{-\dfrac{\gsq}{2}\normsquared{\mixjacobian[\gls{latent}_0]\parenthesis{\gls{latent}_0-\gls{latent}}}}|\mixjacobian[\gls{latent}_0]|\\
=\frac{\sqrt{2\pi}^{-d}\gamma^d}{\sqrt{\left|\mat{G} \transpose{\mat{G}} \right|}}
\exp\parenthesis{-\frac{1}{\gsq}\transpose{\parenthesis{\gls{latent}_0-\gls{latent}}}\parenthesis{\mat{G} \transpose{\mat{G}}}^{-1}\parenthesis{\gls{latent}_0-\gls{latent}}}\,,
\end{multline*}
with $\mat{G}=\unmixjacobian[\gtmix[\gls{latent}_0]]=\mixjacobian[\gls{latent}_0]^{-1}$,
which is also matching the expression of the pushforward of the Gaussian density $\ptheta{\gls{obs}|\gls{latent}=\gls{latent}_0}$ by the linearization of $\gls{invdec}$ around $\gtmix[\gls{latent}_0]$ (i.e. replacing the mapping by its Jacobian at that point, $\mat{G}$).
\subsection{A connection between the $\beta$ parameter of \betavae{}s and the decoder precision $\gamma^2$}
\label{subsec:beta_vae_params}

In the context of disentanglement, a commonly used variant of standard \glspl{vae}~\citep{kingma_auto-encoding_2014} is the \betavae%
~\citep{chen2018isolating,higgins2016beta,kim2018disentangling,rolinek_variational_2019,kumar_implicit_2020}. %
In this model, an additional parameter $\beta$ is added to modify the weight of the \gls{kld} term in \eqref{eq:elbo}%
, whereas the decoder precision $\gamma^2$ is typically set to one~\citep{dai2018diagnosing,ghosh_variational_2020,kumar_implicit_2020,rolinek_variational_2019}.
The $\beta$-VAE objective~\cite{higgins2016beta} can be written as %
\begin{align}
    \betaloss &= \expectation{\gls{q}}\brackets{\log \gls{pxz}} -\beta KL\brackets{\gls{q}\|\gls{pz}}\,,\quad \beta > 0\,.
    \label{eq:betavae}
\end{align}

The influence of the decoder precision \gsq and the $\beta$ parameters on the objective have been related in the literature, see for example~\citep[][\S~2.4.3]{doersch_tutorial_2021}---and similar observations can be found in~\citep[][\S~3.1]{rybkin_simple_2021}.
Under the assumption of a Gaussian decoder, the \gls{elbo} from \cref{eq:elbo} can be written as
(making now explicit mention of its decoder parameter $\gamma$ in parenthesis):
\begin{align}
   \elbolossgamma &= -\kl{\gls{q}}{\gls{pz}} + \expectation{\gls{q}}\brackets{\log \gls{pxz}} 
 \nonumber  \\
 &=-\kl{\gls{q}}{\gls{pz}} -\frac{\gsq }{2}\expectation{\gls{q}}\!\brackets{\normsquared{\gls{obs}-\gtmix}}\!+\!c(\gamma,d)\nonumber\,, %
\end{align}
with $c(\gamma,d)=\gls{obsdim}\log\gamma
\!-\!\frac{d}{2}\log(2\pi) $. 

In contrast, the $\beta$-VAE objective $\betaloss$ (also with explicit mention of $\gamma$) is expressed as: 

\begin{align*}
	\betalossgamma &= -\beta \kl{\gls{q}}{\gls{pz}} + \expectation{\gls{q}}\brackets{\log \gls{pxz}} 
	\\
	&=-\beta \kl{\gls{q}}{\gls{pz}} -\frac{\gsq }{2}\expectation{\gls{q}}\!\brackets{\normsquared{\gls{obs}-\gtmix}}\!+\!c(\gamma,d) \\
\end{align*}
We thus can link this expression to an \gls{elbo} as follows: 
\begin{multline*}
  	\betalossgamma 	=\beta \big{[}-\kl{\gls{q}}{\gls{pz}} -\frac{\gsq }{2\beta}\expectation{\gls{q}}\!\brackets{\normsquared{\gls{obs}-\gtmix}}\!\\
  	+\!\frac{1}{\beta}c(\gamma,d)\big] \\
    =\beta \big(-\kl{\gls{q}}{\gls{pz}} -\frac{\left(\frac{\gamma}{\sqrt{\beta}}\right)^2 }{2}\expectation{\gls{q}}\!\brackets{\normsquared{\gls{obs}-\gtmix}}\\
       	\quad +c\left(\frac{\gamma}{\sqrt{\beta}},d\right)\!-c\left(\frac{\gamma}{\sqrt{\beta}},d\right)+\!\frac{1}{\beta}\left(\gls{obsdim}\log ( \gamma)
    \!-\!\frac{d}{2}\log(2\pi)\right)\big)  \\
    =\beta \left[\elbolossgamma[\frac{\gamma}{\sqrt{\beta}}] \!+\!d\frac{\log \gamma}{\beta}-d\log(\frac{\gamma}{\sqrt{\beta}})%
    \!+\left(1-\frac{1}{\beta}\right)\!\frac{d}{2}\log(2\pi)\right]\,.
\end{multline*}
When we restrict ourselves to common practice, the optimizations of both the \gls{elbo} and the \gls{betaloss} are performed with a fixed value of $\gamma$ %
(and with fixed $\beta$ for \gls{betaloss}). This entails that %
there is an equivalence of the solutions resulting from the optimization of each objective, as they differ only by additive and multiplicative constants. In particular, given the above expression, the following result is immediate:
\begin{prop}\label{prop:beta_vae_ima}
Let us define the self-consistent $\beta$-VAE objective as 
\begin{equation}
    \betalossgammastar[\gamma] = \min_{\encpar}\betalossgamma \,.
\end{equation}
Then the following normalized self-consistent $\beta$-VAE objective
\begin{equation}
\frac{1}{\beta}\left(\betalossgammastar[\frac{\gamma}{\sqrt{\beta}}]-d\log \gamma+\frac{d}{2}\log (2\pi) \right)+d\log(\frac{\gamma}{\sqrt{\beta}})\!-\!\frac{d}{2}\log(2\pi) =  \elbolossgammastar[\frac{\gamma}{\sqrt{\beta}}] \,.
\end{equation}
converges to the regularized IMA objective as $\frac{\gamma}{\sqrt{\beta}}\to +\infty$ under the same conditions as in 
\cref{prop:vae_ima}. 
\end{prop}
\begin{proof}
The $\beta$-VAE objective is, up to an additive constant and a strictly positive multiplicative constant, identical to the \gls{elbo} objective. As a consequence their optimum and the values of parameters $\encpar$ and $\decpar$ at which they are achieved are identical. 
\end{proof}
This suggests that choosing a fixed $\beta$ or a $\beta$ growing as, for example, $\log \gamma$ or even $\sqrt{\gamma}$ would lead to self-consistent solutions with the same properties as the vanilla VAE in the deterministic decoder limit $\gamma \to +\infty$, and notably robustness to spurious solutions. 

\paragraph{Generalizing \citet{kumar_implicit_2020}} Moreover, this connection also shows that our result generalizes previous work. Particularly, considering $\gsq=1$ when expressing the optimal variance analogous to \cref{eq:optimal_sigma}, we can discover the same expression as in \cite{kumar_implicit_2020}.

Namely, the objective function has the form (where we now indicate explicitly the dependence on $\gamma$ and emphasize that both $\gamma$ and $\beta$ are fixed):
\begin{multline}
    \betalossgamma =-\frac{1}{2}\sumk[\gls{obsdim}]\Bigg[\beta\log\frac{1}{\gls{sigmak}}-\beta\\
    +\beta\gls{sigmak}\parenthesis{-\frac{d^2\log p_0}{d\gls{latentcomp}_k^2}(g^{\decpar}_k(\gls{obs}))+\dfrac{\gsq}{\beta}\normsquared{\cols{\mixjacobianbis{\gls{invdec}(\gls{obs})}}} }-2\beta\log(m(g^{\decpar}_k(\gls{obs})))\Bigg]\\+d\log\gamma-\frac{d}{2}\log(2\pi).
\end{multline}
As a consequence, the optimal \gls{sigmak} now includes $\beta$ compared to \eqref{eq:optimal_sigma}:
\begin{align}
    \gls{sigmaoptk} &= \parenthesis{-\frac{d^2\log p_0}{d\gls{latentcomp}_k^2}(g^{\decpar}_k(\gls{obs}))+\dfrac{\gsq}{\beta}\normsquared{\cols{\mixjacobianbis{\gls{invdec}(\gls{obs})}}} }^{-1}. \label{eq:optimal_sigma_beta}
\end{align}
That is, $\beta$ affects the decoder Jacobian column norms as $\nicefrac{1}{\gsq}$. They are not exactly equivalent though---note the $\beta$ factor in front of the log-prior terms:
\begin{multline}
    \gls{betaloss}(\gls{obs};\gls{decpar}, \gls{encparopt}) =-\frac{1}{2}\sumk[\gls{obsdim}]\Bigg[\beta\log\parenthesis{-\frac{d^2\log p_0}{d\gls{latentcomp}_k^2}(g^{\decpar}_k(\gls{obs}))+\dfrac{\gsq}{\beta}\normsquared{\cols{\mixjacobianbis{\gls{invdec}(\gls{obs})}}}}\\ -2\beta\log(m(g^{\decpar}_k(\gls{obs})))\Bigg]+d\log\gamma+cst.\label{eq:betavae_cima}
\end{multline}

Again, as both $\beta$ and $\gamma$ are fixed, the additive/multiplicative constants are irrelevant. Thus, $\beta$ has a similar effect on column-orthogonality as $\nicefrac{1}{\gsq}$. We formalize this observation in the following remark:
\begin{remark}
    $\beta$ affects the column norms of the decoder Jacobian in the same way as $\nicefrac{1}{\gsq}$. One can think of having $\gamma'=\gamma/\sqrt{\beta}$ and tying the tuning of $\beta, \gamma'$. \label{remark:beta_gamma}
\end{remark}

\begin{remark}[Generalization of \citet{kumar_implicit_2020}]
    Assuming the same conditions as in \cref{prop:vae_ima} and $\gsq=1$ (\ie, a Gaussian decoder with a Hessian $\gls{hessian}=\gls{identity}$ and Gaussian prior with \gls{identity} as second-order derivative) with the \betavae loss, then expressing the optimal posterior covariance with $\nicefrac{\gamma}{\sqrt{\beta}}$, we get Eq. (11) of \citet{kumar_implicit_2020},
    \begin{align*}
        \gls{var_cov} &= \inv{\parenthesis{ \gls{identity} + \dfrac{1}{\beta}\mixjacobian\transpose{\mixjacobian} }}
    \end{align*}
    as a special case.
\end{remark}

\subsection{\gls{cima_local} as the (scaled) left \gls{kld} measure of diagonality of $\transpose{\gls{jacobian}_{\gls{dec}}}\gls{jacobian}_{\gls{dec}}$}
\label{subsec:cima_as_left_kl}

In our paper, we used the original definition for \gls{cima_local} (\citep[(8)]{gresele_independent_2021}), which we restate here:
\begin{align*}
        \cima &= 
    \sumk[d]\log\norm{\tfrac{\partial\gls{dec}}{\partial\gls{latentcomp}_k}\parenthesis{\gls{latent}}}-\logabsdet{\gtmixjacobian}.
\end{align*}

However, an alternative formulation exists: namely, \cima can be thought of as the (scaled) \textit{left \gls{kld} measure of diagonality} of the \textit{square} matrix $\transpose{\gtmixjacobian}\gtmixjacobian,$ as shown in \cite{alyani_diagonality_2017}. That is, we can rewrite the above as~\citep[(23)]{gresele_independent_2021} (with $\mat{A} = \transpose{\gtmixjacobian}\gtmixjacobian$):
\begin{align}
    \cima &= \dfrac{1}{2}\mathrm{D}^{\mathrm{left}}_{\mathrm{\gls{kld}}}\parenthesis{\mat{A}} \label{eq:cima_as_left_kl}\\
    &= -\dfrac{1}{2}\logabsdet{\diag{\mat{A}}^{-\nicefrac{1}{2}}\mat{A}\diag{\mat{A}}^{-\nicefrac{1}{2}}}\label{eq:cima_as_scaled_log_det}\\
    &= \dfrac{1}{2}\parenthesis{\logabsdet{\diag{\mat{A}}} -\logabsdet{\mat{A}}}
\end{align}

The importance of \cref{eq:cima_as_left_kl} is twofold: i) it provides a theoretical analysis of \gls{cima_local} as a measure of the column-orthogonality of \gtmixjacobian (or, equivalently, the diagonality of  $\transpose{\gtmixjacobian}\gtmixjacobian$); and ii) it elucidates why \gls{cima_local} \textit{can be used in the $\dim\gls{obs}\neq\dim\gls{latent}$ case,} but only to measure column-orthogonality (to exploit the beneficial properties of \gls{ima} for identifiability, the theory of \gls{ima} needs to be extended to this case). We leverage ii) as a justification to use \gls{cima_local} in our image experiments in \cref{subsec:exp_col_gamma_dis}. 

\subsection{Assessing the value of \gls{cima_local} for recovering the true latents}
\label{subsec:cima_value_assessment}

\gls{cima_local}, given by \eqref{eq:ima_objective}, measures the deviation of the learned decoder from the \gls{ima} function class. As it is positive and unbounded, it is practically relevant to investigate the following question: \textit{how much violation of the \gls{ima} assumption is acceptable to recover the true latents?} Expressed differently, we are interested in whether a threshold can be specified for \cima to decide whether the true (but unknown) latent factors can be recovered.

Unfortunately, our answer is negative, but this is not specific to \gls{ima} theory. Even in the case of linear \gls{ica} we cannot specify a threshold for the non-Gaussianity of the latent (source) variables to ensure that the ground-truth factors can be recovered~\cite{hyvarinen_nonlinear_1999}.

We acknowledge that, among others, \gls{cima_local} being unbounded makes it harder to interpret. Thus, we go back to first principles to develop an imperfect but hopefully more practical intuition on how to assess the value of \gls{cima_local} to decide whether the true latents are recovered (if the true mixing is in the \gls{ima} class). As we already pointed out in \cref{subsec:cima_as_left_kl}, \gls{cima_local} is the left \gls{kld} measure of diagonality of $\transpose{\gls{jacobian}_{\gls{dec}}}\gls{jacobian}_{\gls{dec}}$~\cite{alyani_diagonality_2017}. In their analysis, the authors provide closed form solutions expressions to compare different diagonality measures. 

For our purposes, \citep[Fig.~1]{alyani_diagonality_2017} provides an important insight: it shows that in the two-dimensional case \gls{cima_local} increases nonlinearly in a variable $r$ expressing the degress of diagonality of a matrix (more precisely as $-\log\parenthesis{1-r^2},$ where $r=0$ denotes the identity matrix, whereas $r=1$ a matrix with parallel columns). The takeaway for us is that simply comparing two \gls{cima_local} values can be misleading. For the two-dimensional case, we could define an expression that is linear in $r$ (but this requires us to accept that $r$ is a -in some sense- suitable measure of the diagonality of a matrix), namely:
\begin{align*}
    \cima &= -\dfrac{1}{2}\log\parenthesis{1-r^2} \implies r= \sqrt{{1-\exp\parenthesis{-2\cima}}}.
\end{align*}

Unfortunately, the \gls{obsdim}-dimensional case only has a power series formulation; thus, it is nontrivial how to extend the above reasoning. However, \cref{eq:cima_as_scaled_log_det} expresses \gls{cima_local} as the negative log determinant of a scaled matrix, where the determinant of $\diag{\mat{A}}^{-\nicefrac{1}{2}}\mat{A}\diag{\mat{A}}^{-\nicefrac{1}{2}}$ is between 0 and 1 ($\mat{A}=\transpose{\gtmixjacobian}\gtmixjacobian).$ Thus, we can convert \gls{cima_local} to the \brackets{0;1} interval as follows:
\begin{align}
    \cima &= -\dfrac{1}{2}\logabsdet{\diag{\mat{A}}^{-\nicefrac{1}{2}}\mat{A}\diag{\mat{A}}^{-\nicefrac{1}{2}}} \\
          &= -\dfrac{1}{2}\logabsdet{\hat{\mat{A}}} \\
  0 &\leq \exp\parenthesis{-2\cima} \leq 1.
\end{align}
The merit of the above expression is a more natural way to compare values that are normalized to the \brackets{0;1} interval. Namely, the original formulation of \gls{cima_local} may potentially lead to problems, as, \eg, a value of $t$ and $10t$ does not mean that one model is ten times better in recovering the true latents. Nonetheless, this can be thought (at most) as a small step towards making the analysis of the empirical value of \gls{cima_local} a useful tool in practical scenarios, where we do not have access to the true latent factors.

\section{Main Theoretical Results}
\label{sec:app_proofs}

\subsection{Proof of \autoref{prop:selfconsist}}
We proceed in two steps: first we prove the existence of variational parameters that achieve a global minimum of the \gls{elbo} gap, then we characterize its near-deterministic properties. We then combine these results, which rely on specific assumptions, to obtain our main text result under \cref{assum:VAE}.  

We initially use the following milder assumptions than in main text to prove intermediate results. 
 \begin{assum}[Gaussian Encoder-Gaussian Decoder \gls{vae}, minimal properties] \label{assum:VAEminim}
 We are given a fixed latent prior and three parameterized classes of $\,\RR^d \to \RR^d$ mappings: the mean decoder class $\gls{decpar}\mapsto \gls{dec}$, and the mean and standard deviation encoder classes, $\gls{encpar}\mapsto \encmean$ and $\gls{encpar}\mapsto \encstd$ such that
\begin{enumerate}[label=(\roman*),nolistsep]
    \item  the latent prior has a factorized \gls{iid} density $\gls{pz}\sim \prod_k m(\gls{latentcomp}_k)$, with $m$ smooth fully supported on $\RR$, with concave $\log m$, 
    \item conditional on the latent, the decoder has a factorized Gaussian density $p_{\decpar}$ with mean $\gls{dec}$ such that 
  \begin{equation}
      \gls{obs}|\gls{latent} \sim \mathcal{N}\parenthesis{\gtmix,\gamma^{-2}\gls{identity}}
  \end{equation}
    \item  the encoder is factorized Gaussian with posterior mean and variance maps $\gls{muk},\gls{sigmak}$
    for each component $k$, leading to the factorized posterior density
    \gls{q} such that
      \begin{equation}
         \gls{latentcomp}_{k}|\gls{obs}\sim\mathcal{N}(\gls{muk},\gls{sigmak})
    \end{equation}
    \item the mean and variance encoders classes can fit any function,
    \item for all possible $\decpar$, \gls{dec} is a diffeomorphism of $\RR^d$ with inverse \gls{invdec}.%
\end{enumerate}
\end{assum}

Existence of at least one global minimizer of the gap between true and variational posterior is given by the following proposition. 
\begin{prop}[Existence of global minimum]\label{prop:exist_infoproj}
Under \autoref{assum:VAEminim}. For a fixed $\gls{decpar}$
assume additionally that $\gls{invdec}$ is Lipschitz continuous with Lipschitz constant $B>0$, in the sense that
\[
\forall \gls{obs},\boldsymbol{y} \in \RR^d: \qquad  \left\|\unmix-\gls{invdec}(\boldsymbol{y})\right\|_2\leq B \|\gls{obs}-\boldsymbol{y}\|_2\,.
\] 
Then there exists at least one choice $(\encmean\in \RR^d,\,\encstd \in \RR^{d}_{>0})$ that achieves the minimum of $\kl{\gls{q}}{\gls{pzx}}$.
\end{prop}
\begin{proof}
Using Prop.~\ref{prop:KLlowerLipschitz}, we have the lower bound
\begin{multline}
    \kl{\gls{q}}{\gls{pzx}}
\geq -\sumk[d]\left[\log \encstdcomp_k (\gls{obs})+\log m(\encmeancomp_k)\right]+c(\gls{obs},\gamma)\\+\frac{\gsq}{2}B^{-2}\left[\left\|\unmix-\gls{mu} \right\|^2+\sumk[d]\gls{sigmak}\right] \,.
\end{multline}
We then notice  (see lemma~\ref{lem:optim}) that for all $k$, 
\[\encstdcomp_k(\gls{obs})\to -\log \encstdcomp_k (\gls{obs})+\frac{\gsq}{2}B^{-2}\gls{sigmak}\] achieves a global minimum $n(B,\gamma)=-\log(B/\gamma)+1/2$ at $\encstdcomp_k(\gls{obs})=B/\gamma$.

For arbitrary $k_0$, we now 1) lower bound the $k\neq k_0$ terms by $n(B,\gamma)$; 2) lower bound and all the $\log m$ terms by their global maximum, which exists by Assum.~1i (log-concave prior); and 3) drop the non-negative squared norm term, leading to the following weaker lower bound:
\begin{multline}
    \kl{\gls{q}}{\gls{pzx}}
\geq (d-1)n(B,\gamma) -\log \encstdcomp_{k_0} (\gls{obs})\\-d  \max_t \left(\log  m(t)\right) +c(\gls{obs},\gamma)+\frac{\gsq}{2}B^{-2}\left[\encstdcomp_{k_0}(\gls{obs})^2\right] \,.
\end{multline}

The \gls{kld} divergence is well-defined and finite for any choice of parameters in their domain, therefore it achieves a particular value $K_0\geq 0$ at one arbitrary selected point of the domain. 
Since for all $k$, the lower bound tends to $+\infty$ for both $\encstdcomp_k\to +\infty$ (as the quadratic term dominates the $-\log$ term) and $\encstdcomp_k\to 0^+$, there exist $a>b>0$ (possibly dependent on $(\gamma,\gls{obs})$) such that ${ \kl{\gls{q}}{\gls{pzx}}>K_0 }$ for any $\encstdcomp_k<b$ or $\encstdcomp_k>a$.

Moreover, starting again from the lower bound from Prop.~\ref{prop:KLlowerLipschitz},
\begin{multline}
    \kl{\gls{q}}{\gls{pzx}}
\geq -\sumk[d]\left[\log \encstdcomp_k (\gls{obs})+\log m(\encmeancomp_k)\right]+c(\gls{obs},\gamma)\\+\frac{\gsq}{2}B^{-2}\left[\left\|\unmix-\gls{mu} \right\|^2+\sumk[d]\gls{sigmak}\right] \,,
\end{multline}
we now focus on $\encmean$ and lower bound all $\encstd$ terms. With this, we get the following weaker lower bound in terms of $\encmean$:
\begin{multline}
    \kl{\gls{q}}{\gls{pzx}}
\geq d n(B,\gamma) -d  \max_t \left(\log  m(t)\right) +c(\gls{obs},\gamma)\\+\frac{\gsq}{2}B^{-2}\left[\left\|\unmix-\gls{mu} \right\|^2\right] \,.
\end{multline}

The lower bound also tends to $+\infty$ for $\|\encmean\| \to +\infty$, so there exists a radius $R>0$ (possibly dependent on $(\gamma,\gls{obs})$) such that $ \kl{\gls{q}}{\gls{pzx}}>K_0 $ if $\|\encmean\|>R$. 

As a consequence, the infimum ($\leq K_0$) of the minimization problem (\ref{eq:minKLopt}) cannot be achieved outside the compact set $(\encmean ,\encstd)\in  \{\encmean \in \RR^d: \|\encmean\|\leq R\}\times[a,b]^d$. Since the divergence is continuous in $(\encmean,\encstd)$, there exists a value $(\encmeanopt ,\encstdopt)$ in this compact set achieving the minimum of the KL over the whole parameter domain, and all values achieving this minimum are in this compact set. 
\end{proof}

For given $\gls{obs}$, $\decpar$ and $\gamma>0$, the variational posterior KL divergence mapping
\[
(\gls{mu},\encstd(\gls{obs}))\to \kl{\gls{q}}{\gls{pzx}}
\]
thus has a minimum, and by smoothness of this mapping, this minimum can be characterized by the vanishing gradient of the KL divergence with respect to the parameters. 
Now, let us try to characterize how this minimum behaves for large $\gamma$.
\begin{prop}[Self-consistency of the encoder in the deterministic limit]\label{prop:selfconsist_app}
Under ~\cref{assum:VAEminim}, assume additionally $\gls{dec}$ and $\gls{invdec}$ are Lipschitz continuous with respective Lipschitz constants $C,B>0$, in the sense that
\begin{align}
    \forall \gls{latent},\boldsymbol{w} \in \RR^d: \qquad  &\left\|\gtmix-\gls{dec}(\boldsymbol{w})\right\|_2\leq C \|\gls{latent}-\boldsymbol{w}\|_2\,,
    \\
    \forall \gls{obs},\boldsymbol{y} \in \RR^d: \qquad &\left\|\unmix-\gls{invdec}(\boldsymbol{y})\right\|_2\leq B \|\gls{obs}-\boldsymbol{y}\|_2\,.
\end{align} 
Assume additionally that $-\log m$ is quadratically dominated, in the sense that
\[
\exists D>0,E>0: \qquad   -\log m(u) \leq D|u|^2 +E\,, \qquad \forall u \in \RR.
\]
Then for all $\gls{obs},\decpar$, as $\gamma\to +\infty$, any global minimum of (\ref{eq:minKLopt}) satisfies
\begin{eqnarray}
\gls{muopt} = \unmix+O(\nicefrac{1}{\gamma})\label{eq:meanconv}\\
\gls{sigmaopt} = O(\nicefrac{1}{\gsq})\label{eq:stdconv}\,.
\end{eqnarray}
More precisely, for all $\gls{obs}\in\RR^{\gls{obsdim}}$, $\gamma>0$
\begin{multline*}
 \left\|\unmix-\gls{muopt} \right\|^2
\leq  B^2\frac{2d}{\gsq}\left(\frac{1}{2}(C^2-1)+E+D\left[\dfrac{\|\unmix\|^2}{d}+\frac{1}{\gsq} \right]\right.\\ \left.+M+\frac{1}{2}\log (B^2)\right)\,.
\end{multline*}
and 
\begin{multline*}
 \sumk[d]\gls{sigmaoptk}
\leq  B^{2}\frac{4d}{\gsq}\left(\frac{1}{2}(C^2-1)+E+D\left[\dfrac{\|\unmix\|^2}{d}+\frac{1}{\gsq} \right]+M+\frac{1}{2}\left(\log (2 B^{2})\right)\right)\,.
\end{multline*}

\end{prop}
\begin{proof}
We start from the lower bound expression of Prop.~\ref{prop:KLlowerLipschitz}
\begin{multline*}
    \kl{\gls{q}}{\gls{pzx}}
\geq -\sumk[d]\left[\log \encstdcomp_k (\gls{obs})+\log m(\encmeancomp_k)\right]+c(\gls{obs},\gamma)\\+\frac{\gsq}{2}B^{-2}\left[\left\|\unmix-\encmean \right\|^2+\sumk[d]\gls{sigmak}\right] \,,
\end{multline*}
with $c(\gls{obs},\gamma)= -\frac{d}{2}\left(\log (\gsq)+1\right)+\log \gls{px}$.
For any $\nu\in (0,1]$,
we can thus write
\begin{multline*}
    \kl{\gls{q}}{\gls{pzx}}
\geq \sumk[d]\left[-\log \encstdcomp_k (\gls{obs})+\nu\gsq B^{-2}\frac{\gls{sigmak}}{2}-\log m(\encmeancomp_k)\right]+c(\gls{obs},\gamma)\\+\frac{\gsq}{2}B^{-2}\left[\left\|\unmix-\encmean \right\|^2+(1-\nu)\sumk[d]\gls{sigmak}\right] \,.
\end{multline*}
Now, from lemma~\ref{lem:optim} we get 
\[
\forall u>0: \qquad  -\log u +\alpha u^2/2\geq \frac{1}{2}\log(\alpha)+\frac{1}{2} \,.
\]
We exploit this lower bound to obtain
\begin{multline*}
    \kl{\gls{q}}{\gls{pzx}}
\geq  \frac{d}{2}\left(\log (\nu \gsq B^{-2})+1\right)-\sumk[d]\left[\log m(\encmeancomp_k)\right]+c(\gls{obs},\gamma)\\+\frac{\gsq}{2}B^{-2}\left[\left\|\unmix-\encmean \right\|^2+(1-\nu)\sumk[d]\gls{sigmak}\right] \,.
\end{multline*}
Using the expression of $c(\gls{obs},\gamma)$ we get
\begin{multline*}
    \kl{\gls{q}}{\gls{pzx}}
    \geq  \frac{d}{2}\left(\log (\nu B^{-2})+ \log\gsq +1\right)-\sumk[d]\left[\log m(\encmeancomp_k)\right] -\frac{d}{2}\left(\log\gsq+1\right) \\+\log \gls{px}+\frac{\gsq}{2}B^{-2}\left[\left\|\unmix-\encmean \right\|^2+(1-\nu)\sumk[d]\gls{sigmak}\right] \,.
\end{multline*}
and both the ``$d\log \gamma$'' as well as ``$\nicefrac{d}{2}$'' terms cancel out such that
\begin{multline*}
    \kl{\gls{q}}{\gls{pzx}}
\geq \frac{d}{2}\left(\log (\nu B^{-2} )\right)-\sumk[d]\left[\log m(\encmeancomp_k)\right]  +\log \gls{px}\\+\frac{\gsq}{2}B^{-2}\left[\left\|\unmix-\encmean \right\|^2+(1-\nu)\sumk[d]\gls{sigmak}\right] \,.
\end{multline*}
Finally, using Prop.~\ref{prop:KLupperLipschitz}, the above right hand side is bounded from above by a constant as $\gamma \to +\infty$, and as a consequence, the positive factor of the $\gsq$ term must vanish (by continuity assumption and its limits note $-\log m$ is bounded from below)
\[
    \left\|\unmix-\encmean \right\|^2+(1-\nu)\sumk[d]\gls{sigmak}\to 0
\]
This entails that both positive terms it comprises must vanish too.  

More precisely, we get the inequality between lower and upper bounds at the optimal solution
\begin{multline*}
    \frac{d}{2}\left(\log (\nu B^{-2} )\right)-\sumk[d]\left[\log m(\encmeancompopt_k)\right]  +\log \gls{px}\\+\frac{\gsq}{2}B^{-2}\left[\left\|\unmix-\gls{muopt} \right\|^2+(1-\nu)\sumk[d]\gls{sigmaoptk}\right] \\
    \leq d \left(\frac{1}{2}C^2+E+D\left[\dfrac{\|\unmix\|^2}{d}+\frac{1}{\gsq} \right]\right)-\frac{d}{2} +\log \gls{px},
\end{multline*}
which simplifies to
\begin{multline*}
    \frac{d}{2}\left(\log (\nu B^{-2} )\right)-\sumk[d]\left[\log m(\encmeancompopt_k)\right]  +\frac{\gsq}{2}B^{-2}\left[\left\|\unmix-\gls{muopt} \right\|^2+(1-\nu)\sumk[d]\gls{sigmaoptk}\right] \\
    \leq d \left(\frac{1}{2}C^2+E+D\left[\dfrac{\|\unmix\|^2}{d}+\frac{1}{\gsq} \right]\right)-\frac{d}{2}. 
\end{multline*}

Moreover by continuity assumption and its limits, $-\log m$ is bounded from below by ${-M=-\max_t \log m(t)}$, yielding
\begin{multline*}
\frac{d}{2}\left(\log (\nu B^{-2} )-2M \right) +\frac{\gsq}{2}B^{-2}\left[\left\|\unmix-\gls{muopt} \right\|^2+(1-\nu)\sumk[d]\gls{sigmaoptk}\right]\\
\leq d \left(\frac{1}{2}(C^2-1)+E+D\left[\dfrac{\|\unmix\|^2}{d}+\frac{1}{\gsq} \right]\right)
\end{multline*}
such that 
\begin{multline*}
 \frac{\gsq}{2}B^{-2}\left[\left\|\unmix-\gls{muopt} \right\|^2+(1-\nu)\sumk[d]\gls{sigmaoptk}\right]\\
\leq d \left(\frac{1}{2}(C^2-1)+E+D\left[\dfrac{\|\unmix\|^2}{d}+\frac{1}{\gsq} \right]-\frac{1}{2}\left(\log (\nu B^{-2} )-2M\right)\right)
\end{multline*}
and finally
\begin{multline}\label{eq:fullbound}
 B^{-2}\left[\left\|\unmix-\gls{muopt} \right\|^2+(1-\nu)\sumk[d]\gls{sigmaoptk}\right]\\
\leq \frac{2d}{\gsq}\left(\frac{1}{2}(C^2-1)+E+D\left[\dfrac{\|\unmix\|^2}{d}+\frac{1}{\gsq} \right]+M+\frac{1}{2}\log (B^2/\nu ) \right)
\end{multline}
Taking $\nu=1$ in (\ref{eq:fullbound}) we get the first intended inequality 
\begin{multline*}
 \left\|\unmix-\gls{muopt} \right\|^2
\leq B^2\frac{2d}{\gsq}\left(\frac{1}{2}(C^2-1)+E+D\left[\dfrac{\|\unmix\|^2}{d}+\frac{1}{\gsq} \right]\right.\\ \left.+M+\frac{1}{2}\log (B^2)\right)\,.
\end{multline*}
Alternatively, (\ref{eq:fullbound}) implies
\begin{multline*}
 (1-\nu)\sumk[d]\gls{sigmaoptk}
\leq B^2\frac{2d}{\gsq}\left(\frac{1}{2}(C^2-1)+E+D\left[\dfrac{\|\unmix\|^2}{d}+\frac{1}{\gsq} \right]\right.\\ \left.+M+\frac{1}{2}\left(\log (B^2/\nu )\right)\right)
\end{multline*}

Taking a fixed value of $\nu$, say $\nicefrac{1}{2}$, we get the second intended inequality
\begin{multline*}
 \sumk[d]\gls{sigmaoptk}
\leq B^{2}\frac{4d}{\gsq}\left(\frac{1}{2}(C^2-1)+E+D\left[\dfrac{\|\unmix\|^2}{d}+\frac{1}{\gsq} \right]+M+\frac{1}{2}\left(\log (2 B^{2})\right)\right)\,.%
\end{multline*}
\end{proof}

We now restate the main text proposition and provide the proof.
\propselfconsist*

\begin{proof}
We only have to check that \cref{assum:VAE} allow fulfilling the following requirements of \cref{prop:selfconsist_app}:
\begin{itemize}[nolistsep]
    \item the Lipschitz continuity requirements in~\cref{prop:selfconsist_app} results from the boundedness of the first order derivatives of the decoder mean and of its inverse (by using the multivariate Taylor theorem),
    \item concavity of $\log m$, required by \cref{assum:VAEminim}, is a direct consequence of non-positivity of the second-order logarithmic derivative of $m$ in \cref{assum:VAE}i,
    \item quadratic domination of $-\log m$ comes from the boundedness of the second-order logarithmic derivative of $m$ (by integrating twice).
\end{itemize} 
Then \cref{prop:selfconsist_app} follows and the $O(\nicefrac{1}{\gamma})$ convergence of the variational posterior mean of the inverse, as well as the $O(\nicefrac{1}{\gsq})$ convergence of the variational posterior variance.
\end{proof}
\paragraph{Finer approximation of parameter values}
We now derive a finer result for the convergence of the mean, that we will exploit in \cref{prop:vae_ima}. This relies on the existence of an optimum shown by  \cref{prop:exist_infoproj}.

At such optimum $\encparopt$ we thus have for all $k$
\[
\frac{\partial}{\partial \encmeancomp_k} \left[\kl{\gls{q}}{\gls{pzx}}\right]_{|\encparopt}=0\,,
\]
and
\[
\frac{\partial}{\partial \encstdcomp_k} \left[\kl{\gls{q}}{\gls{pzx}}\right]_{|\encparopt} =0\,.
\]

We derive the constraints entailed by the first expression:
\begin{multline*}
    \frac{\partial}{\partial \encmeancomp_k} \left[\kl{\gls{q}}{\gls{pzx}}\right]_{|\encparopt}=
     \frac{1}{2}\int \frac{\partial}{\partial \encmeancomp_k} q_{\encpar}(\gls{latent})\left[\|\gls{obs}-\gtmix\|^2\gsq  -2\sumk[d]\log m(\gls{latentcomp}_k)\right] d\gls{latent}
    \\
    =
     \frac{1}{2}%
     \int  \prod_{j\neq k} q_{\encpar}^j(z_j)\frac{\partial q_{\encpar}^k(\gls{latentcomp}_k)}{\partial \encmeancomp_k}\left[\|\gls{obs}-\gtmix\|^2\gsq  -2\sumk[d]\log m(\gls{latentcomp}_k)\right] d\gls{latent}
\end{multline*}
with 
\[
    \frac{\partial q_{\encpar}^k(\gls{latentcomp}_k)}{\partial \encmeancomp_{k}} = \frac{ \encmeancomp_{k}- \gls{latentcomp}_k}{{\encstdcomp_{k}}^{2}}q_{\encpar}^k(\gls{latentcomp}_k),
\]
which leads to a set of constraints at optimum
\begin{multline}\label{eq:postmeanderivconst}
     \int   q_{\encparopt}(\gls{latent}) \gls{muoptk} \left[\|\gls{obs}-\gtmix\|^2\gsq  -2\sumk[d]\log m(\gls{latentcomp}_k)\right] d\gls{latent}\\=
     \int  q_{\encparopt}(\gls{latent}) \gls{latentcomp}_k \left[\|\gls{obs}-\gtmix\|^2\gsq  -2\sumk[d]\log m(\gls{latentcomp}_k)\right] d\gls{latent}\,,\, \forall k
\end{multline}

Based on this expression we derive the following result. 

\begin{prop}\label{prop:selfconsmeansqrate}
Under \cref{assum:VAE}, as $\gamma\to +\infty$
\begin{equation}
        \gls{dec}(\gls{muopt})    =\gls{obs}+\frac{1}{\gsq}\gls{jacobian}_{\gls{dec}|\gls{muopt}}^{-T}  n'(\gls{muopt})+O(\nicefrac{1}{\gamma^3}).
\end{equation}
and 
\begin{equation}\label{eq:meanconvbias}
\gls{muopt}    =\unmix+\frac{1}{\gsq}\gls{jacobian}_{\gls{dec}|\unmix}^{-1}\gls{jacobian}_{\gls{dec}|\unmix}^{-T}  n'(\unmix)+O(\nicefrac{1}{\gamma^3})
\end{equation}
\end{prop}

\begin{proof}
We start from the constraints of~\eqref{eq:postmeanderivconst} that we rewrite
\begin{multline*}
     \int   q_{\encparopt}(\gls{latent})\left(\gls{latentcomp}_k - \gls{muoptk})\right) \left[\|\gls{obs}-\gtmix\|^2\gsq  \right] d\gls{latent}\\= \int   q_{\encparopt}(\gls{latent})\left(\gls{latentcomp}_k - \gls{muoptk})\right) \left[  2\sumk[d]\log m(\gls{latentcomp}_k)\right] d\gls{latent}\\
\end{multline*}
We then proceed to approximate the left hand side using a Taylor formula. Assuming bounded Hessian components, we can upper and lower bound using third order centered absolute moments of the Gaussian as
 \begin{multline*}
  \gsq  \int   q_{\encparopt}(\gls{latent})\left(\gls{latentcomp}_k - \gls{muoptk}\right) \left[\|\gls{obs}-\gls{dec}(\gls{muopt})-\gls{jacobian}_{\gls{dec}|\gls{muopt}}(\gls{latent}-\gls{muopt}) \|^2  \right] d\gls{latent} +O(\nicefrac{1}{\gamma}),
\end{multline*}
which we can rewrite (by 1) expanding the norm of the sum; 2) removing constants in the bracket, which lead to zeros after multiplying the zero mean variable and taking the expectation; 3) using Gaussianity, all centered third order terms vanish.)
 \begin{multline*}
   \gsq \int   q_{\encparopt}(\gls{latent})\left(\gls{latentcomp}_k - \gls{muoptk}\right) \left[\|\gls{obs}-\gls{dec}(\gls{muopt})\|^2+\|\gls{jacobian}_{\gls{dec}|\gls{muopt}}(\gls{latent}-\gls{muopt}) \|^2\right.\\ 
    \left.-2\left\langle \gls{obs}-\gls{dec}(\gls{muopt}),\,\gls{jacobian}_{\gls{dec}|\gls{muopt}}(\gls{latent}-\gls{muopt})\right\rangle  \right] d\gls{latent} +O(\nicefrac{1}{\gamma})\\
    = \gsq\int   q_{\encparopt}(\gls{latent})\left(\gls{latentcomp}_k - \gls{muoptk}\right) \left[\|\gls{jacobian}_{\gls{dec}|\gls{muopt}}(\gls{latent}-\gls{muopt}) \|^2 \right.\\
    \left.-2\left\langle \gls{obs}-\gls{dec}(\gls{muopt}),\,\gls{jacobian}_{\gls{dec}|\gls{muopt}}(\gls{latent}-\gls{muopt})\right\rangle   \right] d\gls{latent} +O(\nicefrac{1}{\gamma})\\
    = \gsq\int   q_{\encparopt}(\gls{latent})\left(\gls{latentcomp}_k - \gls{muoptk}\right) \left[(\gls{latent}-\gls{muopt})^T \gls{jacobian}_{\gls{dec}|\gls{muopt}}^T \gls{jacobian}_{\gls{dec}|\gls{muopt}}(\gls{latent}-\gls{muopt}) \right.\\
    \left.-2\left\langle \gls{obs}-\gls{dec}(\gls{muopt}),\,\gls{jacobian}_{\gls{dec}|\gls{muopt}}(\gls{latent}-\gls{muopt})\right\rangle \right] d\gls{latent} +O(\nicefrac{1}{\gamma})\\
     = \gsq\int   q_{\encparopt}(\gls{latent})\left(\gls{latentcomp}_k - \gls{muoptk}\right) \left[ -2\left\langle \gls{obs}-\gls{dec}(\gls{muopt}),\,\gls{jacobian}_{\gls{dec}|\gls{muopt}}(\gls{latent}-\gls{muopt})\right\rangle  \right] d\gls{latent} +O(\nicefrac{1}{\gamma})\\
\end{multline*}
Finally computing this integral we get the left hand side as
\begin{align*}
    -2\gsq \gls{sigmaoptk} \left\langle \gls{obs}-\gls{dec}(\gls{muopt}),\,[\gls{jacobian}_{\gls{dec}|\gls{muopt}}]_{.k}\right\rangle   +O(\nicefrac{1}{\gamma})\\
\end{align*}

For the right hand side we get using a Taylor expansion (with notation $n:\gls{latent}
\to \log(m(\gls{latent}))$)
\begin{multline*}
     \int   q_{\encparopt}(\gls{latent})\left(\gls{latentcomp}_k - \gls{muoptk})\right) \left[  2\sumk[d]\log m(\gls{latentcomp}_k)\right] d\gls{latent}\\= \int   q_{\encparopt}(\gls{latent})\left(\gls{latentcomp}_k - \gls{muoptk})\right) \left[  2\sumk[d]\log m(\gls{muoptk})+n'(\gls{muoptk})(\gls{latentcomp}_k-\gls{muoptk})\right]d\gls{latent} + O(\nicefrac{1}{\gsq})\\
     = 2\gls{sigmaoptk} n'(\gls{muoptk})+ O(\nicefrac{1}{\gsq}).
\end{multline*}
Equating the non-negligible terms of the left and right-hand sides we get for each $k$
\begin{equation*}
    \gsq  \left\langle \gls{obs}-\gls{dec}(\gls{muopt}),\,[\gls{jacobian}_{\gls{dec}|\gls{muopt}}]_{.k}\right\rangle   = -n'(\gls{muoptk})+O(\nicefrac{1}{\gamma})
\end{equation*}
such that
\begin{equation*}
      (\gls{obs}-\gls{dec}(\gls{muopt}))^T \gls{jacobian}_{\gls{dec}|\gls{muopt}}   = -\frac{1}{\gsq} n'(\gls{muopt})+O(\nicefrac{1}{\gamma^3}),
\end{equation*}
where $n'$ is applied component-wise. 
Because the Jacobian is everywhere invertible (implicit consequence of Lipschitz assumptions), we can solve for this equations and get
\begin{equation}\label{eq:selfconstbias_variationalmean}
      \gls{dec}(\gls{muopt})    =\gls{obs}+\frac{1}{\gsq}\gls{jacobian}_{\gls{dec}|\gls{muopt}}^{-T}  n'(\gls{muopt})+O(\nicefrac{1}{\gamma^3}).
\end{equation}
Using again a similar Taylor approximation we get 
\begin{equation*}
\gls{muopt}    =\unmix+\frac{1}{\gsq}\gls{jacobian}_{\gls{dec}|\gls{muopt}}^{-1}\gls{jacobian}_{\gls{dec}|\gls{muopt}}^{-T}  n'(\gls{muopt})+O(\nicefrac{1}{\gamma^3}).
\end{equation*}

This equation has the shortcoming of still referring to the posterior mean on both sides. To fix this, we first note that it implies, by boundedness of the Jacobian, that
\begin{equation*}
    |\gls{muopt}-\unmix|\leq\frac{1}{\gsq} K |n'(\gls{muopt})|+O(\nicefrac{1}{\gamma^3}).
\end{equation*}
By bounding the second-order derivative of the log prior, we get
\begin{equation*}
    |\gls{muopt}-\unmix|\leq\frac{1}{\gsq} K |n'(\unmix)+O(\gls{muopt}-\unmix)|+O(\nicefrac{1}{\gamma^3}),
\end{equation*}
which implies 
\begin{equation*}
    \gls{muopt}=\unmix+O(\nicefrac{1}{\gsq})\,,
\end{equation*}
\ie, we obtain an improved convergence rate. Using this rate and Taylor theorem, we obtain the final equation by replacing the variational posterior mean by the inverse decoder in \eqref{eq:selfconstbias_variationalmean} 
\begin{equation*}
\gls{muopt}    =\unmix+\frac{1}{\gsq}\gls{jacobian}_{\gls{dec}|\unmix}^{-1}\gls{jacobian}_{\gls{dec}|\unmix}^{-T}  n'(\unmix)+O(\nicefrac{1}{\gamma^3})
\end{equation*}

\end{proof}

\subsection{Proof of \autoref{prop:vae_ima}}
\label{subsec:app_vae_ima_log_conv}
This will be a corollary of the following result, that uses as a key assumption a rate of $O(\nicefrac{1}{\gsq})$ in the convergence of the self-consistency equation of the variational mean. %
\begin{prop}[\glspl{vae} with log-concave factorized prior and close-to-deterministic decoder approximate the \gls{ima} objective]\label{prop:vae_ima_selfcons}
Under \cref{assum:VAE}, if additionally the \gls{vae} satisfies the following self-consistency in the deterministic limit
\begin{eqnarray}
    \norm{\gls{muopt} -\unmix}&= & O_{\gamma\to +\infty}(\nicefrac{1}{\gsq})\,,\label{eq:meanselfconstassum}\\
    \normsquared{\gls{sigmaopt}} &= & O_{\gamma\to +\infty}(\nicefrac{1}{\gsq})\,.\label{eq:varselfconstassum}
\end{eqnarray}
then 
\begin{equation}
      \gls{sigmaoptk}=\parenthesis{-\frac{d^2\log p_0}{d\gls{latentcomp}_k^2}(g^{\decpar}_k(\gls{obs}))+\gsq\norm{\cols{\mixjacobianbis{\gls{invdec}(\gls{obs})}}} ^{2}}^{-1} +O(\nicefrac{1}{\gamma^3})\,,
\end{equation}
and the self-consistent \gls{elbo}~\eqref{eq:self_cons_elbo}
approximates the \gls{ima}-regularized log-likelihood~\eqref{eq:ima_objective}:%
\begin{align}%
    \gls{elbo}^* (\gls{obs};\gls{decpar})
    &=
    \log \gls{px} - \gls{cima_local}(\gls{dec}, \gls{invdec}(\gls{obs}))
    +O_{\gamma \to \infty}\left(\nicefrac{1}{\gsq}\right).
\end{align}
\end{prop}

\begin{proof}
    We start from the self-consistent \gls{elbo} decomposition as ``reconstruction error plus posterior regularization'' terms:
    \begin{equation}%
    \gls{elbo}^*(\gls{obs};\gls{decpar})=-\kl{\gls{qopt}}{\gls{pz}}+\expectation{\gls{qopt}}\brackets{\log \gls{pxz}},
    \end{equation}
    and continue with reformulating both terms, based on \cref{assum:VAE}. That is, $p_0$ is factorized with components i.i.d. distributed according to a fully supported \textbf{log-concave} density  $\gls{latentcomp}_k\sim m$.%

    \paragraph{Posterior regularization term}
    \cref{assum:VAE} gives us the formula of \eqref{eq:postreg} for this term in the \gls{elbo}. Taking optimal encoder parameters, we get the posterior regularization term for the $\gls{elbo}^*$
    \begin{equation*}
    -\kl{\gls{qopt}}{\gls{pz}}=\expectation{\gls{qopt}}[\log(\gls{pz})]+\frac{1}{2}\sumk[d]\brackets{\log{\gls{sigmaoptk}}}+ \kappa_d\,,
    \end{equation*}
with   $\kappa_d=\frac{d}{2}\left(\log (2\pi) +1\right).$ 
    Using the factorized Gaussian encoder and i.i.d. prior assumptions we get %
    \begin{multline*}
    -\kl{\gls{qopt}}{\gls{pz}}=\sumk[d]\expectation{\gls{latentcomp}_k\sim \mathcal{N}(\gls{muoptk},\gls{sigmaoptk})}[\log(m(\gls{latentcomp}_k))]+\frac{1}{2}\sumk[d]\brackets{\log \gls{sigmaoptk}}+ \kappa_d\,,
    \end{multline*}
    where we rewrote the distribution $p_0$ as $p_0=\prod_k m(\gls{latentcomp}_k)$.

    Based on the Taylor theorem, with a residual in Lagrange form of $n=\log m$,  we have that for all $k$ and $u$ there exists $\xi \in [\gls{muk},u]$ if $u\geq \gls{muk}$, or $\xi \in [u, \gls{muk}]$ if $u\leq \gls{muk}$ such that
    \begin{multline*}
    n(u)=\log(m(u))=\log(m(\gls{muoptk}))+n'(\gls{muoptk})(u-\gls{muoptk})\\
    +\frac{1}{2}n''(\gls{muoptk})(u-\gls{muoptk})^2+\frac{1}{3!}n^{(3)}(\xi)(u-\gls{muoptk})^3
    \end{multline*}
    We assumed that $|n^{(3)}|$ is bounded over $\RR$ by $F$, such that 
    \begin{multline*}
 -F\left|u-\gls{muoptk}\right|^3 \leq    \log(m(u))-\log(m(\gls{muoptk}))-n'(\gls{muoptk})(u-\gls{muoptk})\\
    -\frac{1}{2}n''(\gls{muoptk})(u-\gls{muoptk})^2 \leq F\left|u-\gls{muoptk}\right|^3\,.
    \end{multline*}
    Taking the expectation and using the expression of centered Gaussian absolute moments\footnote{see e.g. \url{https://arxiv.org/pdf/1209.4340}}
    \begin{multline}
        \left| \expectation{\gls{latentcomp}_k\sim \mathcal{N}(\gls{muoptk},\gls{sigmaoptk})}[\log(m(\gls{latentcomp}_k))] - 
        \log(m(\gls{muoptk}))-\frac{1}{2}n''(\gls{muoptk})\gls{sigmaoptk}\right|\\\leq F \expectation{} \left[ \left| u-\gls{muoptk} \right|^3 \right]
        =F \sigmaoptkcube \frac{2^{3/2}}{\sqrt{\pi}} \,.
    \end{multline}
    As the assumptions entail that optimal posterior variances $\gls{sigmaoptk}$ get small for $\gamma$ large (\cf \eqref{eq:varselfconstassum}), this implies the near-deterministic approximation
    \[
    \expectation{\gls{latentcomp}_k\sim \mathcal{N}(\gls{muoptk},\sigma_k(\gls{obs})^2)}[\log(m(\gls{latentcomp}_k))] = 
    \log(m(\gls{muoptk}))+\frac{1}{2}n''(\gls{muoptk})\gls{sigmaoptk}+O_{\gamma\to +\infty}(\nicefrac{1}{\gamma^3})\,.
    \]
In addition, using again a Taylor formula and the self-consistency assumption for the mean 
\begin{multline*}
    \log(m(\gls{muoptk}))=\log(m(\invdeccompk))+n'(\invdeccompk)(\gls{muoptk}-\invdeccompk)+O_{\gamma\to +\infty}(\nicefrac{1}{\gsq})\\
    =\log(m(\invdeccompk))+O_{\gamma\to +\infty}(\nicefrac{1}{\gsq}).
\end{multline*}
Moreover, using again a Taylor formula for $n''$ under boundedness of $n^{(3)}$ and again using the self-consistency assumption for the mean yields
\begin{equation*}
n''(\gls{muoptk})=n''(\invdeccompk)+O(\gls{muoptk}-\invdeccompk)
=n''(\invdeccompk)+O_{\gamma\to +\infty}(\nicefrac{1}{\gsq})\,.
\end{equation*}

    Overall this leads to the approximation of the posterior regularization term
    \begin{multline}\label{eq:postregapprox}
    -\kl{\gls{qopt}}{\gls{pz}}=\sumk[d]\log(m(\invdeccompk))+\frac{1}{2}n''(\invdeccompk)\gls{sigmaoptk}+\frac{1}{2}\log \gls{sigmaoptk}\\+\kappa_d +O_{\gamma\to +\infty}(\nicefrac{1}{\gsq})\,.
    \end{multline}

\paragraph{Reconstruction term}
Now switching to the first (reconstruction) term of the \gls{elbo}$^*$, adapting the decomposition of \eqref{eq:rec_term} by using optimal encoder parameters we get
\begin{equation*}%
\expectation{\gls{qopt}}\brackets{\log \gls{pxz}}=-\frac{\gsq }{2}\expectation{\gls{qopt}}\brackets{\|\gls{obs}-\gtmix\|^2}+d\log\gamma-\frac{d}{2}\log(2\pi).
\end{equation*}

Then in the small encoder noise limit $\sigma_{k}(\gls{obs})^{2}\ll 1,\forall k$ (justified by~\cref{prop:selfconsist}), we rely on a Taylor approximation around the posterior mean $\gls{latent}^o=\gls{mu}$ based on~\autoref{lem:taylormult}, which bounds this approximation as follows
\begin{equation}\label{eq:taylorbound}
\expectation{\gls{q}}\brackets{\normsquared{ \gtmix - \gls{dec}(\gls{muopt}) - \sumk[d]\frac{\partial \gls{dec}}{\partial \gls{latentcomp}_k}_{|\gls{latent}^o} (\gls{latentcomp}_k-\gls{muoptk})}} \leq   \frac{d^3}{4} 3 K^2  \sum_{i}\encstdcompopt_{i}(\gls{obs})^4\,.
\end{equation}
The linear term in this approximation is easily computed  using successively~\autoref{lem:normexpect} and \autoref{lem:tr_cov} to get an expression with the squared column norms of the partial derivatives scaled by the standard deviations $\frac{\partial \gls{dec}}{\partial z_{k}}_{|\gls{muk}}$. We get 
    \begin{multline}\label{eq:lintermnorm}
    \expectation{\gls{q}}\brackets{\normsquared{\sumk[d]\frac{\partial \gls{dec}}{\partial \gls{latentcomp}_k}_{|z^o} (\gls{latentcomp}_k-\gls{muk})}}=\mbox{trace}\brackets{\mbox{Cov}\brackets{\sumk[d]\derivative{\gls{dec}}{z_{k}}_{|\gls{muk}}(\gls{latentcomp}_k-\gls{muk})}}\\=\sumk[d]\brackets{\norm{ \frac{\partial \gls{dec}}{\partial z_{k}}_{|\gls{muk}}} ^{2}\gls{sigmak}}\,.
    \end{multline}
  This term can be used as an approximation for the expectation term in the reconstruction loss thanks to the following reverse triangle inequality 
  \begin{multline*}
       \left| \expectation{\gls{q}}\brackets{\|\gls{obs}-\gtmix\|^2}- \expectation{\gls{q}}\brackets{\normsquared{\sumk[d]\frac{\partial \gls{dec}}{\partial \gls{latentcomp}_k}_{|z^o} (\gls{latentcomp}_k-\gls{muk})}}\right|\\
       =\left| \expectation{\gls{q}}\brackets{\|\gls{obs}-\gtmix\|^2}-\sumk[d]\brackets{\norm{ \frac{\partial \gls{dec}}{\partial z_{k}}_{|\gls{muk}}} ^{2}\gls{sigmak}}\right|\\
         \leq \expectation{\gls{q}}\brackets{\normsquared{\gls{obs}-\left(\gtmix-\sumk[d]\frac{\partial \gls{dec}}{\partial \gls{latentcomp}_k}_{|z^o} (\gls{latentcomp}_k-\gls{muk})\right)}}\,,
    \end{multline*}
    such that the resulting upper bound can be itself bounded as follows
    \begin{multline*}
       \expectation{\gls{q}}\brackets{\normsquared{\gls{obs}-\left(\gtmix-\sumk[d]\frac{\partial \gls{dec}}{\partial \gls{latentcomp}_k}_{|\gls{latent}^o} (\gls{latentcomp}_k-\gls{muk})\right)}}\\\leq \expectation{\gls{q}}\brackets{\normsquared{\gls{obs}-\gls{dec}(\gls{mu})}}+\expectation{\gls{q}}\brackets{\normsquared{ \gtmix - \gls{dec}(\gls{mu}) - \sumk[d]\frac{\partial \gls{dec}}{\partial \gls{latentcomp}_k}_{|\gls{mu}} (\gls{latentcomp}_k-\gls{muk})}}\,.
    \end{multline*}

Each term of the upper bound can be bounded for the optimum encoder parameters: using from left to right the assumption of  \eqref{eq:meanselfconstassum} and \eqref{eq:taylorbound}, %
respectively, leading to
  \begin{multline*}
        \left| \expectation{\gls{qopt}}\brackets{\|\gls{obs}-\gtmix\|^2}-\sumk[d]\brackets{\norm{ \frac{\partial \gls{dec}}{\partial z_{k}}_{|\gls{muoptk}}} ^{2}\gls{sigmaoptk}}\right|\\
      \leq   O_{\gamma\to +\infty}(\nicefrac{1}{\gamma^4})+\frac{d^3}{4} 3 K^2  \sum_{i}\encstdcompopt_{i}(\gls{obs})^4\,.
    \end{multline*}

    Getting back to the whole reconstruction term, using additionally the variance self-consistency assumption \eqref{eq:varselfconstassum}, the above shows that we can make the approximation
    \begin{equation*}
    \expectation{\gls{qopt}}\brackets{\log \gls{pxz}}=-\frac{\gsq }{2}\sumk[d]\brackets{\norm{ \frac{\partial \gls{dec}}{\partial z_{k}}_{|\gls{muoptk}}} ^{2}\gls{sigmaoptk}}+d\log\gamma-\frac{d}{2}\log(2\pi)+ O_{\gamma\to +\infty}(\nicefrac{1}{\gsq})
    \end{equation*}
    We can further replace the dependency of the derivatives on the encoder mean using a Taylor formula for the derivative
    \[
    \frac{\partial \gls{dec}}{\partial z_{k}}_{|\gls{muopt}} =     \frac{\partial \gls{dec}}{\partial z_{k}}_{|\gls{invdec}(\gls{obs})}+O(\gls{muopt}-\gls{invdec}(\gls{obs})) =     \frac{\partial \gls{dec}}{\partial z_{k}}_{|\gls{invdec}(\gls{obs})}+O(\nicefrac{1}{\gsq})
    \]
    such that
     \begin{multline}\label{eq:reconstapprox}
    \expectation{\gls{qopt}}\brackets{\log \gls{pxz}}=-\frac{\gsq }{2}\sumk[d]\brackets{\norm{ \frac{\partial \gls{dec}}{\partial z_{k}}_{|\gls{invdec}(\gls{obs})}} ^{2}\gls{sigmaoptk}}+d\log\gamma\\-\frac{d}{2}\log(2\pi)+ O_{\gamma\to +\infty}(\nicefrac{1}{\gsq})
    \end{multline}

    \paragraph{\gls{elbo}$^*$ approximation}
    As a consequence of \eqref{eq:postregapprox} and \eqref{eq:reconstapprox} the \gls{elbo}$^*$ becomes 
    \begin{multline*}
    \gls{elbo}^*(\gls{obs};\decpar)=-\frac{1}{2}\sumk[d]\left[\log\frac{1}{\gls{sigmaoptk}}+\gls{sigmaoptk}\parenthesis{- n''(\gls{invdeccomp}_k(\gls{obs}))+\gsq \norm{ \frac{\partial \gls{dec}}{\partial z_{k}}_{|\gls{invdeccomp}_k}} ^{2}}\right.\\
     -2\log(m(\gls{invdeccomp}_k(\gls{obs})))\Bigg]+d\log\gamma+\kappa_d-\frac{d}{2}\log(2\pi)+O_{\gamma \to \infty}(\nicefrac{1}{\gsq}) \\
     =-\frac{1}{2}\sumk[d]\left[\log\frac{1}{\gls{sigmaoptk}}-1+\gls{sigmaoptk}\parenthesis{- n''(\gls{invdeccomp}_k(\gls{obs}))+\gsq \norm{ \frac{\partial \gls{dec}}{\partial z_{k}}_{|\gls{invdec}(\gls{obs})}} ^{2}}\right.\\
     -2\log(m(\gls{invdeccomp}_k(\gls{obs})))\Bigg]+d\log\gamma+O_{\gamma \to \infty}(\nicefrac{1}{\gsq}) 
     \\
      =\widehat{\gls{elbo}}(\gls{sigmaopt};\gls{obs},\decpar,\encparopt) +\sumk[d]\log(m(\gls{invdeccomp}_k(\gls{obs})))+O_{\gamma \to \infty}(\nicefrac{1}{\gsq})
     \,,
    \end{multline*}
    where we isolated the terms that depend on parameters $\gls{sigmaoptk}$ and $\gamma$ in the approximate objective $\widehat{\gls{elbo}}(\boldsymbol{\sigma}^2=\gls{sigmaopt};\gls{obs},\decpar,\encparopt)$ that we define for arbitrary $\boldsymbol{\sigma}^2$. 
    \begin{multline*}
    \widehat{\gls{elbo}}(\boldsymbol{\sigma}^2;\gls{obs},\decpar,\encparopt)=
    -\frac{1}{2}\sumk[d]\left[\log\frac{1}{\gsq\sigma_k^2}-1+\sigma_k^2\parenthesis{- n''(\gls{invdeccomp}_k(\gls{obs}))+\gsq \norm{ \frac{\partial \gls{dec}}{\partial z_{k}}_{|\gls{invdec}(\gls{obs})}} ^{2}}\right]\\
    =\sumk[d]\widehat{\gls{elbo}}_k(\sigma_k^2;\gls{obs},\decpar,\encparopt)
    \end{multline*}
    Where we further break this objective in \gls{obsdim} components  $\widehat{\gls{elbo}}_k(\gls{sigmaoptk};\gls{obs},\decpar,\encparopt)$ according to the terms of the sum as follows
    \begin{multline*}
    \widehat{\gls{elbo}}_k(\sigma_k^2;\gls{obs},\decpar,\encparopt)=
    -\frac{1}{2}\left[\log\frac{1}{\gsq\sigma_k^2}-1+\gsq\sigma_k^2\parenthesis{-\frac{1}{\gsq} n''(\gls{invdeccomp}_k(\gls{obs}))+ \norm{ \frac{\partial \gls{dec}}{\partial z_{k}}_{|\gls{invdec}(\gls{obs})}} ^{2}}\right]\\
    \end{multline*}
    
    and
    where we note that $-n''\geq 0$ due to the log-concavity assumption. 

    Solving term in $k$ $\widehat{\gls{elbo}}_k(\sigma_k^2)$ for optimal $\gsq\sigma_k^*$ we get (see lemma~\ref{lem:optim}):
    \begin{align}
        \gsq\sigma_k^{*2}=\parenthesis{-\frac{1}{\gsq}n''(\gls{invdeccomp}_k(\gls{obs}))+\norm{ \frac{\partial \gls{dec}}{\partial z_{k}}_{|\gls{invdeccomp}_k(\gls{obs})}} ^{2}}^{-1} %
    \end{align}
    and the resulting optimal value $\widehat{\gls{elbo}}_k^*(\gls{obs},\decpar,\encparopt)=\widehat{\gls{elbo}}_k(\sigma_k^{*2};\gls{obs},\decpar,\encparopt)$ is
      \[
    \widehat{\gls{elbo}}_k^*(\gls{obs},\decpar,\encparopt)^*=-\frac{1}{2}{\log \parenthesis{-\frac{1}{\gsq}n''(\gls{invdeccomp}_k(\gls{obs}))+\norm{ \frac{\partial \gls{dec}}{\partial z_{k}}_{|\gls{invdeccomp}_k(\gls{obs})}} ^{2}}}
    \]
    A Taylor formula around this optimum leads, for some value  $\xi_\gamma(\gls{obs})$ lying between $\sigma_k^{*2}$ and $\sigma_k^2$ to (note the first order derivative vanishes, and the second order derivative is upper bounded hence the second line)
    \begin{multline*}
     \widehat{\gls{elbo}}_k(\sigma_k^2;\gls{obs},\decpar,\encparopt) = \widehat{\gls{elbo}}_k^*(\decpar,\encparopt) + \frac{d \widehat{\gls{elbo}}_k(\gls{obs};\decpar,\encparopt) }{d \gsq\sigma_k^2}_{|\sigma_k^{*2}}(\gsq\sigma_k^2-\gsq\sigma_k^{*2})\\+\frac{d^2 \widehat{\gls{elbo}}_k(\gls{obs};\decpar,\encparopt) }{d (\gsq\sigma_k^2)^2}_{|\xi_{\gamma}(\gls{obs})}(\gsq\sigma_k^2-\gsq\sigma_k^{*2})^2\\
     \leq \widehat{\gls{elbo}}_k^*(\decpar,\encparopt) - \frac{1}{2} \norm{ \frac{\partial \gls{dec}}{\partial z_{k}}_{|\gls{invdec}(\gls{obs})}} ^{2}(\gsq\sigma_k^2-\gsq\sigma_k^{*2})^2
    \end{multline*}

as a consequence the non-approximate solution for the true optimal $\gls{elbo}^*$, as $\gamma$ grows, must achieve a value  below this quadratic function, up to a term in $O(\nicefrac{1}{\gsq})$, and at the same time above $\widehat{\gls{elbo}}^*$, also up to a term in $O(\nicefrac{1}{\gsq})$. This entails that it is restricted to a smaller and smaller domain near the approximate solution and we get 
\begin{align}
    \gls{sigmaoptk}=\sigma_k^{*2}+O(\nicefrac{1}{\gamma^3})=\parenthesis{-n''(\gls{invdeccomp}_k(\gls{obs}))+\gsq\norm{ \frac{\partial \gls{dec}}{\partial z_{k}}_{|\gls{invdeccomp}_k(\gls{obs})}} ^{2}}^{-1} +O(\nicefrac{1}{\gamma^3}).\label{eq:optimal_sigma}
\end{align}
Leading to the approximation of the true objective
\begin{align*}
    \gls{elbo}^*(\gls{obs};\gls{decpar})=\!-\frac{1}{2}\sumk[d]\brackets{\log\! \parenthesis{\!-\frac{1}{\gsq}n''(\gls{muk})+\norm{ \frac{\partial \gls{dec}}{\partial z_{k}}_{|\gls{muk}}} ^{2}}\!\!\!-2\log(m(\gls{muk}))}\!\!+O(\nicefrac{1}{\gsq}),
\end{align*}
which reduces to

 \[
    \gls{elbo}^*(\gls{obs};\gls{decpar})=\log p_0(\gls{invdec}(\gls{obs}))-\frac{1}{2}\sumk[d]\brackets{\log \norm{ \cols{\gls{jacobian}_{\gls{dec}}(\gls{invdec}(\gls{obs}))}} ^{2}}+O(\nicefrac{1}{\gsq}),
\]
which is the \gls{ima} objective.

\end{proof}

We now restate the main text theorem and provide its proof.
\theoremvaeima*

\begin{proof}
This is just a corollary of \autoref{prop:vae_ima_selfcons} because \autoref{prop:selfconsmeansqrate} entails through~\eqref{eq:meanconvbias} the required $O(\nicefrac{1}{\gsq})$ rate of convergence for the optimal variational mean in~\eqref{eq:meanselfconstassum}, while~\eqref{eq:varselfconstassum} is fulfilled through~\cref{prop:selfconsist}.
\end{proof}

\section{Auxiliary results}

\subsection{Squared norm statistics}
\begin{lem}[Squared norm variance decomposition]\label{lem:normexpect}
For multivariate RV $X$ with mean $m$
\[
\mathbb{E}\left[ \left\| X\right\|^2\right]
= \mbox{trace}\left[\mbox{Cov}(X)\right]+\left\| m\right\|^2
\]

\end{lem}
\begin{proof}
\[
\mathbb{E}\|X-m\|^2=\mathbb{E}\left\langle X-m,\,X-m\right\rangle
=\mathbb{E}\left[\left\langle X,\,X\right\rangle-2\mathbb{E}\left\langle m,\,X\right\rangle+\left\langle m,\,m\right\rangle\right]
\]
hence
\[
\mathbb{E}\|X-m\|^2=\mathbb{E}\left[\left\| X\right\|^2\right]-\left\| m\right\|^2
\]
This leads to (using that the trace of a scalar is the scalar itself)
\[
\mathbb{E}\left[ \left\| X\right\|^2\right]
=\mathbb{E}\left[\mbox{trace}\left[\|X-m\|^2\right]\right]+\left\| m\right\|^2=\mbox{trace}\left[\mathbb{E}\left[(X-m)^T(X-m)\right]\right]+\left\| m\right\|^2
\]
because $\mbox{trace}[AB]=\mbox{trace}[BA]$ we get
\[
\mathbb{E}\left[ \left\| X\right\|^2\right]
=\mbox{trace}\left[\mathbb{E}\left[(X-m)(X-m)^T\right]\right]+\left\| m\right\|^2= \mbox{trace}\left[\mbox{Cov}(X)\right]+\left\| m\right\|^2
\]
\end{proof}

\begin{lem}[Trace of transformed unit covariance]\label{lem:tr_cov}
When the covariance matrix  $Cov(\boldsymbol{\epsilon})$ is the identity, then  
    \[
    \mbox{trace}[Cov(A\boldsymbol{\epsilon})]=\sum_k \|[A]_{.k}\|^2\,,
    \]
\end{lem}
\begin{proof}
    For arbitrary matrix $A$, $Cov(A\boldsymbol{\epsilon})=A Cov(\boldsymbol{\epsilon})A^T$ and thus 
    \[
    \mbox{trace}[Cov(A\boldsymbol{\epsilon})]=\mbox{trace}[A Cov(\boldsymbol{\epsilon})A^T]=\mbox{trace}[A^T A Cov(\boldsymbol{\epsilon})]\,.
    \] 
    Moreover, in our case $Cov(\boldsymbol{\epsilon})$ is the identity such that 
     \[
    \mbox{trace}[Cov(A\boldsymbol{\epsilon})]=\mbox{trace}[A^T A ]=\sum_k \|[A]_{.k}\|^2\,,
    \]
\end{proof}

\subsection{KL divergence bounds}
\label{appendix:kl_div_bounds}
\begin{prop}[Lipschtiz continuity-based lower bound]\label{prop:KLlowerLipschitz}
Assume $\gls{invdec}$ is Lipschitz continuous with Lipschitz constant $B>0$, in the sense 
\[
\forall \gls{obs},\boldsymbol{y} \in \RR^d , \left\|\unmix-\gls{invdec}(\boldsymbol{y})\right\|_2\leq B \|\gls{obs}-\boldsymbol{y}\|_2\,.
\] 
Then for any encoder parameter choice
\begin{multline}
    \kl{\gls{q}}{\gls{pzx}}
\geq -\sumk[d]\left[\log \encstdcomp_k (\gls{obs})+\log m(\encmeancomp_k)\right]+c(\gls{obs},\gamma)\\+\frac{\gsq}{2}B^{-2}\left[\left\|\unmix-\gls{mu} \right\|^2+\sumk[d]\gls{sigmak}\right] \,,
\end{multline}
with $c(\gls{obs},\gamma)=-\frac{d}{2}\left(\log (\gsq)+1\right)+\log \gls{px}$.
\end{prop}
\begin{proof}
Starting from the KL divergence expression (\ref{eq:KLgeneralexp}), 
\[
\kl{\gls{q}}{\gls{pzx}} = -\sumk[d]\log \encstdcomp_k (\gls{obs})+\frac{1}{2}\mathbb{E}_{\gls{latent}\sim q_{\encpar}}\left[\|\gls{obs}-\gtmix\|^2\gsq  -2\sumk[d]\log m(\gls{latentcomp}_k)\right] +c(\gls{obs},\gamma)
\]
with additive constant $c(\gls{obs},\gamma)=-\frac{d}{2}\left(\log (\gsq)+1\right)+\log \gls{px}$. By Lipschitz continuity
\begin{multline*}
\kl{\gls{q}}{\gls{pzx}} \geq -\sumk[d]\log \encstdcomp_k (\gls{obs})\\+\frac{1}{2}\mathbb{E}_{\gls{latent}\sim q_{\encpar}}\left[B^{-2}\|\unmix-\gls{latent}\|^2\gsq  -2\sumk[d]\log m(\gls{latentcomp}_k)\right] +c(\gls{obs},\gamma)\,.
\end{multline*}
using ~\cref{lem:normexpect} applied to $\unmix-\gls{latent},\, \gls{latent}\sim \gls{q}$ we get
\begin{multline*}
\kl{\gls{q}}{\gls{pzx}}
\geq -\sumk[d]\log \encstdcomp_k (\gls{obs})+\frac{\gsq}{2}B^{-2}\left[\left\|\unmix-\gls{mu} \right\|^2+\mbox{trace}\left[ Cov\left[\gls{latent}\right]\right]\right]\\-\mathbb{E}_{\gls{latent}\sim q_{\encpar}}\left[  \sumk[d]\log m(\gls{latentcomp}_k)\right]  +c(\gls{obs},\gamma)\\
\geq -\sumk[d]\log \encstdcomp_k (\gls{obs})+\frac{\gsq}{2}B^{-2}\left[\left\|\unmix-\gls{mu} \right\|^2+\sumk[d]\gls{sigmak}\right]\\-\mathbb{E}_{\gls{latent}\sim q_{\encpar}}\left[  \sumk[d]\log m(\gls{latentcomp}_k)\right]  +c(\gls{obs},\gamma)
\,.
\end{multline*}
Using Jensen's inequality for $-\log m$ (convex by Assum.~\ref{assum:VAE}(i)),  we get
\begin{multline*}
\kl{\gls{q}}{\gls{pzx}}
\geq -\sumk[d]\left[\log \encstdcomp_k (\gls{obs})\right]+\frac{\gsq}{2}B^{-2}\left[\left\|\unmix-\gls{mu} \right\|^2+\sumk[d]\gls{sigmak}\right]\\
-\sumk[d]\left[\log m(\encmeancomp_k)\right]+c(\gls{obs},\gamma)
\end{multline*}
by reordering the terms we finally get
\begin{multline*}
\kl{\gls{q}}{\gls{pzx}}
\geq -\sumk[d]\left[\log \encstdcomp_k (\gls{obs})+\log m(\encmeancomp_k)\right]+c(\gls{obs},\gamma)\\+\frac{\gsq}{2}B^{-2}\left[\left\|\unmix-\gls{mu} \right\|^2+\sumk[d]\gls{sigmak}\right] 
\end{multline*}
which is the stated KL lower bound.
\end{proof}

\begin{prop}[Optimal encoder KL divergence upper bound]\label{prop:KLupperLipschitz}
Assume $\gls{dec}$ is Lipschitz continuous with Lipschitz constant $C>0$, in the sense that 
\[
\forall \gls{latent},\boldsymbol{w} \in \RR^d: \qquad  \left\|\gtmix-\gls{dec}(\boldsymbol{w})\right\|_2\leq C \|\gls{latent}-\boldsymbol{w}\|_2\,.
\] 
Assume, $-\log m$ is quadratically dominated, in the sense that
\[
\exists D>0,E>0, \forall u \in \RR, -\log m(u) \leq D|u|^2 +E\,.
\]
Then for the optimal encoder solution of (\ref{eq:minKLopt})
\begin{multline}
    \kl{q_{\encparopt}(\gls{latent}|\gls{obs})}{\gls{pzx}}
\leq d \left(\frac{1}{2}C^2+E+D\left[\dfrac{\|\unmix\|^2}{d}+\frac{1}{\gsq} \right]\right)-\frac{d}{2} +\log \gls{px}
\,,
\end{multline}
and
\begin{multline}
\limsup_{\gamma\to +\infty}{ \kl{q_{\encparopt}(\gls{latent}|\gls{obs})}{\gls{pzx}}}
\leq d \left(\frac{1}{2}C^2+E\right)+D\|\unmix\|^2\\-\frac{d}{2} -\log |\gls{jacobian}_{\gls{dec}}(\unmix)|+\log(p_0(\unmix))
\end{multline}
\end{prop}
\begin{proof}
Starting from the KL divergence expression (\ref{eq:KLgeneralexp}), 
\begin{multline*}
    \kl{\gls{q}}{\gls{pzx}} = -\sumk[d]\log \encstdcomp_k (\gls{obs})+\frac{1}{2}\mathbb{E}_{\gls{latent}\sim q_{\encpar}}\left[\|\gls{obs}-\gtmix\|^2\gsq  -2\sumk[d]\log m(\gls{latentcomp}_k)\right] \\+c(\gls{obs},\gamma)
\end{multline*}
with additive constant $c(\gls{obs},\gamma)=-\frac{d}{2}\left(\log (\gsq)+1\right)+\log \gls{px}$.

Let us choose the following posterior (by universal approximation capabilities of the encoder): 
\begin{eqnarray}
\encmeanstar (\gls{obs}) = \unmix\\
\encstdstar (\gls{obs}) = \frac{1}{\gamma}
\end{eqnarray}

Using Lipschitz continuity we get
\begin{multline*}
    \kl{q_{\encparstar}(\gls{latent}|\gls{obs})}{\gls{pzx}} \leq -\sumk[d]\log \encstdcompstar_k (\gls{obs})+\frac{1}{2}\mathbb{E}_{\gls{latent}\sim q_{\encparstar}}\left[C^2\|\encmeanstar(\gls{obs})-\gls{latent}\|^2\gsq  -2\sumk[d]\log m(\gls{latentcomp}_k)\right] \\+c(\gls{obs},\gamma)
\end{multline*}
then, using 
\[\mathbb{E}_{\gls{latent}\sim q_{\encparstar}}\left[\|\encmeanstar(\gls{obs})-\gls{latent}\|^2\right]=\sumk[d]\mathbb{E}_{\gls{latentcomp}_k\sim \mathcal{N} (\encmeancompstar_k(\gls{obs}),\encstdcompstar_k(\gls{obs})^2)}\left[|\encmeancompstar_k(\gls{obs})-\gls{latentcomp}_k|^2\right]=\sumk[d]\encstdcompstar_k (\gls{obs})^2\,,
\]
we get
\begin{multline*}
\kl{q_{\encparstar}(\gls{latent}|\gls{obs})}{\gls{pzx}} \leq \sumk[d]\left( -\log \encstdcompstar_k (\gls{obs})+\frac{1}{2}C^2\encstdcompstar_k(\gls{obs})^2\gsq\right)\\-\mathbb{E}_{\gls{latent}\sim q_{\encparstar}}\left[ \sumk[d]\log m(\gls{latentcomp}_k)\right] +c(\gls{obs},\gamma)
\end{multline*}
using quadratic domination
\begin{multline*}
\kl{q_{\encparstar}(\gls{latent}|\gls{obs})}{\gls{pzx}} \leq \sumk[d]\left( -\log \encstdcompstar_k (\gls{obs})+\frac{1}{2}C^2\encstdcompstar_k(\gls{obs})^2\gsq\right)\\+\mathbb{E}_{\gls{latent}\sim q_{\encparstar}}\left[ dE+\sumk[d]D |\gls{latentcomp}_k|^2\right] +c(\gls{obs},\gamma)\\
\leq \sumk[d]\left( -\log \encstdcompstar_k (\gls{obs})+\frac{1}{2}C^2\encstdcompstar_k(\gls{obs})^2\gsq\right)\\+dE+D\mathbb{E}_{\gls{latent}\sim q_{\encparstar}}\left[  |\gls{latentcomp}_k|^2\right] +c(\gls{obs},\gamma)
\end{multline*}
Using~\cref{lem:normexpect} we get
\begin{multline*}
\kl{(q_{\encparstar}(\gls{latent}|\gls{obs})}{\gls{pzx}}
\leq \sumk[d]\left( -\log \encstdcompstar_k (\gls{obs})+\frac{1}{2}C^2\encstdcompstar_k(\gls{obs})^2\gsq\right)\\+dE+D\left[\|\encmeanstar(\gls{obs})\|^2+\|\encstdstar(\gls{obs})\|^2 \right]+c(\gls{obs},\gamma)\\
\leq d \left( \log \gamma +\frac{1}{2}C^2\right)+dE+D\left[\|\unmix\|^2+\frac{d}{\gsq} \right] -\frac{d}{2}\left(\log (\gsq)+1\right)+\log \gls{px}
\end{multline*}
hence for a parameter $\encparopt$ achieving the minimum divergence we get
\begin{multline*}
\kl{q_{\encparopt}(\gls{latent}|\gls{obs})}{\gls{pzx}} \leq \kl{q_{\encparstar}(\gls{latent}|\gls{obs})}{\gls{pzx}} 
\leq d \left( \log \gamma +\frac{1}{2}C^2\right)\\+dE+D\left[\|\unmix\|^2+\frac{d}{\gsq} \right] -\frac{d}{2}\left(\log (\gsq)+1\right)+\log \gls{px}\\
\leq d \left(\frac{1}{2}C^2+E+D\left[\dfrac{\|\unmix\|^2}{d}+\frac{1}{\gsq} \right]\right)-\frac{d}{2} +\log \gls{px}
\end{multline*}

As $\gamma \to +\infty$, $\log \gls{px}\to |\gls{jacobian}_{\gls{dec}}(\unmix)|^{-1}p_0(\unmix)$ such that the KL divergence for the optimal solutions is upper bounded by a finite number.

\end{proof}
\subsection{Taylor formula-based approximations}

\begin{lem}[Bound on expectation of multivariate Taylor expansion]\label{lem:taylormult}
Assume $\boldsymbol{f}: \RR^d\to \RR$ is $C^2$ and 
assume $\gls{latent}$ is a multivariate RV on $\RR^d$ with indepedent Gaussian components such that
\[
            z_{k}\sim\mathcal{N}(\gls{muk},\gls{sigmak})
\]
then for all $\gls{latent}_o\in \RR^d$
\begin{equation}\label{eq:taylorboundlem}
\expectation{\gls{latent}}\brackets{\normsquared{ \boldsymbol{f}(\gls{latent}) - \boldsymbol{f}(\gls{latent}_o) - \sum_k \frac{\partial \boldsymbol{f}}{\partial z_k}_{|\gls{latent}_o} (z_k-z^o_k)}} \leq   \frac{d^3}{4} 3 K^2  \sum_{i}\parenthesis{\encstdcomp_{i}}^{4}
\end{equation}

\end{lem}

\begin{proof}
    As described in \citep[p. 162]{marsden_vector_2012}, for the $l$-th component of the function
 \begin{multline}\label{eq:decoder_lin}
    f_l(\gls{latent}) %
    = f_l(\gls{latent}_o)+\sum_k \frac{\partial f_l}{\partial z_k}_{|\gls{latent}_o} (z_k-z^o_k)
    +   \frac{1}{2!} \sum_{i,j} \frac{\partial f_l}{\partial z_i \partial z_j}_{|\gls{latent}_o+t_{ij}(\gls{latent}-\gls{latent}_o)}\!\!\!\!\!\! (z_i-z^o_i)(z_j-z^o_j) \,,  t_{ij}\in \parenthesis{0;1}\,.
    \\
    = f_l(\gls{latent}_o)+\sum_k \frac{\partial f_l}{\partial z_k}_{|\gls{latent}_o} (z_k-z^o_k)
    +   \frac{1}{2!} \sum_{i,j} (\gls{latent}-\gls{latent}_o)^T \gls{hessian}_k (\gls{latent}-\gls{latent}_o) \,,
     \end{multline}
     where the second line puts $\nicefrac{1}{2}$ of the partial derivatives in matrix form (note it is not exactly the Hessian as derivatives are taken at different points). As a consequence 
         \begin{align*}
\parenthesis{f_l(\gls{latent}) - f_l(\gls{latent}_o) - \sum_k \frac{\partial f_l}{\partial z_k}_{|\gls{latent}_o} (z_k-z^o_k)}^2 = \parenthesis{(\gls{latent}-\gls{latent}_o)^T \gls{hessian}_k (\gls{latent}-\gls{latent}_o)}^2,\\
\leq \norm{\gls{hessian}_k}_2^2  \norm{\gls{latent}-\gls{latent}_o}^4\\
\leq \norm{\gls{hessian}_k}_F^2  \norm{\gls{latent}-\gls{latent}_o}^4
    \end{align*}
    where $\norm{\gls{hessian}_k}_2$ is the spectral norm of the matrix and $\norm{\gls{hessian}_k}_F$ is the Frobenious norm \footnote{first inequality comes from Cauchy-Schwartz: $<x,Ax>\leq \|x\|\|Ax\|\leq \|x\|\|A\|_2\|x\|$, second is a classical inequality between norms}
    leading to the bound 
    \begin{align*}
        \parenthesis{f_l(\gls{latent}) - f_l(\gls{latent}_o) - \sum_k \frac{\partial f_l}{\partial z_k}_{|\gls{latent}_o} (z_k-z^o_k)}^2 \leq \frac{d^2}{4} K^2  \norm{\gls{latent}-\gls{latent}_o}^4,
    \end{align*}
    where $K$ is an upper bound on the absolute second order derivatives. 
    We have $(z_k-z^o_k) = \encstdcomp_{k}(x)\epsilon_k$, with $\epsilon$ multivariate normal, so taking the expectation of the above simplifies to:
    \begin{align*}
       \expectation{\boldsymbol{Z}}  \parenthesis{f_l(\gls{latent}) - f_l(\gls{latent}_o) - \sum_k \frac{\partial f_l}{\partial z_k}_{|\gls{latent}_o} (z_k-z^o_k)}^2 &\leq \frac{d^2}{4} K^2  \expectation{\boldsymbol{Z}}\norm{\gls{latent}-\gls{latent}_o}^4,\\
       &= \frac{d^2}{4} K^2  \expectation{\boldsymbol{Z}} \sum_{i,j}\norm{z_i-z_j^o}^2\norm{z_i-z_j^o}^2\\
       &= \frac{d^2}{4} K^2  \sum_{i} \expectation{\boldsymbol{Z}} \norm{z_i-z_i^o}^4\\
       &= \frac{d^2}{4} 3 K^2  \sum_{i}\parenthesis{\encstdcomp_{i}}^4\,.
    \end{align*}
    Now gathering all components $f_l$ to get the squared norm yields:
  \begin{align*}
       \expectation{\boldsymbol{Z}}\brackets{\normsquared{ \boldsymbol{f}(\gls{latent}) - \boldsymbol{f}(\gls{latent}_o) - \sum_k \frac{\partial \boldsymbol{f}}{\partial z_k}_{|\gls{latent}_o} (z_k-z^o_k)}} & \leq \frac{d^3}{4} 3 K^2  \sum_{i}\parenthesis{\encstdcomp_{i}}^4\,.
    \end{align*}

\end{proof}

\subsection{Variational posterior variance optimization problem}

\begin{lem}\label{lem:optim}
For $\alpha>0$, the function
\begin{align*}
    h_\alpha: & \RR_{>0}  \to \RR\\
    & u \mapsto -\log u -\frac{1}{2} +\alpha u^2/2 = \frac{1}{2}\log \frac{1}{u^2} -\frac{1}{2} +\alpha u^2/2
\end{align*}
is strictly convex and achieves its global minimum $\min h_\alpha = \frac{1}{2} \log \alpha$ for $u^*=\frac{1}{\sqrt{\alpha}}$.
\end{lem}
\begin{proof}
Function $h_\alpha$ is stricly convex as a sum of two stricly convex functions. Its derivative, 
\[
\frac{dh_\alpha}{du}(u) = -\frac{1}{u} +\alpha u,
\]
thus vanishes only at the minimum for $u^*=\frac{1}{\sqrt{\alpha}}$. We then get that  
\[
\min h_\alpha = h_\alpha(u^*)=\frac{1}{2} \log \alpha\,.
\]
\end{proof}

\section{Related work}
\label{sec:app_related}
\subsection{Implicit inductive biases in the \gls{elbo}}
\label{subsec:related_implicit_elbo}

\looseness-1
\citet{rolinek_variational_2019} reason about the connection to \gls{pca} in the context of nonlinear Gaussian \glspl{vae} with an isotropic prior and assume that the variational posterior has \textit{diagonal covariance with distinct singular values}. The authors make it explicit that they investigate the consequences of optimizing the \gls{elbo}. They locally linearize the decoder to show the inductive bias in \glspl{vae} that promotes decoder orthogonality. Their results hold for \betavae{}s, where $\beta$ should be in the range of satisfying the polarized regime assumption (\ie, when the \gls{vae} is close to partial posterior collapse). The validity of the assumptions (polarized regime and distinct singular values in \gls{var_cov}) are only experimentally investigated. The same authors extend their work in \citep{zietlow_demystifying_2021}, completing the connection to \gls{pca} for \textit{linear} models. Their experiments, inspired by the connection to \gls{pca} for linear models, show that local perturbations in the data prohibit disentanglement for non-linear models.
\citet{nakagawa_quantitative_2021} builds upon the results of \citep{rolinek_variational_2019} and provides a \textit{novel interpretation} of \glspl{vae} by introducing implicit variables to express the latent space in terms of an isometric embedding.

\looseness-1
\citet{lucas_dont_2019} prove that \textit{linear Gaussian} \glspl{vae} with an isotropic prior give rise to a \textit{column-orthogonal decoder} and therefore uniquely recover the \gls{pca} coordinate axes (not just the correct subspace, as \gls{ppca}~\citep{tipping1999probabilistic} does), yielding identifiability for Gaussian models---but only when the eigenvalues of the data covariance are distinct. In their work, the decoder variance is shown to be small when avoiding posterior collapse. More interestingly, the authors derive a formula for the \gls{elbo} gap in the linear case that is remarkably similar to the \gls{ima} objective. We show in \cref{subsec:app_lin_vae} that in the limit of a deterministic decoder linear Gaussian \glspl{vae} optimize the \gls{ima} objective with $\lambda=1$. \citet{dai2018diagnosing} present more general results than \cite{lucas_dont_2019} since they use affine functions. Additionally, a connection to Robust \gls{pca}~\cite{robustpca} is established.

\citet{kumar_implicit_2020} generalizes \citep{rolinek_variational_2019}, as it admits a variational posterior \gls{q} with \textit{block-diagonal covariance} with a uniqueness result for diagonal \gls{var_cov}. The authors derive a formula for the optimal \gls{var_cov}~\citep[Eq. 12]{kumar_implicit_2020}, showing that when the decoder Hessian \gls{hessian} is diagonal, the decoder Jacobian will be column-orthogonal even for \textit{non-Gaussian} decoders. Their analysis relies on a ``concentrated'' \gls{q} (\ie, they work in what we term the near-deterministic regime) and sufficiently small values of $\beta$---this relationship can be read off from ~\citep[Eq. 12]{kumar_implicit_2020}. %
Interestingly, the authors also show that rotations of the latents can be ruled out, though they do not connect the decoder structure (especially, column-orthogonality of its Jacobian) to any specific generative model for the data, or to considerations on identifiability of the ground truth sources.

\citet{dai2018diagnosing} use a different setting that turns out to be very interesting to compare to ours. The most important is that while we use a factorized Gaussian variational posterior, \citet{dai2018diagnosing} use a non-factorized Gaussian, which leads to major differences. Broadly construed, \citet{dai2018diagnosing} are able to show in their Theorem 2 (which includes the case of equal latent and observation dimensions matching our setting) that in the deterministic limit, their $\kappa$-simple \gls{vae} can perfectly fit arbitrary observed data (barring few assumptions), while the \gls{elbo} gap tends to zero. 
The way it is proven relies on a first step with the Darmois construction~\citep{hyvarinen_nonlinear_1999}, chosing the decoder mean parameter such that its pushforward is exactly the observation distribution. %
Then in a second step, by an appropriate choice of variational posterior parameters, they show that asymptotically the \gls{elbo} gap (\ie, the \gls{kld} divergence between true and variational posteriors) tends to zero in the deterministic limit. 
In contrast, our constraint of factorized variational posterior does not allow the \gls{elbo} gap to vanish in the deterministic limit (unless the decoder mean that fits the data perfectly is in the \gls{ima} class, which is a very special case; in particular, if the Darmois construction is used). For this reason, the proofs and scope of our results are very different:  (i) we use information theoretic bounds to show that the encoder inverts the decoder mean (independently from the fact that this one may or may not fit the data perfectly); (ii) we obtain a rigorous convergence to the \gls{ima} regularized likelihood, which demonstrates that the gap is not eliminated in the deterministic limit.

Regarding our result for \betavae{}s (\cref{prop:beta_vae_ima}), the approach of \citet{mathieu_disentangling_2019} is similar as the authors show that \gls{betaloss} can be expressed in terms of a \textit{rescaled} \gls{elbo}. The difference is that \citet{mathieu_disentangling_2019} uses a rescaling of the parameters $\encpar, \decpar,$ whereas we only scale \gsq.

\subsection{(Near)-deterministic \glspl{vae}}
\label{subsec:related_vae}
\looseness-1 Recent work was inspired by the normalizing flow literature and the shortcomings of the stochastic \gls{vae} architecture to propose designs that are (near)-deterministic. Arguments for this regime range from avoiding posterior collapse (as demonstrated in \citep{lucas_dont_2019}) to avoiding sampling for the reconstruction loss term~\citep{kumar_implicit_2020}. Several papers argued for a similar setting: \citet{dai2018diagnosing} take the limit of ${\gamma\to+\infty}$ (here using $\gamma$ as the square root of the decoder precision and not the decoder variance as used in \cite{dai2018diagnosing}) to derive a result relating encoder and decoder properties in this limit in their Theorem 5, that has a similar flavor to \cref{prop:selfconsist}. In contrast to our nonlinear analysis, this is derived when optimizing with respect to both encoder and decoder parameter, and as stated in the previous section, the non-factorized encoder assumptions leads to fundamentally different behavior of the solutions in the deterministic limit.  \citet{rolinek_variational_2019} refer to the \textit{polarized  regime} (a property of which is that encoder variances are small, \cf \citep[Definition 1]{rolinek_variational_2019}), \citet{kumar_implicit_2020} argue for ``concentrated'' variational posteriors. \citet{ghosh_variational_2020} substitute stochasticity with a regularizer on the decoder Jacobian from an intuitive, whereas \citet{kumar_regularized_2020} motivate these results from an injective flow perspective. \citet{nielsen_survae_2020} also take a normalizing flow perspective to connect \glspl{vae} to deterministic models. 
Besides benefits of avoiding posterior collapse or possible improvements during optimization, this regime serves as a potential connection to the identifiability literature. %

\section{Further remarks on the the \gls{ima}--\gls{vae} connection%
}
\label{sec:app_ext_ima_vae}
In this section, we elaborate on the connection between \glspl{vae} and \gls{ima}, by showing that previous work on linear \glspl{vae} can be directly connected to optimizing \gls{imaloss}. Our intent with this analysis is to provide additional insights about the role of \gls{gamma} in a simpler setting.

\subsection{Linear \gls{vae} from \citet{lucas_dont_2019}}
\label{subsec:app_lin_vae}

We restate the linear \gls{vae} model of \citep{lucas_dont_2019}:
\begin{align}
    \gls{pxz} &= \mathcal{N}\parenthesis{\gls{lindec} \gls{latent}+\boldsymbol{\mu}; \frac{1}{\gsq}\gls{identity}}\\
    \gls{q} &= \mathcal{N}\parenthesis{\gls{linenc}\parenthesis{\gls{obs}-\boldsymbol{\mu}}; \gls{d}},
\end{align}
where \gls{d} is a diagonal matrix, \gls{lindec} the decoder and \gls{linenc} the encoder weights, $\boldsymbol{\mu}$ the mean latent representation. %

The authors show that in stationary points, the optimal value for \gls{d} is
\begin{align}
    \gls{d}^* &= \dfrac{1}{\gsq} \inv{\parenthesis{ \diag{\transpose{\gls{lindec}}\gls{lindec}}+\frac{1}{\gsq}\gls{identity}}} \label{eq:lin_vae_post_cov}
\end{align}
If we substitute this expression into the \gls{elbo} gap (\ie, the \gls{kld} between the variational and true posteriors), we get a similar expression to \gls{cima_local}---as formalized in \cref{prop:lin_gauss_vae}.

\begin{prop}[The \gls{elbo} converges to \gls{imaloss} for linear Gaussian \glspl{vae} if $\gamma\to+\infty$]
For linear Gaussian \glspl{vae}, in the limit of $\gamma \to \infty$, the \gls{elbo} equals the \gls{ima}-regularized log-likelihood in stationary points with $\lambda=1$.\label{prop:lin_gauss_vae}
\end{prop}
\begin{proof}
In \citep[Appendix C.2]{lucas_dont_2019}, it is shown that the gap between exact log-likelihood and \gls{elbo}  for linear Gaussian \glspl{vae} in stationary points reduces to

\begin{align}
    \kl{\gls{q}}{\gls{pzx}} &= \dfrac{1}{2}\parenthesis{\log\det \tilde{\Mmat}-\log\det \Mmat} \label{eq:lin_vae_kl_gap_stationary}\\
    \Mmat &= {\transpose{\gls{lindec}}\gls{lindec}}+\frac{1}{\gsq}\gls{identity}\\
    \tilde{\Mmat} &= \diag{\transpose{\gls{lindec}}\gls{lindec}}+\frac{1}{\gsq}\gls{identity},
\end{align}
where \gls{lindec} is the decoder weight matrix. Reformulating the above expression, we arrive at :
\begin{align}
        \kl{\gls{q}}{\gls{pzx}} &= \log \dfrac{\abs{\diag{\transpose{\gls{lindec}}\gls{lindec}}+\frac{1}{\gsq}\gls{identity}}}{\abs{{\transpose{\gls{lindec}}\gls{lindec}}+\frac{1}{\gsq}\gls{identity}}}\\
        &= \log \dfrac{\abs{\diag{\transpose{\gls{lindec}}\gls{lindec}+\frac{1}{\gsq}\gls{identity}}}}{\abs{{\transpose{\gls{lindec}}\gls{lindec}}+\frac{1}{\gsq}\gls{identity}}} \label{eq:lin_vae_kl_gap_matrix}
\end{align}

Noting that $\transpose{\gls{lindec}}\gls{lindec}$ is symmetric with a \gls{svd} of $\Umat\Lambdamat\transpose{\Umat}$ ($\Umat$ is orthogonal, $\Lambdamat_{ii}=\normsquared{\cols{\gls{lindec}}}$), and $\gls{identity} = \Umat\transpose{\Umat}$; thus:
\begin{align*}
    {\transpose{\gls{lindec}}\gls{lindec}}+\frac{1}{\gsq}\gls{identity}&= {\Umat\Lambdamat\transpose{\Umat}}+\frac{1}{\gsq}\Umat\transpose{\Umat}={\Umat\brackets{\Lambdamat+\frac{1}{\gsq}\gls{identity}}\transpose{\Umat}}
\end{align*}
    
Therefore, \eqref{eq:lin_vae_kl_gap_matrix} can be reformulated as the left \gls{kld}-measure of diagonality~\citep{alyani_diagonality_2017} of the matrix ${\Umat\brackets{\Lambdamat+\nicefrac{1}{\gsq}\gls{identity}}\transpose{\Umat}}$:
\begin{align}
    \kl{\gls{q}}{\gls{pzx}} &= \log \dfrac{\abs{\diag{\transpose{\gls{lindec}}\gls{lindec}+\frac{1}{\gsq}\gls{identity}}}}{\abs{{\transpose{\gls{lindec}}\gls{lindec}}+\frac{1}{\gsq}\gls{identity}}} \\
    &= \log \dfrac{\abs{\diag{\Umat\brackets{\Lambdamat+\frac{1}{\gsq}\gls{identity}}\transpose{\Umat}}}}{\abs{\Umat\brackets{\Lambdamat+\frac{1}{\gsq}\gls{identity}}\transpose{\Umat}}},\label{eq:lin_vae_cima}
\end{align}
which is by definition the local \gls{ima} contrast \gls{cima_local} (\cf  \citep[Appendix C.1]{gresele_independent_2021}).
When $\gamma \rightarrow+\infty$, the above expression converges to the left \gls{kld}-measure of diagonality for $\transpose{\gls{lindec}}\gls{lindec},$ \ie, the local \gls{ima} contrast for the decoder.

\looseness-1
$\gamma\!\rightarrow\!+\infty$ thus means that the \gls{elbo} converges to the \gls{ima} regularized log-likelihood \gls{imaloss} with $\lambda=1$ :
\begin{align*}
    \gls{elbo} &= \log\gls{px}-\kl{\gls{q}}{\gls{pzx}}\\
    &= \log\gls{px}-\cima[\gls{lindec}],
\end{align*}
which concludes the proof.
\end{proof}

\cref{prop:lin_gauss_vae}, especially \eqref{eq:lin_vae_cima}, gives us intuitive understanding on why and how $\gamma$ influences how much the orthogonality of \gls{lindec} is enforced. 
\begin{enumerate}[nolistsep]
    \item Small $\gamma$ (high observation noise) means that there is no reason to promote the orthogonality of the decoder, as the high noise level (\ie, low-quality fit of \gls{obs})  will drive \eqref{eq:lin_vae_cima} towards diagonality via $\nicefrac{1}{\gsq}$.
    \item On the other hand, when $\gamma\rightarrow+\infty,$ then the orthogonality of the decoder is promoted. That is, the decoder precision \gsq acts akin to a weighting factor influencing how strong the \gls{ima} principle should be enforced.
\end{enumerate}  

We can observe that the \gls{elbo} recovers the exact log-likelihood for column-orthogonal \gls{lindec}:
\begin{cor}[For column-orthogonal \gls{lindec} the \gls{elbo} equals the exact log-likelihood]\label{cor:lin_vae_ima}
When \gls{lindec} is in the form $\gls{lindec} = \gls{o}\gls{d},$ then $\diag{\transpose{\gls{lindec}}\gls{lindec}}=\transpose{\gls{lindec}}\gls{lindec}=  \gls{d}\transpose{\gls{o}}\gls{o}\gls{d}=\gls{d}^2,$ \ie the \gls{elbo} corresponds to the exact log-likelihood since \eqref{eq:lin_vae_cima} is zero.
\end{cor}
\cref{cor:lin_vae_ima} also implies that $\gamma$ does not affect the gap between \gls{elbo} and exact log-likelihood for column-orthogonal \gls{lindec}.

\section{Experimental details}
\label{sec:app_exp}
\subsection{The relationship of weight matrix structures and the \gls{ima} function class}
During the experiments we have used different weight matrices either to \textit{ensure} that the mixing is within or to \textit{exclude} it from the \gls{ima} function class. Here we summarize our choices also including the \textit{depth} of the network as it can affect the mixing's place \wrt the \gls{ima} function class.

When we use \textit{orthogonal} weight matrices (\cref{subsec:exp_self_cons},\cref{subsec:exp_elbo_ima_likelihood}), then a single-layer network is within the \gls{ima} class, but adding more layers with elements-wise nonlinearities will move the \gls{mlp} outside the function class. When using \textit{triangular} \glspl{mlp} (\cref{subsec:exp_elbo_ima_likelihood}), the network is also outside the \gls{ima} class (triangular matrices are orthogonal only when they are \textit{diagonal}, see also \citep[Lemma C.1]{gresele_independent_2021}). %

Notably, Möbius transforms~\cite{phillips1969liouville} are conformal maps (thus, they are in the \gls{ima} class) irrespective of the structure of the weight matrix used (\cf \cref{subsec:app_exp_col_gamma_dis} for details).

\subsection{Self-consistency in practical conditions (\cref{subsec:exp_self_cons})}
\label{subsec:app_exp_self_cons}
\looseness-1
For the self-consistency experiments, the mixing is a 3-layer \gls{mlp} with smooth Leaky ReLU nonlinearities~\citep{gresele_relative_2020} and orthogonal weight matrices---which intentionally does not belong to the \gls{ima} class, since our self-consistency result is not constrained to the \gls{ima} class. The 60,000 source samples are drawn from a standard normal distribution and fed into a \gls{vae} composed of a 3-layer \gls{mlp} encoder and decoder with a Gaussian prior. We use 20 seeds for each  $\gsq\xspace \in \braces{\expnum{1}{1};\expnum{1}{2};\expnum{1}{3};\expnum{1}{4};\expnum{1}{5}}$. Additional parameters are described in \cref{tab:exp_self_cons}. Training is continued until the $\gls{elbo}^*$ improves on the \textit{validation set} (we use early stopping~\citep{Prechelt97earlystopping}), then all metrics are reported for the maximum $\gls{elbo}^*$  (\cref{figure:self_cons}).

\begin{table}[H]
    \caption{Hyperparameters for the self-consistency experiments (\cref{subsec:exp_self_cons})}
    \label{tab:exp_self_cons}
    \vskip 0.15in
    \begin{center}
    \begin{small}
    \begin{sc}
    \begin{tabular}{lr} \toprule
        Parameter & Values \\ \midrule
        Encoder & 3-layer \gls{mlp}\\
        Decoder & 3-layer \gls{mlp}\\
        Activation & smooth Leaky ReLU~\citep{gresele_relative_2020}\\
        Batch size & $64$\\
        \# Samples (train-val-test) & $42-12-6$k\\
        Learning rate & \expnum{1}{-3}\\
        \gls{obsdim} & 3\\
        Ground truth & Gaussian\\
        \gls{pz} & Gaussian\\
        \gls{var_cov} & Diagonal\\
        \gsq & \braces{\expnum{1}{1};\expnum{1}{2};\expnum{1}{3};\expnum{1}{4};\expnum{1}{5}} \\
        \# Seeds & 20\\
        \bottomrule
    \end{tabular}
    \end{sc}
    \end{small}
    \end{center}
    \vskip -0.1in
\end{table}

\subsection{Relationship between $\gls{elbo}^*$, \gls{ima}-regularized, and unregularized log-likelihoods (\cref{subsec:exp_elbo_ima_likelihood})}
\label{subsec:app_exp_elbo_ima_likelihood}

\begin{table}[H]
    \caption{Hyperparameters for the \textit{triangular \gls{mlp}} (\textit{not} from the \gls{ima} class) $\gls{elbo}^*$--\gls{imaloss}--log-likelihood experiments (\cref{subsec:exp_elbo_ima_likelihood})}
    \label{tab:hyper_elbo_ima_likelihood_not_ima}
    \vskip 0.15in
    \begin{center}
    \begin{small}
    \begin{sc}
    \begin{tabular}{lr} \toprule
        Parameter & Values \\ \midrule
        Encoder & 3-layer \gls{mlp}\\
        Decoder & 2-layer triangular \gls{mlp} (ground truth)\\
        Activation & Sigmoid\\
        Batch size & $64$\\
        \# Samples (train-val-test) & $100-30-15$k\\
        Learning rate & \expnum{1}{-4}\\
        \gls{obsdim} & 2\\
        Ground truth & Gaussian\\
        \gls{pz} & Gaussian\\
        \gls{var_cov} & Diagonal\\
        \gsq & \brackets{\expnum{1}{1};\expnum{1}{5}} \\
        \# Seeds & 5\\
        \gls{cima_global} (mixing) & $7.072$\\
        \bottomrule
    \end{tabular}
    \end{sc}
    \end{small}
    \end{center}
    \vskip -0.1in
\end{table}

\begin{wrapfigure}{r}{6.5cm}
\centering
    \includesvg[width=11.5cm,keepaspectratio]{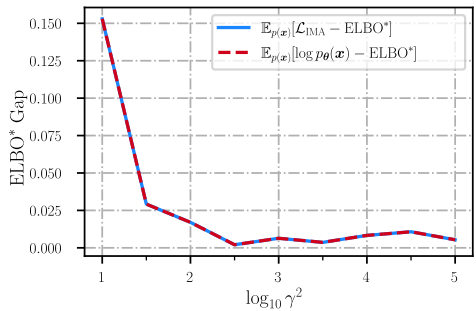}
    \caption{\looseness-1 Comparison of the $\gls{elbo}^*$, the \gls{ima}-regularized and unregularized log-likelihoods over different \gsq with an \gls{ima}-class mixing}
    \label{figure:ima_elbo_likelihood_ima_class}
\end{wrapfigure} 
\looseness-1
For the experiments comparing the $\gls{elbo}^*$, \gls{ima}-regularized, and unregularized log-likelihoods, data is generated by mixing points from a standard Gaussian prior using an invertible neural network. When the mixing is not in the \gls{ima}-class (\cref{tab:hyper_elbo_ima_likelihood_not_ima}), we use a two-layer neural network with sigmoid nonlinearites and triangular weight matrices. When the mixing is from the \gls{ima}-class (\cref{tab:hyper_elbo_ima_likelihood_ima}), we use a one-layer neural network with orthogonal weight matrices. The data dimensionality in both cases is two.

\looseness-1
Training is carried out using a \gls{vae} with a decoder fixed to the ground-truth and separate encoder models for the means and variances of the approximate posterior. The encoder comprises two three-layer neural networks with ReLU non-linearities and a hidden layer size of $50$. Due to training instabilities when using a large $\gamma$, we train the model by first fixing the mean encoder to the ground-truth inverse of the mixing for the first $30$ epochs; thus, only training the variances. We then train both for the remaining epochs. Training is stopped after the $\gls{elbo}^*$ plateaus on the \textit{validation set}. A training set of ${100,\!000}$ samples is used, with a validation set and test set of ${30,\!000}$ and ${15,\!000}$ samples, respectively. The learning rate is \expnum{1}{-4} and the batch size $64$.

We provide additional results when the mixing is from the \gls{ima} class (\cref{tab:hyper_elbo_ima_likelihood_ima}): as \gls{cima_global} is zero, we expect that both \gls{imaloss} and the unregularized log-likelihood match. Indeed, this is what \cref{figure:ima_elbo_likelihood_ima_class} demonstrates.

\begin{table}[H]
    \caption{Hyperparameters for the \textit{orthogonal \gls{mlp}} (from the \gls{ima} class) $\gls{elbo}^*$--\gls{imaloss}--log-likelihood experiments (\cref{subsec:exp_elbo_ima_likelihood})}
    \label{tab:hyper_elbo_ima_likelihood_ima}
    \vskip 0.15in
    \begin{center}
    \begin{small}
    \begin{sc}
    \begin{tabular}{lr} \toprule
        Parameter & Values \\ \midrule
        Encoder & 3-layer \gls{mlp}\\
        Decoder & 1-layer orthogonal \gls{mlp} (ground truth)\\
        Activation & Sigmoid\\
        Batch size & $64$\\
        \# Samples (train-val-test) & $100-30-15$k\\
        Learning rate & \expnum{1}{-4}\\
        \gls{obsdim} & 2\\
        Ground truth & Uniform\\
        \gls{pz} & Uniform\\
        \gls{var_cov} & Diagonal\\
        \gsq & \brackets{\expnum{1}{1};\expnum{1}{5}} \\
        \gls{cima_global} (mixing) & $0$\\
        \bottomrule
    \end{tabular}
    \end{sc}
    \end{small}
    \end{center}
    \vskip -0.1in
\end{table}

\subsection{Connecting the \gls{ima} principle, \gsq, and disentanglement (\cref{subsec:exp_col_gamma_dis})}
\label{subsec:app_exp_col_gamma_dis}

\paragraph{Synthetic data (Möbius transform)}
We use 3-dimensional conformal mixings (\ie, the Möbius transform~\citep{phillips1969liouville}) from the \gls{ima} class with the functional form:
\begin{align*}
    \gls{obs} &= \boldsymbol{t} + \gls{scalar}\dfrac{\gls{lindec}\parenthesis{\gls{latent}-\boldsymbol{b}}}{\norm{\gls{latent}-\boldsymbol{b}}^\epsilon},
\end{align*}
where $\boldsymbol{t}, \boldsymbol{b} \in \RR^{\gls{obsdim}}$, $\gls{lindec}\in \RR^{\gls{obsdim}\times \gls{obsdim}}$, $\gls{scalar} \in \RR$, and $\epsilon=2$ (to ensure nonlinearity) with $\gls{obsdim}=3$. Both ground-truth and prior distributions are \textit{uniform} to avoid the singularity when $\gls{latent}=\boldsymbol{b}.$

To determine whether a mixing from the \gls{ima} class is beneficial for disentanglement, we apply a volume-preserving linear map after the Möbius transform (using $100$ seeds) to construct a mixing outside of the \gls{ima} class. We fix $\gsq=\expnum{1}{5}$ and  report further parameters in \cref{tab:app_exp_col_gamma_dis_synth}. 
Training is continued until the $\gls{elbo}^*$ improves on the \textit{validation set} (we use early stopping~\citep{Prechelt97earlystopping}), then all metrics are reported for the maximum $\gls{elbo}^*$ (\cref{figure:mcc_vs_cima_vs_gamma}).

\begin{table}[H]
    \caption{Hyperparameters for the \textit{synthetic (Möbius)} \gls{ima}--disentanglement experiments (\cref{subsec:exp_col_gamma_dis}) with a linear map}
    \label{tab:app_exp_col_gamma_dis_synth}
    \vskip 0.15in
    \begin{center}
    \begin{small}
    \begin{sc}
    \begin{tabular}{lr} \toprule
        Parameter & Values \\ \midrule
        Encoder & 3-layer \gls{mlp}\\
        Decoder & 3-layer \gls{mlp}\\
        Activation & smooth Leaky ReLU~\citep{gresele_relative_2020}\\
        Batch size & $64$\\
        \# Samples (train-val-test) & $42-12-6$k\\
        Learning rate & \expnum{1}{-3}\\
        \gls{obsdim} & 3\\
        Ground truth & Uniform\\
        \gls{pz} & Uniform\\
        \gls{var_cov} & Diagonal\\
        \gsq & \expnum{1}{5} \\
        \# Seeds & 100\\
        \gls{cima_global} (mixing) & \brackets{0.398;6.761}\\
        \bottomrule
    \end{tabular}
    \end{sc}
    \end{small}
    \end{center}
    \vskip -0.1in
\end{table}

\paragraph{Image data (Sprites)}

We train a \gls{vae} (not \betavae) with a factorized Gaussian posterior and Beta prior on a Sprites image dataset generated using the spriteworld renderer~\citep{spriteworld19} with a Beta ground truth distribution. Similar to~\citep{jack_brady_isprites_2020}, we use four latent factors, namely, \textit{x- and y-position, color and size}, and omit factors that can be problematic, such as shape (as it is discrete) and rotation (due to symmetries)~\citep{rolinek_variational_2019,klindt_towards_2021}. Our choice is motivated by~\citep{horan_when_2021,donoho_image_2005} showing that the data-generating process presumably is in the \gls{ima} class. The architecture both for encoder and decoder consists of four convolutional and three linear layers with ReLU nonlinearities (\cref{tab:app_exp_col_gamma_dis_img}).
Training is continued until the $\gls{elbo}^*$ improves on the \textit{validation set} (we use early stopping~\citep{Prechelt97earlystopping}), then all metrics are reported for the maximum $\gls{elbo}^*$.
\begin{table}[H]
    \caption{Hyperparameters for the \textit{image (Sprites)} \gls{ima}--disentanglement experiments (\cref{subsec:exp_col_gamma_dis})}
    \label{tab:app_exp_col_gamma_dis_img}
    \vskip 0.15in
    \begin{center}
    \begin{small}
    \begin{sc}
    \begin{tabular}{lr} \toprule
        Parameter & Values \\ \midrule
        Encoder & 4-layer Conv2D + 3-layer \gls{mlp}\\
        Decoder & 4-layer Conv2D + 3-layer \gls{mlp}\\
        Activation & ReLU\\
        Batch size & $64$\\
        \# Samples (train-val-test) & $42-12-6$k\\
        Learning rate & \expnum{1}{-5}\\
        \gls{obsdim} & 3\\
        Ground truth & Beta\\
        \gls{pz} & Beta\\
        \gls{var_cov} & Diagonal\\
        \gsq & \expnum{1}{0} \\
        \# Seeds & 10\\
        \bottomrule
    \end{tabular}
    \end{sc}
    \end{small}
    \end{center}
    \vskip -0.1in
\end{table}

\subsection{Optimality of \gsq \wrt its \acrshort{mle}}

During our experiments, we \textit{do not optimize} \gsq, as it is generally the case in the literature~\cite{rolinek_variational_2019,lucas_dont_2019,kumar_implicit_2020}. However, as noted by \citet{rybkin_simple_2021}, doing so could lead to superior sample quality. The price we need to pay for improved sample generation is a more difficult optimization task (also noted in \citep{seitzer_pitfalls_2021}): including \gsq as a trainable parameter might require a careful learning rate tuning, and smaller learning rates can yield suboptimal likelihood values in the beginning~\cite{rybkin_simple_2021}.

During our experiments, we confirmed that making \gsq learnable (all else being equal) yields suboptimal results, particularly in terms of \gls{mcc}. Thus, we opted for comparing our hyperparameter setting to the \textit{maximum likelihood estimate of the decoder variance}, as proposed in \citep[Eq.~(8)]{rybkin_simple_2021}. Accomodating the parameter \gsq instead of the decoder variance, we reformulate the equation as:
\begin{align}
    \gsq_{\mathrm{\gls{mle}}} &= \arg\max_{\gsq} \mathcal{N}\parenthesis{\gtmix,\dfrac{1}{\gsq}\gls{identity}} = \dfrac{1}{\mathrm{\acrshort{mse}}\parenthesis{\gls{obs}, \gls{dec}\parenthesis{\gls{mu}}}} \label{eq:gamma_mle}\\
    &= \inv{\brackets{\dfrac{1}{\abs{\gls{Obs}}}\sum_{\gls{obs}\in \gls{Obs}}\normsquared{\gls{obs}- \gls{dec}\parenthesis{\gls{mu}}}}},
\end{align}
\ie, the \gls{mle} is the mean squared error between observations \gls{obs} and the \textit{decoded mean encodings} $\gls{dec}\parenthesis{\gls{mu}},$ where $\abs{\gls{Obs}}$ denotes the number of observations. Interestingly, this is the inverse of the quantity we report on the right plot of \cref{figure:self_cons}.

\begin{figure}[tb]
    \centering
    \includesvg[]{figures/gamma_vs_mle_estimate.svg}
    \caption{Comparison of $\gsq_{\mathrm{\gls{mle}}}$ and the optimal \gsq we found via grid search (experimental details are the same as in \cref{subsec:exp_self_cons}, detailed in \cref{subsec:app_exp_self_cons})}
    \label{figure:gamma_vs_mle_estimate}
\end{figure}

To compare $\gsq_{\mathrm{\gls{mle}}}$ (calculated as \cref{eq:gamma_mle}) and the optimal value of \gsq we found via grid search from the values $\braces{\expnum{1}{1};\expnum{1}{2};\expnum{1}{3};\expnum{1}{4};\expnum{1}{5}},$ we plot the log of both values in \cref{figure:gamma_vs_mle_estimate}. We can observe that for all values except \expnum{1}{5}, $\gsq_{\mathrm{\gls{mle}}}$ is larger, sometimes with more than one order of magnitude. For \expnum{1}{5}, the mean (for the 20 seeds) lie in the range $[\expnum{0.8}{5};\expnum{3.8}{5}]$ with a mean and standard deviation of $\expnum{2.3\pm0.77}{5},$ indicating that $\gsq=\expnum{1}{5}$ and $\gsq_{\mathrm{\gls{mle}}}$ are in the same order of magnitude, corroborating that we used the optimal setting up to the granularity of our original grid search.

\section{Computational resources}
\label{subsec:app_compute}

The self-consistency (\cref{subsec:exp_self_cons}), the likelihood comparison (\cref{subsec:exp_elbo_ima_likelihood}), and the synthetic experiments with the Möbius transform (\cref{subsec:exp_col_gamma_dis}, particularly \cref{figure:mcc_vs_cima_vs_gamma}) were ran on a MacBook Pro with a Quad-Core Intel Core i5 CPU and required approximately nine days. The Sprites experiments (\cref{subsec:exp_col_gamma_dis}, particularly \cref{figure:dsprites}) required approximately four and a half days on an Nvidia RTX 2080 GPU.

\section{Societal impact}
\label{sec:app_impact}
Our paper presents basic research and is mainly theoretical, though the lack of direct connection to a specific application does not mean that our results could not be used for malevolent purposes. We acknowledge that providing a possible mechanism for why unsupervised \glspl{vae} can learn disentangled representations can inform specific actors that unsupervised \glspl{vae} might be used to extract the true generating factors. Since no auxiliary variables, labels, or conditional distributions are required, this might lead to a broader use of unsupervised \glspl{vae} for trying to learn the true generating factors---including applications with potentially negative societal impact such as extracting features from images, video, or text for personal identification; thus, possibly violating the desire of those who intend to remain anonymous.

\section{Notation}
\label{sec:app_notation}

\end{document}